%% file: neurips_2026.tex
\documentclass{article}

 \usepackage[main, final]{neurips_2026}
\usepackage[utf8]{inputenc} % allow utf-8 input
\usepackage[T1]{fontenc}    % use 8-bit T1 fonts
\usepackage{hyperref}       % hyperlinks
\usepackage{url}            % simple URL typesetting
\usepackage{booktabs}       % professional-quality tables
\usepackage{amsfonts}       % blackboard math symbols
\usepackage{nicefrac}       % compact symbols for 1/2, etc.
\usepackage{microtype}      % microtypography
\usepackage{xcolor}         % colors
\usepackage{graphicx}
\usepackage{subcaption}
\usepackage{amsmath}
\usepackage{mathtools}
\usepackage{multirow}
\usepackage{color}
\usepackage{float}
\usepackage{enumitem}
\usepackage{bbm}
\usepackage{pifont}% http://ctan.org/pkg/pifont
\usepackage{etoc}
\usepackage{algorithm}
\usepackage{algorithmic}
\usepackage{wrapfig}

\input{mysymbol.sty}

\title{Flow Map Denoisers: Traversing the Distortion-Perception Plane for Inverse Problems}

\author{%
  Nicolas Zilberstein\\
  Rice University\\
  \texttt{nzilberstein@rice.edu}\\
  \And
  Morteza Mardani \\ %\thanks{Equal advising.} \\
  NVIDIA Inc. \\
  \texttt{mmardani@nvidia.com} \\
  \AND
  Santiago Segarra \\%\footnotemark[1] \\
  Rice University\\
  \texttt{segarra@rice.edu} \\
}

\begin{document}
\maketitle

\begin{abstract}
Image restoration faces a fundamental tradeoff: methods that minimize error produce blurry reconstructions, while those that maximize perceptual quality yield sharp but less faithful images.
Existing approaches either commit to a single operating point on this distortion--perception (DP) frontier or require paired-data supervision, auxiliary models, or hyperparameter tuning of the sampler to access different points.
We show that flow map models, a recent extension of flow matching for few-step sampling that learns an average field, implicitly define a one-parameter family of denoisers that continuously traverses the DP plane. The lookahead parameter $t$ acts as a control knob between the MMSE and perceptual regimes. For isotropic Gaussian targets, we prove that varying $t$ exactly recovers the optimal DP frontier; for natural images, we observe similar behavior empirically. Within a Plug-and-Play solver, the same mechanism extends to general inverse problems, where the lookahead acts as a regularization parameter that controls a tradeoff between perceptual alignment and data consistency. Despite the lack of exact optimality guarantees in this setting, a single trained flow map traverses the DP tradeoff, being competitive with specialized baselines at both extremes and with training-free DP-traversal methods.
Extensive experiments on CelebA ($128\times 128$) and AFHQ ($256\times 256$) across several linear and nonlinear inverse tasks validate our findings.
Code is available at \url{https://github.com/nzilberstein/Flow-map-denoisers}.
\end{abstract}

% =============================================================================
% =============================================================================
    \section{Introduction}
Image restoration is a highly ill-posed inverse problem: a single degraded observation can correspond to a vast set of plausible clean images.
Standard regression approaches that minimize mean squared error (MSE) approximate the conditional expectation~\citep{ongie2020deep}, achieving minimum distortion at the cost of over-smoothed reconstructions that lack fine detail.
In contrast, posterior sampling techniques~\citep{kawar2021snips} produce realistic-looking solutions at the cost of higher distortion.
This inherent compromise is formalized by the \emph{distortion-perception (DP) tradeoff}~\citep{blau2018perception, freirich2021theory}.
Because the ideal balance between low distortion and high perceptual quality is strictly user-dependent, designing a single model capable of traversing the DP plane at inference time remains a central challenge in computational imaging.

Most existing methods commit to a \emph{single operating point in the DP plane}, either targeting minimum distortion through supervised regression~\citep{dong2015image, terris2026reconstruct}, or minimum perception by leveraging recent generative models as priors~\citep{pokle2024trainingfree, chung2022diffusion, kadkhodaie2021stochastic} to draw samples from the posterior.
A few recent works based on data-dependent flow models aim to traverse the DP plane~\citep{ohayonposterior, delbracio2023inversion, interpolant, zhu2026oftsr}, by explicitly interpolating between measurement-dependent estimators (e.g., the conditional mean) and the clean image.
Other training-free methods move along the DP plane by tuning the sampler, e.g., scaling the injected noise~\citep{zhu2023denoising, wang2025traversing}, interpolating between guidance terms~\citep{garber2024image}, or re-noising a MAP estimate before posterior sampling~\citep{zhang2026stagewise}.
Unlike these approaches, which require external mechanisms, such as interpolation, solver hyperparameters, or retraining, to move along the DP frontier, we show that this tradeoff is already implicitly encoded in the dynamics of flow maps.

In this work, we propose a unified approach based on \emph{flow maps}~\citep{boffibuild, boffiflow, gengmean, sabouralign}, which unifies all operating points along the DP plane within a single model. 
While prior work primarily uses flow maps to accelerate sampling~\citep{gulle2025consistency, sabour2025test, spagnoletti2025latino}, we instead leverage them to \emph{traverse the DP plane}. 
Our key observation is that the lookahead $t$ selects a point along a continuum of estimators---termed \emph{average denoisers}---implicitly encoded within a single trained flow map, thus acting as a \emph{knob} to control the characteristics of the restored image. 
This perspective reveals DP control as an intrinsic property of the learned flow map, eliminating the need for multiple models, auxiliary networks, or heuristic solver choices.

For isotropic Gaussian targets, we prove that this family spans the entire DP frontier, recovering the MMSE estimator at one extreme and a perfect-perception estimator at the other, tracing exactly the optimal DP curve~\citep{freirich2021theory}. While this exact theoretical correspondence strictly holds for Gaussian denoising problems, we empirically observe similar behavior for natural images.
When extending this approach to general, non-Gaussian inverse problems, traversing the lookahead serves as an approximation to the true DP frontier; here, the lookahead acts as a regularization parameter that controls a tradeoff combining perceptual alignment with data-consistency effects, rather than exactly reproducing the denoising DP frontier. By embedding these denoisers within a Plug-and-Play (PnP) framework~\citep{venkatakrishnan2013plug, romano2017little}, we leverage this property to obtain a versatile solver that effectively traverses the DP plane across a wide range of inverse problems. Fig.~\ref{fig:teaser} illustrates this behavior on a 2D Gaussian mixture model.

Overall, our contributions are:
\begin{itemize}[leftmargin=0.1in]
\setlength\itemsep{-0.05em}
\item \textbf{Flow maps as a continuum of DP estimators.}
We reinterpret flow maps as a family of denoisers indexed by a lookahead parameter $t$, enabling continuous traversal of the DP plane with a single model.

\item \textbf{Exact optimality in the Gaussian case.}
We show that this family recovers the \emph{optimal} DP curve for isotropic Gaussian targets, establishing a direct link between flow maps and DP theory.

\item \textbf{PnP reconstruction with continuous control.}
We embed these denoisers into a plug-and-play framework, yielding a unified solver that traverses the DP plane without retraining, and enables high-perceptual reconstructions at a fraction of the cost of posterior sampling methods.

\item \textbf{Empirical validation.}
On CelebA and AFHQ inverse problems (inpainting, motion deblurring, super-resolution, Gaussian deblurring), a single trained flow map is competitive with specialized baselines at both DP endpoints and traces a smooth curve in between, without retraining, auxiliary networks, or retuning the solver across operating points.
\end{itemize}
% =============================================================================

\begin{figure}[t]
    \centering
    \includegraphics[width=0.95\linewidth]{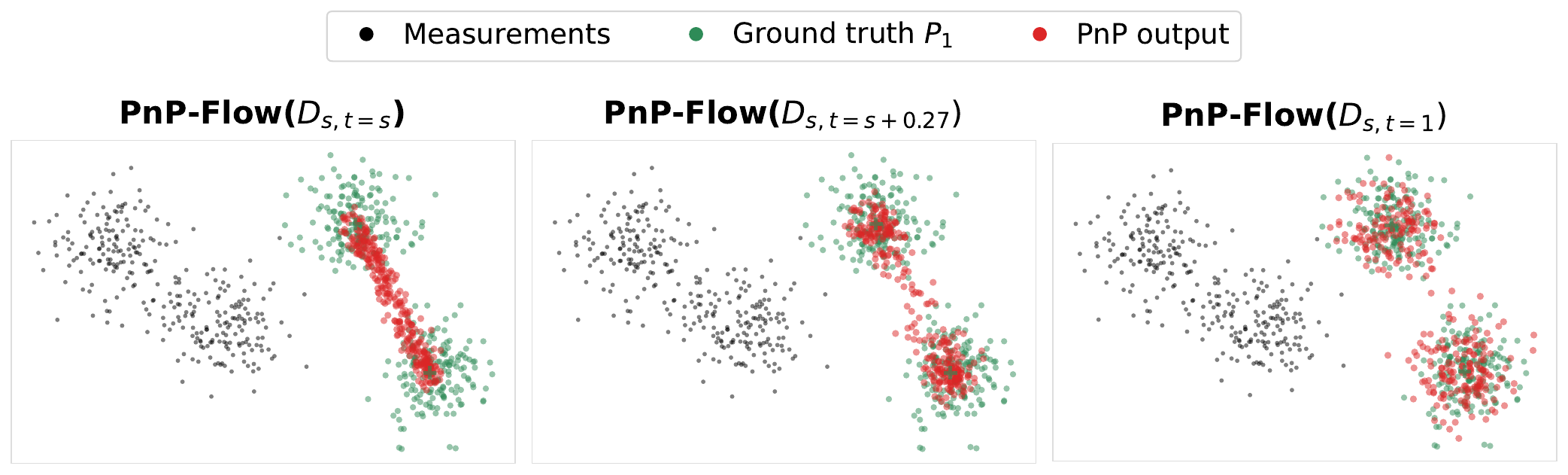}
    \caption{\small{PnP with average denoisers on a 2D mixture-of-Gaussians inverse problem ($\bby = \bbH\bbx + \sigma\bbepsilon$, with $\bbH$ a rotation+scaling operator and $\sigma=0.3$).
    $D_{s,t}$ denotes the average denoiser (defined formally in Section~\ref{sec:avg_denoiser}). 
    The lookahead $t$ controls the DP trade-off: at $t=s$ the denoiser is MMSE and PnP reconstructions concentrate at the posterior mean, which for this symmetric bimodal posterior lies between the two modes, yielding low distortion but poor perceptual quality; as $t\to 1$, the lookahead progressively restores the bimodal structure, recovering samples from both modes at the cost of higher distortion; details of this experiment are in Appendix~\ref{app:teaser_experiment}.}}
    \label{fig:teaser}
    \vspace{-0.2cm}
\end{figure}
\vspace{-0.4cm}
\paragraph{Notation.}
Bold symbols (e.g., $\mathbf{x}$) denote vectors, and bold uppercase symbols (e.g., $\mathbf{H}$) denote matrices. 
Mappings such as the flow map $X_{s,t}$ and denoisers $D_{s,t}$ are written in standard (non-bold) font.
\vspace{-0.4cm}

\section{Background}
% =============================================================================
 
% -----------------------------------------------------------------------------
\subsection{Flow matching and flow maps}
\label{sec:flow_matching}
% -----------------------------------------------------------------------------
 
\paragraph{Stochastic interpolants.}
Stochastic interpolants \citep{albergo2023stochastic, lipmanflow} 
define a family of stochastic processes that interpolate between a 
source distribution $p_0$ (typically $\mathcal{N}(0, \mathbf{I})$) and a target 
distribution $p_1$ (the data distribution):
\begin{equation}
    \bbx_t = \alpha_t \bbx_0 + \beta_t \bbx_1,
    \label{eq:interpolant}
\end{equation}
where $\alpha, \beta: [0,1] \to \mathbb{R}$ are continuously differentiable functions satisfying the boundary conditions $\alpha_0 = 1, \alpha_1 = 0$ and $\beta_0 = 0, \beta_1 = 1$. 
A common choice is the linear interpolant $\alpha_t = 1-t$ and $\beta_t = t$, yielding $\bbx_t = (1-t)\bbx_0 + t \bbx_1$. 
The marginal velocity field $\bbv(\bbx, t) = \mathbb{E}[\dot{\bbx}_t \mid \bbx_t = \bbx]$ yields an ODE $\dot{\bbx}_t = \bbv_t(\bbx_t)$, defining a continuous-time transport map from the source to the target distribution and which can be used to draw samples from $p_1$. 
In practice, a neural network $\bbv_\theta(\bbx, t)$ is trained to approximate $v_t$ via the mean square error loss:
\begin{equation}
\label{eq:dsm_loss}
\mathbb{E}_{t, \bbx_0, \bbx_1} \left[ \left\| \bbv_\theta(\bbx_t, t) - (\dot{\alpha}_t \bbx_0 + \dot{\beta}_t \bbx_1) \right\|_2^2 \right],
\end{equation}
which plays a role analogous to denoising score matching~\citep{vincent2011connection} in diffusion models~\citep{song2020score}.
For the linear interpolant, the velocity field relates directly to the Minimum Mean Squared Error (MMSE) denoiser -- \emph{instantaneous denoiser} -- via Tweedie's formula~\citep{robbins1992empirical, miyasawa1961empirical}:
\begin{equation}
    D_t(\bbx) := \mathbb{E}[\bbx_1 \mid \bbx_t = \bbx] = \bbx + (1-t)\bbv(\bbx, t).
    \label{eq:inst_denoiser}
\end{equation}
\vspace{-0.3cm}
\paragraph{Flow maps.}
Sampling from $p_1$ using the learned velocity field requires numerical integration of the ODE, typically with many function evaluations. 
Flow maps~\citep{boffibuild, boffiflow, gengmean} bypass this integration by training a model to directly predict the average over a finite interval $[s,t]$:
\begin{equation}
    X_{s,t}(\bbx) = \bbx + (t-s)\,\bbv(\bbx, s, t),
    \label{eq:flow_map}
\end{equation}
where $\bbv(\bbx, s, t)$ is the \emph{average velocity} over the interval $[s,t]$, which reduces to the instantaneous velocity on the diagonal, $\bbv(\bbx, t, t) = \bbv(\bbx, t)$.
A single network $\bbv_\theta(\bbx, s, t)$ is trained to predict this average velocity, conditioned on both the current time $s$ and the target time $t$.
Among the self-distillation objectives for flow maps~\citep{boffibuild}, we use the Lagrangian self-distillation (LSD) loss, which combines the diagonal loss~\eqref{eq:dsm_loss} applied to the slice $\bbv_\theta(\cdot, t, t)$ with the off-diagonal term
\begin{equation}
\label{eq:loss_lsd}
\mathcal{L}_{\mathrm{LSD}}(\theta)=\int_0^1 \int_0^t \mathbb{E}_{\bbx_0, \bbx_1}\left[\left\|\partial_t \hat{X}_{s, t}\left(\bbx_s\right)-\bbv_{\theta}\left(\hat{X}_{s, t}\left(\bbx_s\right), t, t\right)\right\|^2\right] \mathrm{d}s\, \mathrm{d}t,
\end{equation}
where $\hat{X}_{s,t}(\bbx) = \bbx + (t-s)\,\bbv_\theta(\bbx, s, t)$. The velocity on the right-hand side is the diagonal slice of the \emph{same} network, so no pre-trained teacher is required.
We use LSD because it gave the most stable training and the best visual quality among the objectives we tried; further background is in Appendix~\ref{app:flowmaps}.
% -----------------------------------------------------------------------------
\subsection{Inverse problems and Plug-and-Play methods}
\label{sec:pnp}

We consider inverse problems of the form
\begin{equation}
    \bby = f(\bbx) + \bbn,
    \label{eq:forward_model}
\end{equation}
where $f(\cdot)$ is a known (possibly non-linear) degradation operator and $\bbn$ is additive noise, typically Gaussian, $\bbn\sim \ccalN(\mathbf{0},\sigma^2\bbI)$. Two main families of methods combine flow or diffusion priors with data-consistency updates~\citep{diffusion_survey, chung2025diffusion, zhenginversebench}: guidance-based methods~\citep{zhang2024flow, pokle2024trainingfree, chung2022diffusion, song2022pseudoinverse} and plug-and-play (optimization-based) methods~\citep{ben2024d, martin2025pnp}. 
We focus on the latter and defer a discussion of guidance-based approaches to Appendix~\ref{app:inverse_problems}.

\paragraph{Plug-and-Play methods.}
Plug-and-Play (PnP) algorithms~\citep{venkatakrishnan2013plug, romano2017little} originate from proximal optimization methods for solving inverse problems of the form $\min_\bbx g(\bbx) + h(\bbx)$. In particular, forward--backward splitting (FBS) alternates between a gradient step on the data-fidelity term, $\bbz^k = \bbx^k - \lambda \nabla g(\bbx^k)$, and a proximal step on the regularizer, $\bbx^{k+1} = \prox_{\lambda h}(\bbz^k)$, where the proximal operator is defined as $\prox_{\lambda h}(\bbz) := \arg\min_{\bbx} \tfrac{1}{2}\|\bbx - \bbz\|^2 + \lambda h(\bbx)$. 
Computing $\prox_{\lambda h}$ corresponds to solving a denoising problem.

\paragraph{Flow and diffusion models as PnP priors.}
Following the PnP principle~\citep{venkatakrishnan2013plug, romano2017little}, recent works instantiate this proximal step with a learned denoiser $D_t$, yielding the iteration $\bbx^{k+1} = D_t(\bbz^k)$~\citep{zhu2023denoising, martin2025pnp, mardani2023variational, zilbersteinrepulsive, renaud2024plug, laumont2022bayesian, laumont2023maximum, hu2024restoration}.
This approach leverages implicit priors learned by generative models while retaining the flexibility of PnP for arbitrary forward operators.
Because PnP is derived from proximal splitting, its output is biased toward a MAP-type solution, which favors low distortion at the expense of perceptual quality and yields over-smoothed reconstructions.
In Section~\ref{sec:method}, we show that replacing the instantaneous denoiser with the average denoiser of a flow map enables continuous control over this tradeoff.
% -----------------------------------------------------------------------------
\subsection{Distortion-Perception Tradeoff}
\label{sec:pd_tradeoff}
% -----------------------------------------------------------------------------
The standard evaluation of image restoration methods relies on average distortion $D = \mathbb{E}[\Delta(\bbx, \hat{\bbx})]$ (e.g., MSE or LPIPS), which measures the discrepancy between the ground truth $\bbx$ and its estimate $\hat{\bbx}$. 
However, since reconstructions should also appear natural to humans, methods are additionally assessed by perceptual quality. 
In practice, this is approximated by a \emph{perceptual index} $P = d(p_\bbx, p_{\hat{\bbx}})$, where $p_{\hat{\bbx}}$ denotes the marginal distribution of the reconstructions and $d$ is a divergence (e.g., the Wasserstein distance) between the true and reconstructed distributions; it is estimated via surrogates such as FID~\citep{heusel2017gans} or KID~\citep{binkowski2018demystifying}.

\citet{blau2018perception} showed that distortion and perception are fundamentally at odds, defining the \emph{distortion--perception (DP) tradeoff}.
Formally, the DP function is the smallest distortion achievable by any estimator at a given perception level,
\[
D(P) = \min_{p_{\hat{\bbx}|\bby}} \mathbb{E}\big[ \|\bbx - \hat{\bbx}\|^2\big] \quad \text{s.t.} \quad W_2(p_\bbx, p_{\hat{\bbx}}) \leq P,
\]
where we use the MSE distortion and the Wasserstein-2 distance for perception; we refer to the graph of $D(P)$ as the \emph{optimal DP curve}.
For this choice, \citet{freirich2021theory} showed that $D(P) = D^* + \big(\max\{P^* - P, 0\}\big)^2$, where $D^* = \mathbb{E}[\|\bbx - \bbx^*\|^2]$ is the MMSE distortion and $P^* = W_2(p_\bbx, p_{\bbx^*})$ is the perceptual gap of the posterior mean $\bbx^* = \mathbb{E}[\bbx|\bby]$.
Note that $P$ constrains the \emph{marginal} distribution of the reconstructions, which is also what FID and KID measure.
The optimal estimator at each perception level is given by the interpolation
\begin{equation}
\label{eq:optimal_pd}
\hat{\bbx}_P = \left(1 - \frac{P}{P^*}\right)\hat{\bbx}_0 + \frac{P}{P^*}\bbx^*,
\end{equation}
where $\hat{\bbx}_0$ is a perfectly perceptual estimator.
This result implies that the entire DP frontier can be recovered by interpolating between two endpoints: the MMSE estimator $\bbx^*$ and a perceptually optimal estimator $\hat{\bbx}_0$. 
Existing approaches access intermediate operating points by interpolating between separate estimators, by task-specific training~\citep{delbracio2023inversion, ohayonposterior}, or by tuning the sampler~\citep{wang2025traversing, zhang2026stagewise}.
In contrast, we seek a \emph{single} model whose conditioning directly parameterizes this continuum.
% =============================================================================
\section{Average denoisers: traversing the DP plane}
\label{sec:method}
% =============================================================================
 
In this section we define and analyze the \emph{average denoiser} associated to flow maps and its connection to the DP tradeoff (Section~\ref{sec:avg_denoiser}).

% -----------------------------------------------------------------------------
\subsection{Flow maps as denoisers}
\label{sec:avg_denoiser}

Recall we seek a family of estimators parametrized by a \emph{single model} that interpolates between $P = P^*$ (lowest distortion) and $P = 0$ (highest perceptual quality).
A natural candidate is the denoiser implicitly defined by the flow map (Section~\ref{sec:flow_matching}). The same object was independently introduced by \citet{lee2026one, lu2026one} for one-step language modeling on discrete data; here we study it in the continuous setting and show it parametrizes the DP frontier.
\begin{definition}[Average denoiser]
    \label{def:avg_denoiser}
    Given a flow map with average velocity $\bbv(., s, t)$, the \emph{average denoiser} (or two-time denoiser) is:
    \begin{equation}
        D_{s,t}(\bbx) := \bbx + (1-s)\bbv(\bbx, s, t).
        \label{eq:avg_denoiser}
    \end{equation}
\end{definition}
This definition has the same form as the instantaneous denoiser~\eqref{eq:inst_denoiser}, but replaces the instantaneous velocity $\bbv(.,s)$ with the average velocity $\bbv(.,s,t)$ over the interval $[s,t]$. 
This modification has an important consequence: the parameter $t$ now controls where $D_{s,t}(\bbx)$ sits on the DP plane. At $t=s$, $D_{s,s}$ is the MMSE denoiser (lowest distortion); at $t=1$, $D_{s,1}(\bbx) = X_{s,1}(\bbx)$ transports the noisy input to the end of the flow, so its outputs follow the data distribution (highest perceptual quality); intermediate $t$ interpolates between the two.
Importantly, both $s$ and $t$ are inputs of $\bbv_\theta(\mathbf{x}, s, t)$, so a \emph{single trained model} produces a whole family of denoisers with different characteristics, as we show in the following section.

% -----------------------------------------------------------------------------
\subsection{Traversing the DP plane with flow maps}
\label{sec:traversing}

\subsubsection{Warm-up: Gaussian case (1D analysis)}
\label{sec:gaussian}

We start with the Gaussian case, which renders a closed-form expression of the average denoiser in~\eqref{eq:avg_denoiser}.
We derive the results in the scalar case ($x \in \mathbb{R}$). They extend to isotropic targets $\ccalN(\mathbf{0}, \sigma_p^2\bbI)$ by applying the analysis component-wise.
Let $p_1 = \ccalN(0, \sigma_p^2)$, with $\sigma_p^2 \leq 1$ (a normalization convention, not a requirement; see Appendix~\ref{app:gaussian}).
Following the stochastic interpolant formulation in Section~\ref{sec:flow_matching}, the marginal variance at noise level $t$ is $\sigma_t^2 = t^2 \sigma_p^2 + (1-t)^2$,
and the instantaneous denoiser is \emph{linear}, given by $\mathbb{E}[x_1\mid x_t = x] = a_t x$, with the scalar gain $a_t = t\sigma_p^2/\sigma_t^2$.
Using Tweedie's relationship in~\eqref{eq:inst_denoiser}, the instantaneous velocity field is also linear, $v(x,t) = b_t x$, with $b_t = [t(\sigma_p^2+1) - 1]/\sigma_t^2$.

\paragraph{Flow map and average denoiser.}
The linearity of the instantaneous velocity yields a scalar flow map
$x_t = \Phi(t,s)\,x_s$ with
$\Phi(t,s) = \sigma_t / \sigma_s$, obtained by integrating $\int_s^t b_\tau\,d\tau = \log(\sigma_t/\sigma_s)$ (see Appendix~\ref{app:gaussian}).
Following Definition~\ref{def:avg_denoiser}, the average denoiser reduces to
$D_{s,t}(x_s) = \Lambda(s,t)\,x_s$ where:
\begin{equation}
    \Lambda(s,t) = \frac{(1-s)\sigma_t - (1-t)\sigma_s}
    {\sigma_s(t-s)}.
    \label{eq:Lambda}
\end{equation}
This gain satisfies $\Lambda(s,s) = a_s$, the MMSE gain (lowest distortion), and $\Lambda(s,1) = \sigma_p/\sigma_s$, the gain that maps $x_s \sim \ccalN(0,\sigma_s^2)$ exactly onto $p_1$ (highest perceptual quality), and is strictly increasing in $t$ (see Appendix~\ref{app:gaussian}).

\paragraph{Optimal DP estimators in the Gaussian case.}
For a Gaussian target, the optimal DP estimator
in~\eqref{eq:optimal_pd} is also linear, given by $\hat{x}_P = \Gamma_s(\alpha)\, x_s$, with the gain given by
\begin{equation}
    \Gamma_s(\alpha) = \alpha\, a_s + (1-\alpha)\, \sigma_p/\sigma_s.
    \label{eq:Gamma}
\end{equation}
The parameter $\alpha = P / P^* \in [0,1]$ parametrizes the perception level, interpolating linearly between the MMSE gain $a_s$ and the perfect-perception gain $\sigma_p/\sigma_s$; both gains, and hence $\Gamma_s$, depend on the noise level $s$.
Given this particular instance of the DP estimator, we can establish a direct relationship between the flow map and the DP estimator in the Gaussian case:

\begin{theorem}[Exact optimality]
    \label{thm:optimal}
    Let $p_1 = \mathcal{N}(0, \sigma_p^2)$ with $\sigma_p^2 \leq 1$, and
    let $\Lambda(s,t)$ be the average denoiser gain~\eqref{eq:Lambda}
    associated with the true flow map. Then for every $s \in (0,1)$ and
    $t \in [s,1]$, there exists a perception level
    \begin{equation}
        \alpha(s,t) = \frac{\sigma_p/\sigma_s - \Lambda(s,t)}
        {\sigma_p/\sigma_s - a_s} \in [0,1]
    \end{equation}
    such that $\Lambda(s,t) = \Gamma_s(\alpha(s,t))$. Moreover,
    $\alpha(s,t)$ decreases monotonically from $1$ at $t=s$
    (MMSE, $P = P^*$) to $0$ at $t=1$ (perfect perception, $P=0$).
    Consequently, the average denoiser $D_{s,t}$ traces the optimal
    DP curve exactly as $t$ varies in $[s,1]$.
\end{theorem}

The proof is deferred to Appendix~\ref{app:gaussian}.
In summary, both $\Lambda(s,t)$ and the optimal DP gain $\Gamma_s(\alpha)$ are continuous monotonic functions that share the same endpoints: $a_s$ (MMSE, lowest distortion) and $\sigma_p/\sigma_s$ (highest perceptual quality).
Hence, inverting the affine relation $\Gamma_s(\alpha) = \Lambda(s,t)$ yields a bijection $\alpha(s,t)$ between the lookahead and the perception level, so the average denoiser exactly realizes the optimal DP estimator at every intermediate operating point.

For anisotropic targets $\ccalN(\mathbf{0}, \bbSigma)$, both endpoints remain exactly optimal, but intermediate lookaheads are exactly optimal only for a flat spectrum (Appendix~\ref{app:multivariate}).
We now turn to more general distributions.

\subsection{Beyond the Gaussian case: image denoising experiments}
\label{sec:learned}

\begin{figure*}[t] % Use figure* if you are in a two-column format so it spans the whole top page
    \centering
    
    % --- Top Row: Three Plots ---
    \begin{subfigure}[b]{0.74\textwidth}
        \centering
        \includegraphics[width=1.01\textwidth]{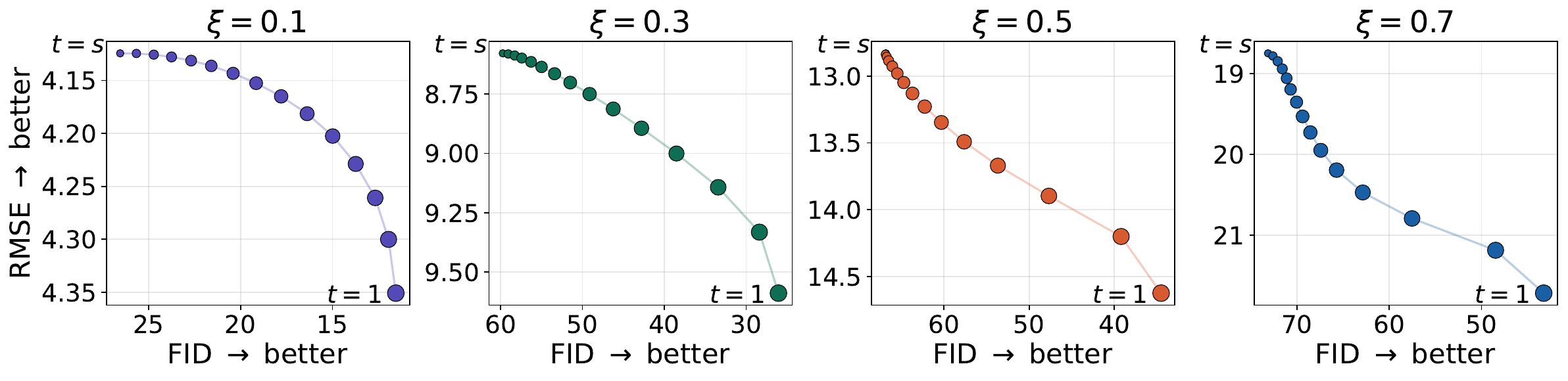}
        \caption{DP tradeoff curves}
        \label{fig:pd_curves}
    \end{subfigure}
    % \vspace{0.5cm} % Space between the top 4 plots and this row
        \begin{subfigure}[b]{0.24\textwidth}
            \centering
            \includegraphics[width=1.01\textwidth]{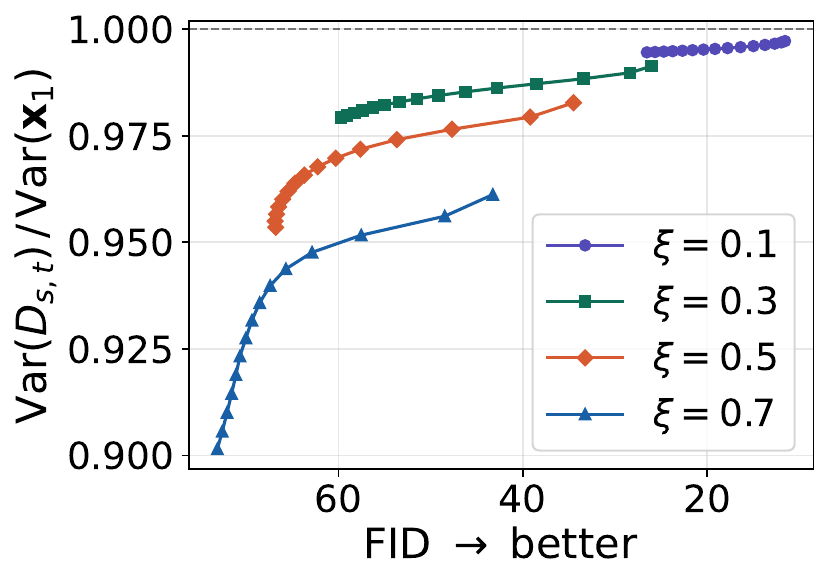}
            \caption{Var. vs FID}
            \label{fig:variance_vs_fid}
        \end{subfigure}

    % --- Main Caption ---
    \caption{{\small Distortion--perception tradeoff and variance restoration for image denoising on CelebA $128 \times 128$. \textbf{(a)} RMSE vs.\ FID at different noise levels $\xi$; each curve is swept by varying the lookahead $t$ of a single trained model. \textbf{(b)} Variance ratio vs.\ FID; the near-collapse of the curves across noise levels suggests a common relationship between variance restoration and perceptual quality.}}
    \label{fig:combined_pd_results}
    \vspace{-0.7cm}
\end{figure*}

For general image distributions, the average denoiser no longer admits a closed-form characterization. While exact DP optimality is not expected beyond the Gaussian case, the structural insight underlying Theorem~\ref{thm:optimal} should still hold: as $t$ varies from $s$ to $1$, the output of $D_{s,t}$ is expected to interpolate smoothly between the MMSE estimator (low distortion, low perception quality) and a perceptually-aligned estimator (higher distortion, improved perception).
Consequently, we expect $(i)$ smooth, monotonic DP curves parameterized by $t$, and $(ii)$ convex trajectories consistent with the optimal DP function (Section~\ref{sec:pd_tradeoff}), which we verify empirically below. While this section focuses on image data, we further support this intuition for a 2D mixture of Gaussians (MoG) in Appendix~\ref{app:mog_pd}, where the DP curve can be computed explicitly.

\paragraph{Setup.}
We train a flow map network $\bbv_\theta(\bbx, s, t)$ on CelebA $128 \times 128$
(Appendix~\ref{app:exp_training}) and evaluate $D_{s,t}(\bbx_s)$ at multiple lookahead values $t \in [s, 1]$, where $t = s$ recovers the MMSE denoiser $D_s$ and $t = 1$ reaches the endpoint of the flow.
Since the stochastic interpolant~\eqref{eq:interpolant} at time $s$ takes the form $\bbx_s = s\,\bbx_1 + (1-s)\,\bbepsilon$ with $\bbepsilon \sim \ccalN(0, \bbI)$, applying the average denoiser $D_{s,t}$ to $\bbx_s$ corresponds to Gaussian denoising: recovering the clean signal $\bbx_1$ from an observation where the signal is scaled by $s$ and corrupted by noise of standard deviation $1-s$. 
We parameterize the noise level by $\xi := 1 - s \in (0,1)$, where larger $\xi$ means heavier noise.

\paragraph{DP traversal.}
Fig.~\ref{fig:pd_curves} shows the empirical DP curves obtained by sweeping the lookahead $t$ from $s$ (MMSE) to $1$ (full perception) across four noise levels $\xi \in \{0.1, 0.3, 0.5, 0.7\}$. 
As expected, the average denoiser traces smooth, monotone curves in the DP plane: increasing $t$ improves perception (lower FID) at the cost of distortion (higher RMSE). 
The resulting curves exhibit the convex shape predicted by the optimal DP function, except at the highest noise level, where they become nearly linear; we attribute this to the flow-map approximation error over the long interval $[s,1]$ (see Appendix~\ref{app:convexity}).
Our approach thus yields a smooth and continuous parameterization of the DP tradeoff from a single trained model.

We also observe that the traversal range increases with the noise level. 
For $\xi = 0.1$, the range is narrow since the MMSE estimator is already close to the data distribution, whereas for $\xi = 0.7$, the lookahead spans a wide range of operating points. 
This is consistent with~\citet{freirich2021theory}, as the perception gap $P^*$ grows with noise, allowing a larger tradeoff between distortion and perception. 
Fig.~\ref{fig:visual} illustrates this effect: as $t$ increases from $s$ to $1$, reconstructions transition from smooth averages to sharp, detailed images.

\paragraph{Variance restoration.}
A natural question is \emph{why} increasing the lookahead improves perceptual quality. 
A partial explanation follows from the law of total variance: for any noisy observation $\bbx_s$, we have $\mathrm{Var}(\bbx_1) = \mathrm{Var}\!\big(\mathbb{E}[\bbx_1 \mid \bbx_s]\big) + \mathbb{E}\!\big[\mathrm{Var}(\bbx_1 \mid \bbx_s)\big]$.
Since the second term is strictly positive, the MMSE estimator $D_s(\bbx_s) = \mathbb{E}[\bbx_1 \mid \bbx_s]$ has lower variance than the true data, i.e., $\mathrm{Var}(D_s(\bbx_s)) < \mathrm{Var}(\bbx_1)$. 
This \emph{variance gap} leads to reduced contrast and over-smoothed reconstructions, and is precisely what the optimal DP estimator~\eqref{eq:optimal_pd} corrects via interpolation. 
As the lookahead $t$ increases, the average denoiser progressively restores this variance, effectively traversing the same tradeoff within a single model. 
Note that matching the variance alone is not sufficient, since adding independent noise also restores variance while degrading the reconstruction; for a Gaussian target, the optimal perfect-perception estimator is instead a deterministic transport map with no added noise~\citep[Thm.~4 and Cor.~2]{freirich2021theory}, which $D_{s,1}$ matches.
Fig.~\ref{fig:variance_vs_fid} supports this empirically: variance ratio and FID improve together, following a nearly identical relationship across noise levels, suggesting that variance restoration underlies the observed perceptual improvements.

%-----------------------------------------------------------------------------
\section{Plug-and-play with average denoisers}
\label{sec:pnp_flow_map}

We now embed the average denoiser~\eqref{eq:avg_denoiser} into the PnP framework from Section~\ref{sec:pnp} to solve general inverse problems with continuous DP control.
Importantly, the DP interpretation applies at the level of the denoiser $D_{s,t}$, while the overall PnP solution reflects its interaction with the forward model. As a result, in general inverse problems, the lookahead acts as a regularization parameter that controls a tradeoff combining perceptual alignment with data-consistency effects, rather than exactly reproducing the denoising DP frontier.

\paragraph{Algorithm.}
Given a trained flow map $\bbv_\theta(\cdot, s, t)$, observations $\bby$, and a chosen lookahead, our method alternates between a gradient descent step on the data-fidelity term $g(\bbx) = \tfrac{1}{2}\|\bby - f(\bbx)\|^2$ and a proximal step utilizing the average denoiser. 
Following~\citet{martin2025pnp}, we include a renoising step prior to applying the denoiser. 
This reprojects the gradient-updated iterate onto the support of the distribution at noise level $s_k$, ensuring that the average denoiser $D_{s_k,t_k}$ processes inputs at its intended operational noise level. 
The noise level is never tuned directly: it follows the Euler discretization $s_k = k/N$, $k = 0, \ldots, N-1$, where $N$ is the number of PnP iterations, as in prior PnP works with flow and diffusion priors. The lookahead is set with a constant offset $\Delta$, i.e., $t_k = \min(s_k + \Delta, 1)$.
The complete procedure is summarized in Algorithm~\ref{alg:pnp_fm}.

\begin{algorithm}[t]
\small
\caption{PnP with Flow Map Denoisers}
\label{alg:pnp_fm}
\begin{algorithmic}[1]
\REQUIRE Trained flow map $\bbv_\theta(\cdot, s, t)$, observations $\bby$, forward model $f$, iterations $N$, noise levels $s_k = k/N$, lookaheads $t_k \in [s_k, 1]$, step sizes $\{\lambda_k\}_{k=0}^{N-1}$
\STATE Initialize $\bbx^0$
\FOR{$k = 0, \ldots, N-1$}
    \STATE $\bbz^k = \bbx^k - \lambda_k \nabla g(\bbx^k)$ \hfill $\triangleright$ Gradient step (data fidelity)
    \STATE Sample $\bbepsilon \sim \ccalN(\mathbf{0}, \bbI)$
    \STATE $\tilde{\bbz}^k = (1 - s_k)\,\bbepsilon + s_k\,\bbz^k$ \hfill $\triangleright$ Renoising to level $s_k$
    \STATE $\bbx^{k+1} = \tilde{\bbz}^k + (1-s_k)\, \bbv_\theta(\tilde{\bbz}^k, s_k, t_k)$ \hfill $\triangleright$ Average denoiser $D_{s_k,t_k}(\tilde{\bbz}^k)$
\ENDFOR
\RETURN $\bbx^N$
\end{algorithmic}
\end{algorithm}

\paragraph{Role of the renoising step.} The stochastic renoising step is what makes the lookahead a DP knob in the inverse-problem setting. With renoising, $D_{s,t}$ operates at its intended noise level $s$, placing it in the regime where Theorem~\ref{thm:optimal} characterizes its behavior. Importantly, the injected noise $(1-s_k)\bbepsilon$ is not a hyperparameter, so the only degree of freedom is the lookahead. Without renoising and for Gaussian priors, the iteration reduces to a fixed-point Tikhonov estimator and the lookahead degenerates into a bias-variance knob; we treat this case in Appendix~\ref{app:fixed_point}. In general, the iteration remains MAP-biased, with the lookahead changing the implicit regularizer, and converges due to the annealing schedule rather than contraction (Appendices~\ref{app:fixed_point} and~\ref{app:lipschitz}).

\paragraph{Computational cost and connection to PnP-Flow.}
Algorithm~\ref{alg:pnp_fm} requires one forward pass through $\bbv_\theta$ per iteration, and reduces to PnP-Flow~\citep{martin2025pnp} for $t_k = s_k$.
The lookahead $t$ is merely a conditioning input, adding zero overhead and preserving efficiency over methods requiring ODE backpropagation or trace computations.

%=================================================
\section{Related work}
\label{sec:related}

\paragraph{Consistency models and flow maps for inverse problems.}
Consistency models~\citep{song2023consistency} learn to map any point on a probability flow ODE to its origin, enabling single-step generation. 
Consistency trajectory models~\citep{kimgeneralized} and flow maps~\citep{boffiflow} generalize this to arbitrary pairs $(s,t)$ by learning the ODE solution operator. 
More recent works repurposed these models for controlled generation, either as guided samplers~\citep{sabour2025test, spagnoletti2025latino, garber2025zero, holderrieth2026diamond}, as priors~\citep{gulle2025consistency}, or for optimizing the noise initialization~\citep{mammadov2026variational}.
Most recently, and concurrently to this work, \citet{huang2026guide} also leverage flow maps for guidance, formulating reward-based generation as an optimal control problem and using the flow map for fast guided sampling.
However, all these approaches primarily target accelerated sampling at a fixed operating point and do not provide explicit control over the restoration tradeoff.
In contrast, we treat the lookahead $t$ as a continuous parameter, turning a single flow map into a family of estimators that traverses the distortion--perception (DP) plane.

\paragraph{Methods that traverse the DP plane.}
Recent approaches explore distortion--perception control by learning data-dependent flows between degraded and clean distributions. Inversion by Direct Iteration (InDI)~\citep{delbracio2023inversion} and stochastic interpolants~\citep{interpolant} learn flows from measurements to clean images, while Posterior-Mean Rectified Flow (PMRF)~\citep{ohayonposterior} replaces the source distribution with an approximation of the posterior mean, and OFTSR~\citep{zhu2026oftsr} distills a one-step super-resolution flow with a tunable fidelity--realism parameter.
However, they rely on supervised training with paired data $(\bby, \bbx)$ and, in the case of PMRF and OFTSR, require an auxiliary or teacher model. 
In contrast, our approach learns a single unconditional flow map from clean images alone, whose lookahead parameter naturally induces a continuous family of denoisers. As a result, DP control is obtained from a single model and extends naturally to arbitrary forward operators through Plug-and-Play inference.

Closer to our setting, several training-free methods traverse the DP plane with a single pre-trained model by acting on the sampler: DiffPIR~\citep{zhu2023denoising} scales the per-step injected noise, DDPG~\citep{garber2024image} interpolates between back-projection and least-squares guidance, \citet{wang2025traversing} scale the injected noise of the reverse SDE, and MAP-RPS~\citep{zhang2026stagewise} re-noises a MAP estimate before posterior sampling (see Appendix~\ref{app:dp_baselines} for further related methods).
Where the knob acts determines which problem is solved. Because their knob controls the output per observation, these methods traverse the \emph{conditional} DP plane, which requires a stochastic estimator. Our method acts on the denoiser with the solver fixed: instead of adding noise, we deterministically transport the estimator itself, traversing the \emph{marginal} DP plane, which is what FID measures, and for which the optimal estimator is provably deterministic~\citep{freirich2021theory}. We compare with \citet{wang2025traversing} and MAP-RPS in Appendix~\ref{app:dp_baselines}.
%s=============================================================================
\section{Numerical experiments}
\label{sec:experiments}
% =============================================================================
Throughout the experiments, we aim to $(i)$ demonstrate competitive performance with state-of-the-art methods, and $(ii)$ show that our method enables continuous traversal of the DP plane with a single model.

\begin{table*}[t]
\centering
\caption{\small{Quantitative results for random inpainting ($90\%$ missing pixels, $\sigma = 0.01$) and motion deblurring ($\sigma = 0.05$) on CelebA and AFHQ. Rows are operating points on the DP plane, not a ranking: posterior samplers (e.g., DPS-ODE) target perception, optimization-based methods (e.g., D-Flow) target distortion. Best results are in \textbf{bold}; second best \underline{underlined}.
}}  
\label{table:inv_problems}
\setlength{\tabcolsep}{2pt} % Tightens column spacing to fit side-by-side
\scalebox{0.62}{
    \begin{tabular}{@{} r cccc cccc @{\hspace{15pt}} cccc cccc @{}}
        \toprule
        \multirow{3}{*}{\textbf{Sampler}} & \multicolumn{8}{c}{\textbf{CelebA $128\times 128$}} & \multicolumn{8}{c}{\textbf{AFHQ $256\times 256$}}\\
        \cmidrule(r){2-9} \cmidrule(l){10-17}
        & \multicolumn{4}{c}{Inpainting} & \multicolumn{4}{c}{Motion Deblurring} & \multicolumn{4}{c}{Inpainting} & \multicolumn{4}{c}{Motion Deblurring} \\
        \cmidrule(lr){2-5} \cmidrule(lr){6-9} \cmidrule(lr){10-13} \cmidrule(lr){14-17}
        & \textbf{PSNR$\uparrow$} & \textbf{LPIPS$\downarrow$} & \textbf{FID$\downarrow$} & \textbf{DISTS$\downarrow$} &\textbf{PSNR$\uparrow$} & \textbf{LPIPS$\downarrow$} & \textbf{FID$\downarrow$} & \textbf{DISTS$\downarrow$}  &\textbf{PSNR$\uparrow$} & \textbf{LPIPS$\downarrow$} & \textbf{FID$\downarrow$} & \textbf{DISTS$\downarrow$} &\textbf{PSNR$\uparrow$} & \textbf{LPIPS$\downarrow$} & \textbf{FID$\downarrow$} & \textbf{DISTS$\downarrow$}  \\
        \midrule
        Flow-priors & 26.50 & 0.089 & 59.43 & 0.132& 
        25.07& 0.151 & 114.41 & 0.188 & 
        \underline{23.94} & 0.193& 27.22 & 0.164 &
        23.14 & \textbf{0.207} & 18.97 & \underline{0.154}\\
        D-Flow & 
        \underline{26.55} & \underline{0.066} & 50.27 & {0.110} &
        \underline{29.77} & \textbf{0.071} & 73.15 & \textbf{0.112}& 
        22.80&0.192& 18.31 &0.137&
        \textbf{25.52} & \underline{0.230} & \underline{19.92} & 0.156 \\
        OT-ODE & 22.16 & 0.158 & 84.92 & 0.179 & 
        21.07 & {0.198} & 106.78 & {0.202} & 
        21.14 & 0.263 & 16.03 & 0.181 & 
        16.73 & 0.348 & 77.53 & 0.234\\
        DPS-ODE &
        25.60 &0.070 & \textbf{45.74} & \textbf{0.106} &
        27.36& {0.101} & 57.02 & {0.133}& 
        23.08 &\underline{0.127} & \textbf{8.69} & \underline{0.098} & 
        20.49 & 0.292 & 36.71 & 0.179\\
        % DPS-ODE  &&&&&&&&\\
        \midrule 
        PnP-Flow ($t = s$)  & 
        \textbf{27.29} & 0.073 & 48.88 & 0.113& 
        \textbf{29.79} & 0.114 & \underline{53.65} & 0.130 & 
        \textbf{24.99} & 0.162 & 13.49 & 0.135 &
        \underline{25.40} & {0.347} & {20.37} & {0.201}\\
        PnP-Flow ($t = 1$) & 
        26.03 & \textbf{0.065} & \underline{46.79} & \underline{0.109} & 
        {29.42} & \underline{0.082} & \textbf{52.61} & \underline{0.119} & 
        23.23 & \textbf{0.122} & \underline{9.11} & \textbf{0.095} & 
        23.31 & \textbf{0.207} & \textbf{14.04} & \textbf{0.137} \\
        \bottomrule
    \end{tabular}
}
\vspace{-0.2cm}
\end{table*}

\paragraph{Setup.}
We evaluate all methods in two datasets: CelebA $128\times128$ and AFHQ $256 \times 256$.
For each dataset, we train a flow map model $\bbv_\theta(\bbx, s, t)$, parametrized by the Song UNet architecture~\citep{song2020denoising, ronneberger2015u}, from scratch using the LSD loss described in Section~\ref{sec:flow_matching} (diagonal and off-diagonal terms).
More training details can be found in Appendix~\ref{app:training_details}.
We assess the reconstruction quality of the different samplers with mean squared error (in terms of $\mathrm{PSNR} = -10\log_{10} \mathrm{MSE}$) and two perceptual metrics, LPIPS~\citep{zhang2018unreasonable} and DISTS~\citep{ding2020image}.
In addition, we compute the FID~\citep{heusel2017gans} and KID~\citep{binkowski2018demystifying} between ensembles of test set images and posterior samples.
We compute the average of 100 samples, and for all the methods we consider Gaussian noise with $\sigma = 0.05$, except for random inpainting with $\sigma = 0.01$; also, in Appendix~\ref{app:poisson} we include an experiment with Poisson noise.

We compared with recent methods that leverage pre-trained flow models, namely OT-ODE~\citep{pokle2024trainingfree}, Flow Priors~\citep{zhang2024flow}, D-Flow~\citep{ben2024d}, DPS-ODE~\citep{chung2022diffusion} and PnP-Flow~\citep{martin2025pnp}, which corresponds to our proposed method for $t = s$.
For all the models, we performed a grid search of hyperparameters targeting the metric each method is designed for (perceptual metrics for posterior samplers, PSNR for optimization-based methods), and we use the same trained network for the prior/regularizer; we defer to Appendix~\ref{app:hyperparamters} for further details.
We also include a comparison with baselines using a flow matching model as prior (a model trained only with the diagonal loss), which can be found in App.~\ref{app:flow_matching_comparison}.

\subsection{Inverse problems}
We evaluate on the following tasks: $(i)$ motion deblurring with random kernel generated with the code\footnote{\url{https://github.com/LeviBorodenko/motionblur}} and size $31\times31$ (CelebA) and $61\times61$ (AFHQ) and $\sigma_b = 1$; $(ii)$ random inpainting with $90\%$ missing pixels; $(iii)$ super-resolution ($\times 4$); and $(iv)$ Gaussian deblurring, which are deferred to Appendix~\ref{app:additional}.

Quantitative results are reported in Table~\ref{table:inv_problems} (with error bars in Appendix~\ref{app:error_bars}), and visual results in Appendix~\ref{app:visual_results} and Fig.~\ref{fig:main_visual}.
PnP-Flow with $t = s$ achieves the best or near-best PSNR across all four (dataset, task) settings, confirming that the average denoiser at small lookaheads behaves as a low-distortion estimator.
At the other extreme, PnP-Flow with $t=1$ achieves the best or near-best LPIPS, DISTS, and FID on every task; on inpainting, it matches DPS-ODE on FID while substantially improving LPIPS and DISTS. 
Competing methods either commit to a single operating point (DPS-ODE favors perception, D-Flow favors distortion) or fall short on both, whereas our approach attains both regimes from a single trained network. Beyond accuracy, PnP-Flow is also the fastest of the methods compared, as shown in Appendix~\ref{app:running_time}.
Lastly, we note that PnP-Flow with $t=1$ is not uniformly the best perceptual method in all tasks. In particular, on Gaussian deblurring and super-resolution (in Appendix~\ref{app:additional}), some baselines achieve better perceptual metrics. Nevertheless, increasing the lookahead from $t=s$ to $t=1$ consistently shifts PnP-Flow toward the perceptual end of the DP spectrum, substantially improving FID, LPIPS, and DISTS relative to the distortion-oriented regime.

\subsection{Analysis of the DP tradeoff}
\label{subsec:dp_exp}

\begin{figure}[t]
    \centering
    \includegraphics[width=0.9\linewidth]{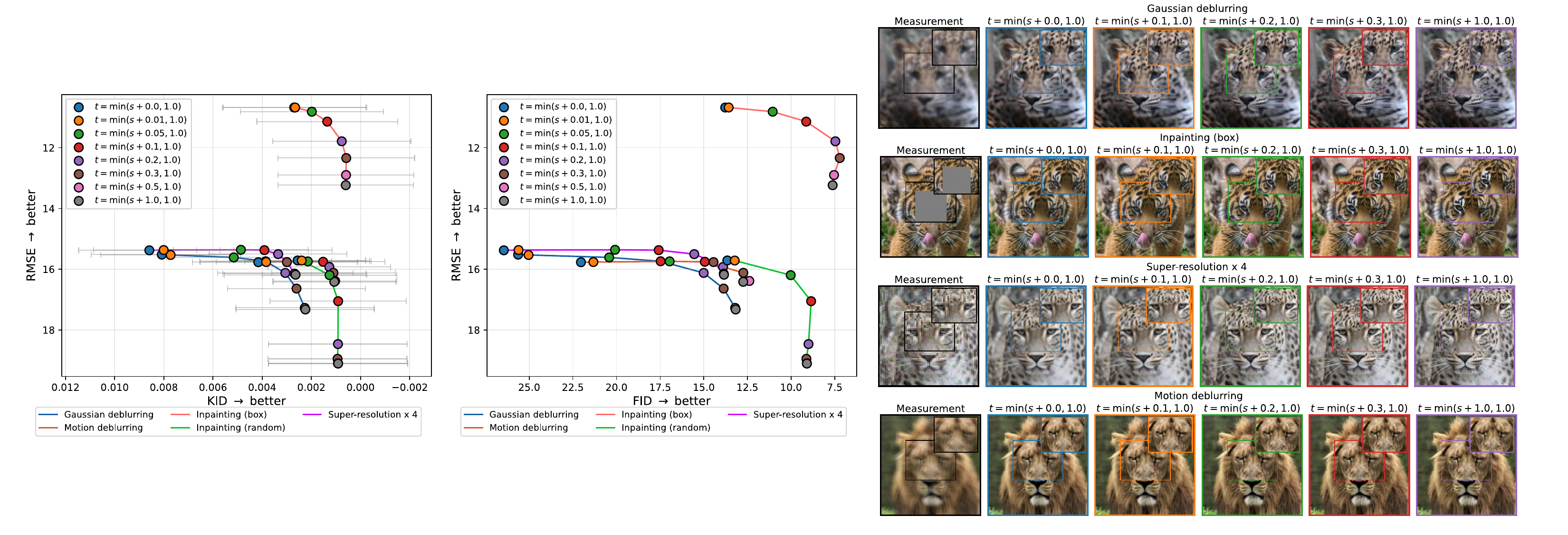}
    \vspace{-0.4cm}
    \caption{\small{Across five linear inverse problems on AFHQ (Gaussian/motion deblurring, box/random inpainting, SR$\times4$), we sweep the lookahead $t$ within a PnP solver and report RMSE (distortion) versus FID and KID (perception). Each marker corresponds to a lookahead offset $\Delta$ (with $t_k = \min(s_k+\Delta,1)$), with small lookaheads favoring low distortion and large ones favoring low FID. A single trained model with a fixed solver configuration per task traverses the DP plane in all tasks; reconstructions on the right illustrate the perceptual change as $t$ increases (larger version in Fig.~\ref{fig:trajectory_afhq}).}}
    \label{fig:DP_exp}
    \vspace{-0.4cm}
\end{figure}

We now analyze how the lookahead parameter $t$ controls the characteristics of our reconstructions by moving in the DP plane.
We evaluate the average denoiser $\{D_{s,t}\}_t$ at multiple lookahead offsets $\Delta$ within a PnP solver and report RMSE (distortion) against FID (perception); all other hyperparameters are fixed per task and shared across lookaheads, and for box inpainting, we use a mask of size $80\times80$.

Results are in Fig.~\ref{fig:DP_exp}, and two observations stand out.
First, varying $t$ produces a smooth, (almost) monotonic trajectory in the DP plane: small lookaheads yield minimum RMSE, large lookaheads minimum FID, and intermediate values smoothly interpolate between the two.
This is consistent with our analysis from Section~\ref{sec:gaussian} and suggests that the behavior extends to natural images and general degradations.
Second, the same qualitative trend appears across all five inverse problems---Gaussian deblurring, motion deblurring, random/box inpainting, and super-resolution---despite the very different forward operators, suggesting that the DP control given by $t$ is degradation-agnostic; the exception is box inpainting near $t=1$, which we analyze in Appendix~\ref{app:failures_limitations}.
Finally, the same behavior holds for the non-linear JPEG restoration and colorization tasks (Appendix~\ref{app:nonlinear}).

\begin{figure}[t] % Use figure* if you are in a two-column format so it spans the whole top page
    \centering
    
    % --- Top Row: Three Plots ---
    \begin{subfigure}[b]{0.45\textwidth}
        \centering
        \includegraphics[width=1\textwidth]{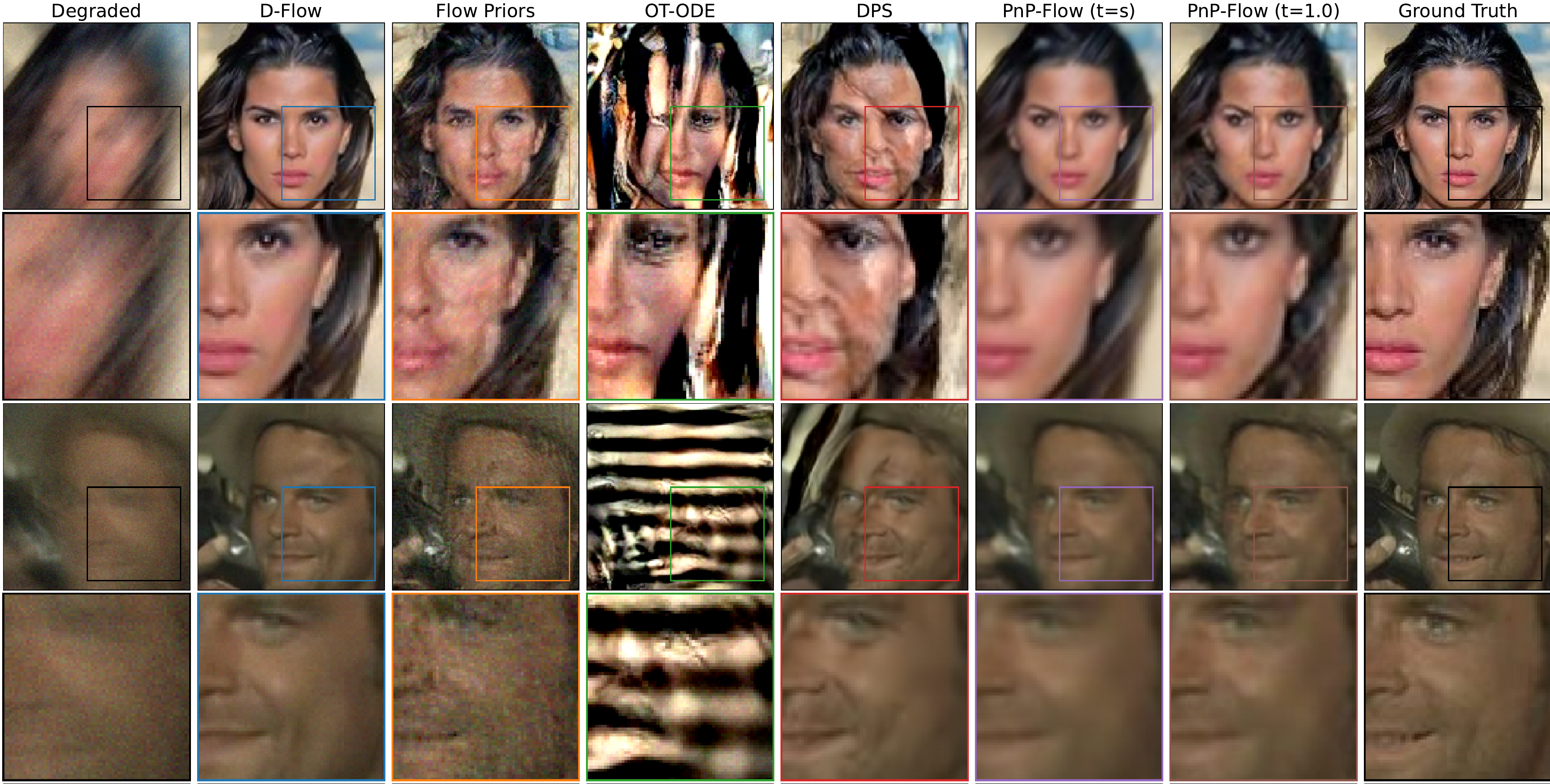}
        \caption{Motion deblurring}
        \label{fig:motion}
    \end{subfigure}
    % \vspace{0.5cm} % Space between the top 4 plots and this row
        \begin{subfigure}[b]{0.45\textwidth}
            \centering
            \includegraphics[width=1\textwidth]{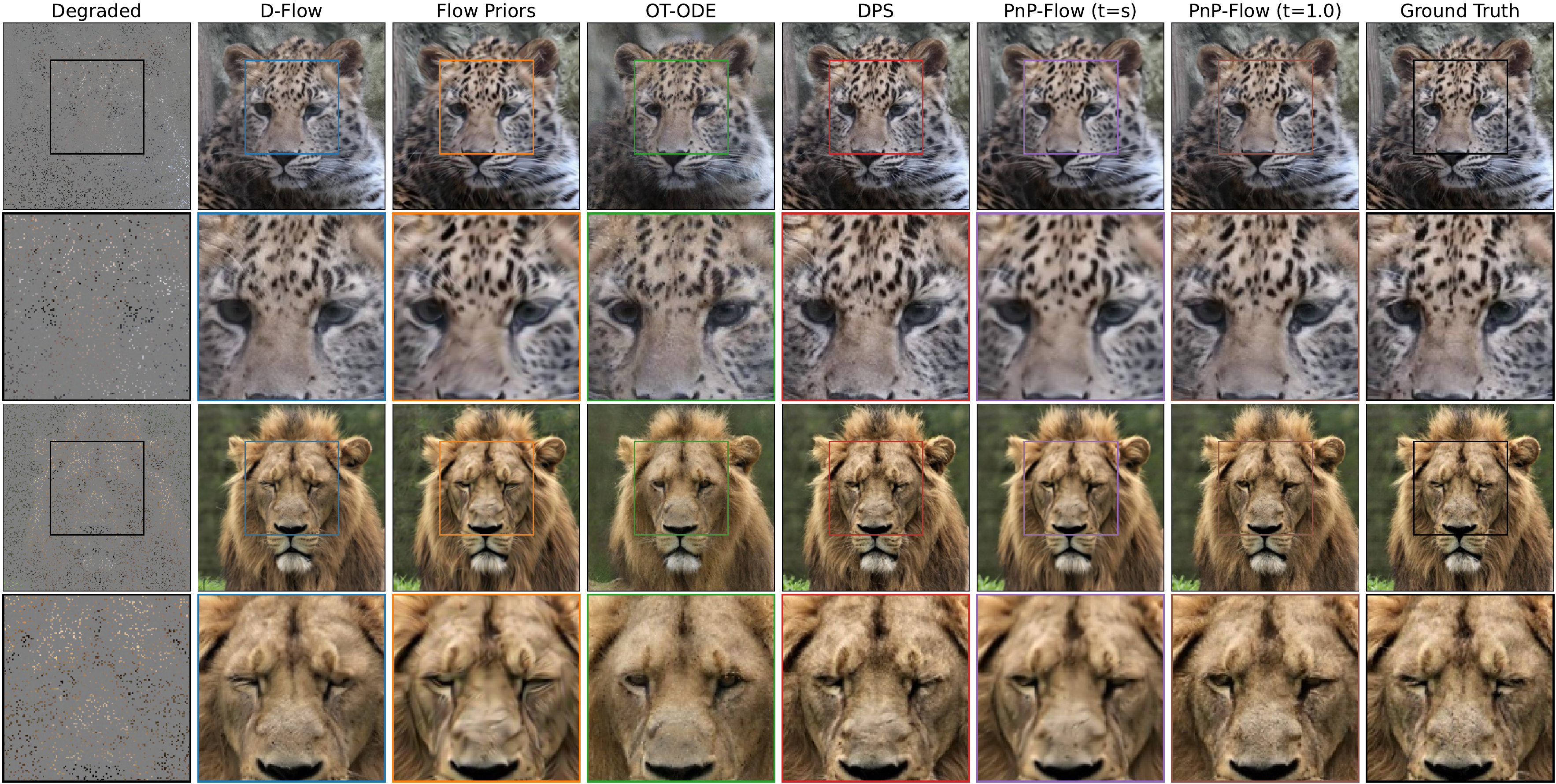}
            \caption{Random inpainting}
            \label{fig:random_inp}
        \end{subfigure}

    % --- Main Caption ---
    \caption{{\small Results on CelebA and AFHQ. PnP-Flow with $t=1$ produces sharper reconstructions than with $t=s$, closer to those of posterior sampling methods.}}
    \label{fig:main_visual}
    % \vspace{-0.7cm}
\end{figure}

\section{Discussion and limitations}
\label{sec:conclusions}
We introduced flow map denoisers, a generalization of the standard instantaneous denoiser that enables continuous, zero-shot traversal of the DP tradeoff. By reinterpreting flow maps as a continuum of average denoisers indexed by a lookahead parameter $t$, we link flow maps to the theory of optimal DP estimators. This traversal is provably optimal for isotropic Gaussian targets and empirically observed for natural images. Embedded within a Plug-and-Play (PnP) solver, our approach enables inference-time control over restoration quality across diverse degradations, without retraining, paired data, or auxiliary models.

Our approach has several limitations.
First, the formal optimality result is restricted to the Gaussian case, and we do not provide convergence guarantees for the PnP iteration. In particular, certain degradations can lead to non-monotonic DP behavior, like box inpainting in Fig.~\ref{fig:DP_exp}, which we attribute to spatially coherent null spaces, and phase retrieval, which we attribute to the non-convex data fidelity (Appendix~\ref{app:failures_limitations}).
Second, flow map models are more expensive to train than standard flow matching models (about $10$ versus $4$ days on $8$ A100 GPUs) and may underperform on pure distortion metrics such as PSNR.
Third, while increasing the lookahead parameter consistently moves reconstructions toward the perceptual end of the DP spectrum, the resulting endpoint ($t=1$) is not always the best perceptual method. 
This suggests that the lookahead parameter should primarily be viewed as a mechanism for traversing the DP tradeoff rather than as a way of reaching the optimal perceptual operating point for every restoration task.

Future directions include training average denoisers directly via SURE-based/consistency losses \citep{raphan2011least, darasambient, tachella2024unsure}. Lastly, our experiments are limited to pixel-space models (CelebA 128, AFHQ 256); extending the framework to latent-space and text-to-image settings is an important direction.

\begin{ack}
This work was supported in part by the NVIDIA Academic Grants Program.
The research was sponsored by the National Science Foundation (CCF 2340481).
\end{ack}

\bibliography{citations}
\bibliographystyle{apalike}
% =============================================================================
% APPENDIX
% =============================================================================
\newpage
% \etocsetnexttocdepth{subsection}
\appendix
\raggedbottom

\section*{Appendix}
% \addcontentsline{toc}{section}{Appendix}

% % Only show appendix entries in TOC
% \etocsettocdepth{subsection}

% % \etocsetnexttocdepth{subsection}

% % Filter: only sections after Appendix
% \etocsetnexttocdepth{subsection}
% \localtableofcontents

\section{Background}

\subsection{Flow Maps}
\label{app:flowmaps}

Consider a probability flow defined by the ordinary differential equation
\[
\dot{\bbx}_t = \bbv_t(\bbx_t), \quad \bbx_0 \sim p_0,
\]
which transports samples from a simple distribution $p_0$ to a target distribution $p_1$. 
The associated \emph{flow map} $\bbX_{s,t} : \mathbb{R}^d \to \mathbb{R}^d$ is defined as the solution operator of this ODE:
\[
X_{s,t}(\bbx_s) = \bbx_t,
\]
i.e., it maps a point at time $s$ to its position at time $t$ along the flow trajectory.
In particular, $X_{0,1}$ directly maps samples from $p_0$ to $p_1$, enabling one-step generation without numerical integration.
A convenient parameterization of the flow map is given by
\[
X_{s,t}(\bbx) = \bbx + (t - s)\,\bbv(\bbx, s, t),
\]
where $\bbv_{s,t}$ can be interpreted as the average velocity between times $s$ and $t$. 
In the limit $t \to s$, this recovers the instantaneous velocity field:
\[
\lim_{s \to t} \partial_t X_{s,t}(\bbx) = \bbv(\bbx, t),
\]
which implies that $\bbv(\bbx, t, t) = \bbv(\bbx, t)$.

\paragraph{Characterization of the flow map.}
The flow map can be equivalently characterized in three ways:

\begin{itemize}
\item \textbf{Lagrangian form:}
\[
\partial_t X_{s,t}(\bbx) = \bbv_t(X_{s,t}(\bbx)),
\]
which describes evolution along trajectories;

\item \textbf{Eulerian form:}
\[
\partial_s X_{s,t}(\bbx) + \nabla X_{s,t}(\bbx)\,\bbv_s(\bbx) = 0,
\]
which expresses conservation along the flow;

\item \textbf{Semigroup property:}
\[
X_{s,t}(\bbx) = X_{u,t}(X_{s,u}(\bbx)),
\]
which encodes consistency across time intervals.
\end{itemize}

These equivalent formulations provide the basis for learning flow maps in practice~\citep{boffibuild}.

\paragraph{Self-distillation framework.}
Instead of learning the instantaneous velocity field $\bbv(.,t)$ and integrating the ODE, flow map models directly learn $\bbv(.,s,t)$ by combining:

\begin{itemize}
\item the \emph{diagonal loss} in~\eqref{eq:dsm_loss} enforcing $\bbv(\bbx,t,t) \approx \bbv(\bbx, t)$,
\item an \emph{off-diagonal loss} enforcing consistency with one of the above three characterizations.
\end{itemize}

Notice that throughout this work we used the LSD loss in~\eqref{eq:loss_lsd}, which is the one associated to the Lagrangian condition.
This approach, known as \emph{self-distillation}, allows training a single model that approximates the full flow map $X_{s,t}$ without requiring a pre-trained teacher model.

\subsection{Inverse problems with flow/diffusion priors}
\label{app:inverse_problems}

While this work focuses on flow-based models, a large body of work uses diffusion priors, which are closely related to our method.
We briefly summarize some of these works.

These methods generate a sample from the posterior by running the reverse process using conditional score at $t$ obtained via Bayes' rule as
\begin{equation}\label{eq:score_post}
    \nabla_{\bbx_t}\log p(\bbx_t|\bby) =  \nabla_{\bbx_t}\log p(\bby|\bbx_t) +  \nabla_{\bbx_t}\log p(\bbx_t).
\end{equation}
While the second term uses a pre-trained diffusion model, the first is intractable, as seen from $p(\bby|\bbx_t) = \int p(\bby|\bbx_0)p(\bbx_0|\bbx_t)\text{d}\bbx_0$. Prior works~\citep{chung2022diffusion,song2022pseudoinverse,kadkhodaie2021stochastic,song2023loss} address this with a Gaussian approximation of $p(\bbx_0|\bbx_t)$ using Tweedie's formula $\mathbb{E}[\bbx_0|\bbx_t] = \frac{1}{\alpha_t}\left(\bbx_t - \sigma_t \bbepsilon_{\bbtheta}(\bbx_t, t)\right)$.
This requires computing the \emph{score Jacobian}, which is computationally expensive, especially for pixel-space models at high resolution. 
This approximation is the one used by DPS-ODE in Section~\ref{sec:experiments}.

Beyond guidance and PnP/optimization-based techniques, another line of work uses sequential Monte Carlo (SMC)~\citep{wu2023practical}, whose particle-based approximations target the exact posterior in the limit of infinitely many particles.

\section{Implementation}
\label{app:exp_training}

\subsection{Training details}
\label{app:training_details}

All models are trained unconditionally with the self-distillation objective of~\citet{boffibuild}: the diagonal loss~\eqref{eq:dsm_loss} on $\bbv_\theta(\cdot,t,t)$ plus the off-diagonal LSD term~\eqref{eq:loss_lsd}. We also compared the off-diagonal objectives of~\citet{boffibuild} (LSD, PSD, and ESD) and selected LSD, which trained stably and gave the best visual quality, consistent with their analysis.

\paragraph{Architecture.}
We use the SongUNet~\citep{song2020score}, a UNet with self-attention, with a base channel width of 128 and channel multipliers $[1, 2, 3, 4]$, giving four resolution scales. 
Self-attention is applied at spatial resolutions $16\times16$ and $8\times8$.  
Each resolution level contains 3 residual blocks, dropout rate $0.1$, and a cosine resampling filter $[1,3,3,1]$.  

\paragraph{Time conditioning.}
The network is conditioned on two scalars $(\alpha_r, \alpha_t) \in [0,1]^2$ representing the interpolation times of the source and target.
Both are encoded with a \emph{positional (sinusoidal) embedding}.
Following~\citet{boffibuild}, the two conditioning scalars are embedded as $\mathrm{PE}(\alpha_t - \alpha_r)$ and $\mathrm{PE}(\alpha_r)$, summed, and passed through a two-layer MLP with SiLU activations to produce a global conditioning vector injected into every residual block via adaptive group normalization.
Compared with a flow matching model of the same architecture, the only addition is this second time embedding, so both models have approximately $128$M parameters.

\paragraph{Loss weighting.}
For the off-diagonal term, we sample $s, t$ uniformly from $\{(s,t) : 0 < s < t \leq 1\}$ as detailed in~\citet{boffibuild}.
Following~\citet{karras2024analyzing}, we learn a per-sample uncertainty weight for the loss: a small network takes the pair $(\alpha_r, \alpha_t)$ and outputs a scalar log-variance $\log\sigma^2(\alpha_r, \alpha_t) \in \mathbb{R}$, which weights the loss as $\mathcal{L} \propto e^{-\log\sigma^2}\,\ell + \log\sigma^2$.
This adaptive weighting automatically rescales the diagonal and off-diagonal terms to comparable magnitudes, so no manual balancing is required.

\paragraph{Datasets and resolution.}
\begin{itemize}[noitemsep]
  \item \textbf{CelebA $128\!\times\!128$}~\citep{liu2015faceattributes}:
        202{,}599 images; 90\% train split.
  \item \textbf{AFHQ $256\!\times\!256$}~\citep{choi2020starganv2}:
        16{,}130 images; 90\% train split.
\end{itemize}

During training, 75\% of each batch is used for the diagonal loss and the remaining 25\% for the off-diagonal one. We adopted this split from~\citet{boffibuild} and did not tune it; with it, training was stable on both datasets, without divergence, restarts, or per-dataset tuning.
Augmentation (horizontal flip, rotation, scaling, anisotropic
scaling, translation) is applied with probability $p=0.12$.

\paragraph{Optimization.}
All models are optimized with Adam~\citep{kingma2014adam} ($\beta_1=0.9$, $\beta_2=0.95$,
$\varepsilon=10^{-8}$) at a learning rate of $2\times10^{-4}$ with no
learning-rate schedule. 

\paragraph{Compute.}
For training, we use 8 NVIDIA A100 GPUs, while for inference we use a single GPU for all methods.
Flow maps are more expensive to train than flow matching models, because the off-diagonal term requires a Jacobian--vector product at every step. On CelebA, reaching the reported quality took about $10$ days on $8$ A100 GPUs (with good results after $6$--$7$ days), versus about $4$ days for a flow matching model with the same architecture and data (the model of Appendix~\ref{app:flow_matching_comparison}). No large-scale pre-trained flow map is currently available, so we trained all models from scratch; our construction only requires evaluating a flow map at pairs $(s,t)$ with $s<t$, and thus applies to any such model once available.

\begin{table}[H]
\centering
\caption{Per-run training configuration.}
\small
\begin{tabular}{lcccc}
\toprule
Dataset & Resolution & Batch size & GPUs & Days\\
\midrule
CelebA & $128\!\times\!128$ & 128 & 8 & 10 ($\approx 700$k iterations)\\
AFHQ      & $256\!\times\!256$ & 32  & 8 & 9 ($\approx 580$k iterations)\\
\bottomrule
\end{tabular}
\label{tab:training_config}
\end{table}

\subsection{Hyperparameters and baselines}
\label{app:hyperparamters}

For all baselines, unless otherwise specified, we use $\sigma = 0.05$ for Gaussian/motion deblurring and super-resolution, and $\sigma = 0.01$ for inpainting. 
The blur kernel size is set to 61 for AFHQ and 31 for CelebA.
For each method, hyperparameters are grid-searched around the values suggested by the authors, targeting the metric the method is designed to optimize: perceptual metrics for posterior samplers (DPS-ODE, OT-ODE) and PSNR for optimization-based methods (Flow-Priors, D-Flow, PnP-Flow with $t=s$).

\paragraph{Flow-Priors~\citep{zhang2024flow}.}
We follow the original parameterization with guidance weight $\lambda$ and step size $\eta$. We use $\lambda = 10^{5}$, $\eta = 10^{-2}$, and $K = 1$ inner iteration in all settings. 
The number of outer iterations is fixed to $N = 100$.

\paragraph{D-Flow~\citep{ben2024d}.}
We tune the regularization parameter $\lambda \in \{0.1, 0.01, 0.001\}$ and the initialization blending parameter $\alpha \in \{0.1, 0.3, 0.5\}$, retaining $\lambda = 0.01$ and $\alpha = 0.1$. As recommended by the authors, the inner LBFGS optimization runs for $20$ iterations and the ODE is solved with $6$ Euler steps. The number of outer iterations is fixed to $20$.

\paragraph{OT-ODE~\citep{pokle2024trainingfree}.}
We tune the initial time $t_0 \in \{0.1, 0.2, 0.3, 0.4\}$ and the step-size schedule $\gamma \in \{\text{constant}, \sqrt{t}\}$. We use $\gamma = \text{constant}$ throughout. The retained values are $t_0 = 0.1$ for inpainting, $t_0 = 0.3$ for Gaussian deblurring and super-resolution, and $t_0 = 0.4$ for motion deblurring on AFHQ. The number of iterations is fixed to $N = 100$.

\paragraph{DPS-ODE~\citep{chung2022diffusion}.}
This method corresponds to the ODE (deterministic) version of DPS, originally designed for diffusion models. We use the ODE version to match the flow matching literature~\citep{zhang2024flow}; DPS is often reported to perform better with stochasticity, and stochastic samplers are covered by the comparison of Appendix~\ref{app:dp_baselines}.
We tune the guidance weight $\eta \in \{10^{2}, 10^{3}, 10^{4}\}$, retaining $\eta = 10^{3}$ as default, $\eta = 10^{2}$ on CelebA box inpainting, and $\eta = 10^{2}$ on AFHQ Gaussian deblurring. 
The number of iterations is fixed to $N = 100$.

\paragraph{PnP-Flow w/ $t=s$~\citep{martin2025pnp} and w/ $t=1$ (ours).}
We adjust the learning-rate exponent $\alpha \in \{0.05, 0.1\}$ and the number of iterations $N \in \{30, 50, 100, 300\}$, which also sets the noise schedule $s_k = k/N$. We use $\alpha = 0.05$ in most settings and $\alpha = 0.1$ for motion deblurring; $N = 100$ by default, $N = 300$ for random inpainting on AFHQ, and $N \in \{30, 50\}$ for motion deblurring.

\subsection{Computational cost}
\label{app:running_time}

\paragraph{Running time.}
Table~\ref{tab:running_time} reports per-image reconstruction times on a single A100 GPU. PnP-Flow is the fastest method we compare against, while reaching perceptual quality typically associated with posterior sampling. This illustrates a concrete advantage of using a flow map as a PnP denoiser: it inherits the speed of MAP-style approaches without sacrificing perceptual fidelity.

\begin{table}[H]
\centering
\caption{Running time for each method. PnP with a flow map as a denoiser is the fastest method, while generating images that are closer to those from posterior sampling methods.}
\begin{tabular}{lc}
\toprule
Sampler & Time [sec] \\
\midrule
Flow-priors   & 34.8 \\
D-Flow      & 322 \\
OT-ODE &  8.4\\
DPS-ODE &  7.7\\
PnP-flow &  4.5 \\
\bottomrule
\end{tabular}
\label{tab:running_time}
\end{table}

\paragraph{Memory footprint.} Overall, almost all the methods fit in a single A100 with a batch of 5 images for AFHQ 256x256; the only one that could not fit is D-flow, which needed to use a single image per batch.

\subsection{Experiment of the teaser plot}
\label{app:teaser_experiment}

We detail here the experiment underlying Fig.~\ref{fig:teaser}. 
The data distribution $p_1$ is a 2D isotropic mixture of Gaussians,
\[
p_1 = \tfrac{1}{2}\mathcal{N}(\bbmu_a, \sigma_c^2 \bbI)
     + \tfrac{1}{2}\mathcal{N}(\bbmu_b, \sigma_c^2 \bbI),
\quad
\bbmu_a = (4.5, 3.5),\ \bbmu_b = (3.5, 5.5),\ \sigma_c = 0.35.
\]
Observations are generated according to the linear forward model in~\eqref{eq:forward_model}, with $f(\cdot)=\bbH$ a fixed rotation of angle $\pi/6$ followed by isotropic rescaling. 

Similarly to the case in Appendix~\ref{app:mog_pd}, the MMSE denoiser $D_s(x)=\mathbb{E}[X_1\mid X_s=x]$ admits a closed-form expression via the component posteriors. We use this analytic denoiser to construct the flow-map denoiser
\[
D_{s,t}(\bbx) = \bbx + (1-s)\,\bbv(\bbx, s, t), 
\quad 
\bbv(\bbx, s, t), = \frac{1}{t-s}\int_s^t \frac{D_\tau(\bbx_\tau)-\bbx_\tau}{1-\tau}\,d\tau,
\]
where the average velocity is approximated by Euler integration of the flow $\dot\bbx_\tau = (D_\tau(\bbx_\tau)-\bbx_\tau)/(1-\tau)$ using $100$ steps.
Using this approximation of the average denoiser, we run PnP-Flow in Alg.~\ref{alg:pnp_fm} with $s=0.45$ (which is fixed to isolate the behavior of the lookahead), step size $\gamma_k=\lambda\sqrt{1-k/K}$ with $\lambda=0.4$, and $K=50$ iterations. 
We initialize with $\bbx^{(0)}=\bbH^{-1}\bby$ and run the algorithm independently for each sample.
We consider three lookaheads $t\in\{s,\,s+0.27,\,1\}$, corresponding to the MMSE, an intermediate, and the full lookahead regime.

The behavior follows the distortion–perception tradeoff: at $t=s$, reconstructions collapse toward the inter-mode mean (low variance, low distortion), while at $t=1$ the bimodal structure is recovered (higher perceptual quality). Intermediate lookaheads interpolate smoothly between these regimes, all using a single denoiser.

\paragraph{Trajectories induced by the flow map.}
Fig.~\ref{fig:teaser_traj} shows, for a subset of samples, how the lookahead moves each reconstruction. Let $\tilde\bbz^{K}$ be the input of the denoiser at the last iteration; we plot the path $\tau \mapsto D_{s,\tau}(\tilde\bbz^{K})$ for $\tau \in [s, t]$. It starts at the MMSE estimate $D_{s,s}(\tilde\bbz^{K})$ and ends at the PnP output $D_{s,t}(\tilde\bbz^{K})$. At $t = s$ the path reduces to a point between the two modes; as $t$ increases, the flow map transports the estimate away from the posterior mean and into one of the modes, which restores the bimodal structure of $p_1$.

\begin{figure}[H]
    \centering
    \includegraphics[width=\linewidth]{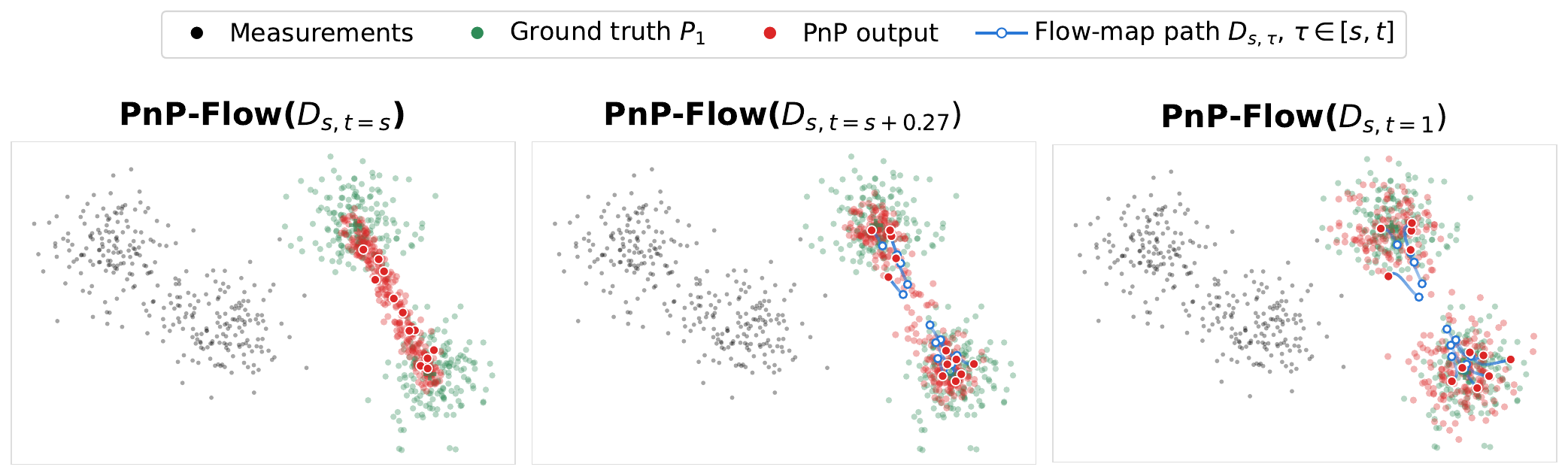}
    \caption{\small{Same experiment as Fig.~\ref{fig:teaser}, with the paths $\tau \mapsto D_{s,\tau}(\tilde\bbz^{K})$, $\tau \in [s,t]$, induced by the flow map for $12$ samples (blue; open circle: MMSE estimate at $\tau = s$; red: PnP output at $\tau = t$).}}
    \label{fig:teaser_traj}
\end{figure}

\section{Additional experiments}

\subsection{Results with error bars}
\label{app:error_bars}

Here, we provide the results from Section~\ref{sec:experiments} with the corresponding error bars.

\begin{table}[H]
\centering
\caption{\small{Quantitative results for random inpainting ($90\%$) with $\sigma = 0.01$ on CelebA and AFHQ. Best results are in \textbf{bold}.}}
\label{table:inpainting_errorbars}
\setlength{\tabcolsep}{4pt}
\scalebox{0.8}{
\begin{tabular}{@{} r cccc cccc @{}}
\toprule
\multirow{2}{*}{\textbf{Sampler}} 
& \multicolumn{4}{c}{\textbf{CelebA $128\times128$}} 
& \multicolumn{4}{c}{\textbf{AFHQ $256\times256$}} \\
\cmidrule(r){2-5} \cmidrule(l){6-9}
& \textbf{PSNR$\uparrow$} & \textbf{LPIPS$\downarrow$} & \textbf{FID$\downarrow$} & \textbf{DISTS$\downarrow$}
& \textbf{PSNR$\uparrow$} & \textbf{LPIPS$\downarrow$} & \textbf{FID$\downarrow$} & \textbf{DISTS$\downarrow$} \\
\midrule
Flow-priors &
$26.50\pm {\scriptstyle 1.00}$ & $0.089\pm {\scriptstyle 0.011}$ & 59.43 & $0.132 \pm {\scriptstyle 0.012}$ & 
$23.94 \pm {\scriptstyle 0.79}$ & $0.193 \pm {\scriptstyle 0.021}$  & 18.96 & $0.164 \pm {\scriptstyle 0.011}$ \\
D-Flow & 
$\underline{26.55}\pm {\scriptstyle 2.79}$ & $\underline{0.066}\pm {\scriptstyle 0.031}$ & 50.27 & ${0.110}\pm {\scriptstyle 0.031}$ &
$22.80\pm {\scriptstyle 3.80}$ & $0.192\pm {\scriptstyle0.146}$ & 18.31 & $0.137\pm {\scriptstyle0.083}$ \\
OT-ODE & 
$22.16 \pm {\scriptstyle 0.74}$ & $0.158 \pm {\scriptstyle 0.019}$ & 84.92 & $0.179\pm {\scriptstyle 0.016}$ & 
$21.14 \pm {\scriptstyle 0.26}$ & $0.263 \pm {\scriptstyle 0.019}$ & 16.03 & $0.181\pm {\scriptstyle0.012}$ \\
DPS-ODE & 
$25.60 \pm {\scriptstyle1.32}$ & $0.070 \pm {\scriptstyle 0.054}$ & \textbf{45.74} & $\textbf{0.106}\pm {\scriptstyle 0.024}$ & 
$23.08\pm {\scriptstyle 0.73}$ & $\underline{0.127}\pm {\scriptstyle 0.021}$ & \textbf{8.69} & $\underline{0.098}\pm {\scriptstyle 0.014}$ \\
\midrule
PnP-Flow ($t = s$) & 
$\textbf{27.29}\pm {\scriptstyle 1.15}$ & $0.073\pm {\scriptstyle 0.073}$ & 48.88 & $0.113\pm {\scriptstyle 0.013}$ & 
$\mathbf{24.99}\pm {\scriptstyle 0.78}$ & $0.162\pm {\scriptstyle0.013}$ & 13.49 & $0.135\pm {\scriptstyle0.008}$ \\
PnP-Flow ($t = 1$) & 
$26.03\pm {\scriptstyle 1.03}$ & $\textbf{0.065}\pm {\scriptstyle 0.009}$ & \underline{46.79} & $\underline{0.109} \pm {\scriptstyle0.012}$ & 
$23.23\pm {\scriptstyle 0.739}$ & $\textbf{0.122}\pm {\scriptstyle 0.012}$ & \underline{9.11} & $\textbf{0.095}\pm {\scriptstyle 0.006}$ \\
\bottomrule
\end{tabular}
}
\vspace{-0.2cm}
\end{table}

\begin{table}[H]
\centering
\caption{\small{Quantitative results for motion deblurring with $\sigma = 0.05$ on CelebA and AFHQ. Best results are in \textbf{bold}.}}
\label{table:motion_deblur_errorbars}
\setlength{\tabcolsep}{4pt}
\scalebox{0.8}{
\begin{tabular}{@{} r cccc cccc @{}}
\toprule
\multirow{2}{*}{\textbf{Sampler}} 
& \multicolumn{4}{c}{\textbf{CelebA $128\times128$}} 
& \multicolumn{4}{c}{\textbf{AFHQ $256\times256$}} \\
\cmidrule(r){2-5} \cmidrule(l){6-9}
& \textbf{PSNR$\uparrow$} & \textbf{LPIPS$\downarrow$} & \textbf{FID$\downarrow$} & \textbf{DISTS$\downarrow$}
& \textbf{PSNR$\uparrow$} & \textbf{LPIPS$\downarrow$} & \textbf{FID$\downarrow$} & \textbf{DISTS$\downarrow$} \\
\midrule
Flow-priors & $25.07 \pm {\scriptstyle 1.34}$ & $0.151\pm {\scriptstyle .035}$ & 114.41 & $0.188\pm {\scriptstyle 0.021}$ &
$23.14\pm {\scriptstyle 2.26}$ & $\textbf{0.207} \pm {\scriptstyle 0.063}$ & 27.73 & $\underline{0.154}\pm {\scriptstyle 0.028}$ \\
D-Flow & 
$\textbf{29.77} \pm {\scriptstyle 1.67}$ & $\textbf{0.071}\pm {\scriptstyle 0.022}$ & 73.15 & $\textbf{0.112}\pm {\scriptstyle0.019}$ & 
$\textbf{25.52}\pm {\scriptstyle 2.60}$ & $\underline{0.230}\pm {\scriptstyle 0.085}$ & \underline{19.92} & $0.156 \pm {\scriptstyle 0.042}$ \\
OT-ODE & 
$21.07\pm {\scriptstyle 7.43}$ & $0.198\pm {\scriptstyle0.157}$ & 106.78 & $0.202\pm {\scriptstyle 0.113}$ & 
$16.73 \pm {\scriptstyle 5.51}$ & $0.348 \pm {\scriptstyle 0.165}$ & 77.53 & $0.234 \pm {\scriptstyle 0.122}$ \\
DPS-ODE & 
$27.36\pm {\scriptstyle5.16}$ & ${0.101}\pm {\scriptstyle 0.063}$ & 57.02 & ${0.133}\pm {\scriptstyle 0.054}$ & 
$20.49 \pm {\scriptstyle 3.66}$ & $0.292 \pm {\scriptstyle 0.123}$ & 36.71 & $0.179 \pm {\scriptstyle 0.60}$ \\
\midrule
PnP-Flow ($t = s$) & 
$\underline{29.79}\pm {\scriptstyle 2.70}$  & $0.114\pm {\scriptstyle 0.056}$ & \underline{53.65} & $0.130\pm {\scriptstyle 0.034}$ & 
$\underline{25.40} \pm {\scriptstyle 2.15}$& $0.347\pm {\scriptstyle 0.076}$ & 20.37 & $0.201 \pm {\scriptstyle 0.033}$ \\
PnP-Flow ($t = 1$) & 
${29.42} \pm {\scriptstyle 2.80}$& $\underline{0.082}\pm {\scriptstyle0.042}$ & \textbf{52.61} & $\underline{0.119}\pm {\scriptstyle 0.029}$ & 
$23.31\pm {\scriptstyle 2.50}$ & $\textbf{0.207}\pm {\scriptstyle 0.059}$ & \textbf{14.04} & $\textbf{0.137}\pm {\scriptstyle 0.028}$ \\
\bottomrule
\end{tabular}
}
\vspace{-0.2cm}
\end{table}

\subsection{Additional inverse problems}
\label{app:additional}

\paragraph{Gaussian deblurring.} We consider a kernel of size $31\times31$ for CelebA and $61\times61$ for AFHQ, with $\sigma_b=3$.
The quantitative results are in Table~\ref{table:gaussian_deblur}.

\paragraph{Super-resolution.} We consider super-resolution with a cubic method, and a factor of 4.
The quantitative results are in Table~\ref{table:super_resolution}.

\paragraph{Box-inpainting.} We consider a box of size $80\times 80$ for AFHQ and $60\times 60$ for CelebA.
The quantitative results are in Table~\ref{table:box_inp}.

\begin{table}[H]
\centering
\caption{\small{Quantitative results for Gaussian deblurring w/$\sigma = 0.05$ across CelebA and AFHQ datasets. Best results are in \textbf{bold}.}}
\label{table:gaussian_deblur}
\setlength{\tabcolsep}{4pt}
\scalebox{0.7}{
\begin{tabular}{@{} r cccc @{\hspace{15pt}} cccc @{}}
\toprule
\multirow{2}{*}{\textbf{Sampler}} 
& \multicolumn{4}{c}{\textbf{CelebA $128\times128$}} 
& \multicolumn{4}{c}{\textbf{AFHQ $256\times256$}} \\
\cmidrule(r){2-5} \cmidrule(l){6-9}
& \textbf{PSNR$\uparrow$} & \textbf{LPIPS$\downarrow$} & \textbf{FID$\downarrow$} & \textbf{DISTS$\downarrow$}
& \textbf{PSNR$\uparrow$} & \textbf{LPIPS$\downarrow$} & \textbf{FID$\downarrow$} & \textbf{DISTS$\downarrow$} \\
\midrule
Flow-priors 
& \textbf{30.41} & \textbf{0.044} & 80.53 & \underline{0.125} 
& {22.73} & 0.182 & 12.87 & 0.124 \\

D-Flow 
& 26.59 & \underline{0.081} & 66.93 & \textbf{0.120} 
& 22.80 & \underline{0.175} & 12.31 & \underline{0.121} \\

OT-ODE 
& 25.77 & {0.100} & \underline{63.10} & 0.129
& \underline{24.04} & \textbf{0.168} & \underline{11.51} & \textbf{0.114} \\

DPS-ODE 
& 24.80 & 0.129 & 72.13 & 0.148 
& 20.62 & 0.253 & 15.47 & 0.158 \\

\midrule
PnP-Flow ($t = s$)  
& \underline{27.25} & 0.233 & 70.80 & 0.183 
& \textbf{25.14} & 0.453 & 25.67 & 0.235 \\

PnP-Flow ($t = 1$) 
& {26.60} & {0.157} & \textbf{61.57} & {0.152} 
& 23.41 & 0.191 & \textbf{11.09} & 0.122 \\
\bottomrule
\end{tabular}
}
\vspace{-0.2cm}
\end{table}

\begin{table}[H]
\centering
\caption{\small{Quantitative results for super-resolution $\times4$ w/$\sigma = 0.05$ across CelebA and AFHQ datasets. Best results are in \textbf{bold}.}}
\label{table:super_resolution}
\setlength{\tabcolsep}{4pt}
\scalebox{0.7}{
\begin{tabular}{@{} r cccc @{\hspace{15pt}} cccc @{}}
\toprule
\multirow{2}{*}{\textbf{Sampler}} 
& \multicolumn{4}{c}{\textbf{CelebA $128\times128$}} 
& \multicolumn{4}{c}{\textbf{AFHQ $256\times256$}} \\
\cmidrule(r){2-5} \cmidrule(l){6-9}
& \textbf{PSNR$\uparrow$} & \textbf{LPIPS$\downarrow$} & \textbf{FID$\downarrow$} & \textbf{DISTS$\downarrow$}
& \textbf{PSNR$\uparrow$} & \textbf{LPIPS$\downarrow$} & \textbf{FID$\downarrow$} & \textbf{DISTS$\downarrow$} \\
\midrule
Flow-priors 
& 26.14 & 0.227 & 96.46 & 0.211 
& \underline{22.93} & 0.338 & 20.41 & 0.209 \\

D-Flow 
& 24.84 & \underline{0.137} & 119.93 & 0.181 
& 23.77 & 0.389 & 20.69 & 0.204 \\

OT-ODE 
& 24.79 & {0.142} & 81.18 & 0.173 
& 20.36 & 0.400 & 44.28 & 0.247 \\

DPS-ODE 
& 26.38 & \textbf{0.083} & \textbf{55.10} & \textbf{0.123} 
& 24.18 & \textbf{0.165} & \textbf{10.88} & \textbf{0.114} \\

\midrule
PnP-Flow ($t = s$)  
& \textbf{26.94} & 0.235 & 71.30 & 0.187 
& \textbf{25.10} & 0.429 & 26.45 & 0.232\\

PnP-Flow ($t = 1$) 
& \underline{26.68} & {0.166} & \underline{63.97} & \underline{0.158} 
& \underline{24.71} & \underline{0.269} & \underline{13.88} & \underline{0.159} \\
\bottomrule
\end{tabular}
}
\vspace{-0.2cm}
\end{table}

% \subsubsection{Box-inpainting}

\begin{table}[H]
\centering
\caption{\small{Quantitative results for box inpainting w/$\sigma = 0.01$ across CelebA and AFHQ datasets. Best results are in \textbf{bold}.}}
\label{table:box_inp}
\setlength{\tabcolsep}{4pt}
\scalebox{0.7}{
\begin{tabular}{@{} r cccc @{\hspace{15pt}} cccc @{}}
\toprule
\multirow{2}{*}{\textbf{Sampler}} 
& \multicolumn{4}{c}{\textbf{CelebA $128\times128$}} 
& \multicolumn{4}{c}{\textbf{AFHQ $256\times256$}} \\
\cmidrule(r){2-5} \cmidrule(l){6-9}
& \textbf{PSNR$\uparrow$} & \textbf{LPIPS$\downarrow$} & \textbf{FID$\downarrow$} & \textbf{DISTS$\downarrow$}
& \textbf{PSNR$\uparrow$} & \textbf{LPIPS$\downarrow$} & \textbf{FID$\downarrow$} & \textbf{DISTS$\downarrow$} \\
\midrule
Flow-priors 
& 25.09 & {0.062} & 36.28 & \underline{0.068} 
& \underline{26.51} & \underline{0.048} & 16.62 & \underline{0.052} \\

D-Flow 
& \underline{26.36} & \textbf{0.056} & 34.70 & 0.073 
& 26.87 & 0.069 & \underline{11.08} & 0.079 \\

OT-ODE 
& 24.65 & {0.065} & \underline{31.99} & 0.073 
& 20.80 & 0.085 & 25.97 & 0.095 \\

DPS-ODE 
& 24.13 & 0.115 & 58.64 & 0.132
& 21.07 & 0.069 & 19.48 & 0.071 \\

\midrule
PnP-Flow ($t = s$)  
& \textbf{26.77} & \underline{0.059} & 32.29 & \textbf{0.062} 
& \textbf{28.33} & 0.057 & 13.79 & \underline{0.052} \\

PnP-Flow ($t = 1$) 
& {24.77} & {0.062} & \textbf{26.07} & \textbf{0.062} 
& 26.40 & \textbf{0.031} & \textbf{7.61} & \textbf{0.033} \\
\bottomrule
\end{tabular}
}
\vspace{-0.2cm}
\end{table}

\subsubsection{Non-linear}
\label{app:nonlinear}
\paragraph{JPEG.}
We consider JPEG restoration with quantization factor $10$. Since OT-ODE and Flow-Priors cannot handle non-linear degradations, we only compare with DPS-ODE; the results on AFHQ are in Table~\ref{table:non_linear_jpeg}.
Fig.~\ref{fig:dp_jpeg} shows the effect of the number of PnP iterations: the perceptual endpoint consistently achieves better perceptual quality, although increasing the number of iterations eventually degrades FID, as in the linear problems (Appendix~\ref{app:num_steps_ablation}).

\begin{figure}[H]
    \centering
    \includegraphics[width=0.5\linewidth]{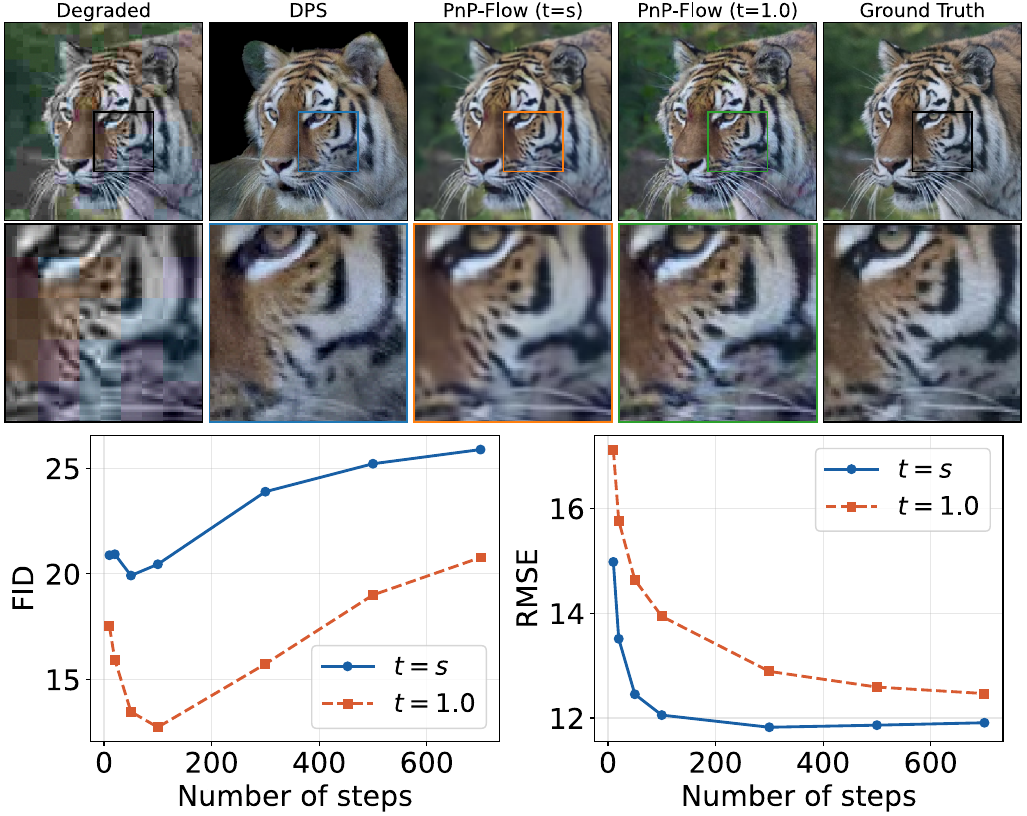}
    \caption{DP behavior on JPEG restoration (AFHQ) as a function of the number of PnP iterations.}
    \label{fig:dp_jpeg}
\end{figure}

\begin{table}[H]
\centering
\caption{\small{Quantitative results for JPEG w/$\sigma = 0.01$ on AFHQ datasets. Best results are in \textbf{bold}.}}
\label{table:non_linear_jpeg}
\setlength{\tabcolsep}{4pt}
\scalebox{0.7}{
\begin{tabular}{@{} r cccc @{}}
\toprule
\multirow{2}{*}{\textbf{Sampler}} & \multicolumn{4}{c}{\textbf{AFHQ $256\times256$}} \\
\cmidrule(r){2-5}
& \textbf{PSNR$\uparrow$} & \textbf{LPIPS$\downarrow$} & \textbf{FID$\downarrow$} & \textbf{DISTS$\downarrow$} \\
\midrule
DPS-ODE & 18.76 & 0.287 & 19.05 & 0.184 \\
\midrule
PnP-Flow ($t = s$)  & \textbf{26.79} & 0.245 & 20.33 & 0.166 \\
PnP-Flow ($t = 1$)  & 25.52 & \textbf{0.143} & \textbf{12.84} & \textbf{0.124} \\
\bottomrule
\end{tabular}
}
\vspace{-0.2cm}
\end{table}

\paragraph{Colorization.}
We consider also Colorization; the result is shown in Table~\ref{table:non_linear_col} for AFHQ.

\begin{table}[H]
\centering
\caption{\small{Quantitative results for Colorization w/$\sigma = 0.001$ on AFHQ datasets. Best results are in \textbf{bold}.}}
\label{table:non_linear_col}
\setlength{\tabcolsep}{4pt}
\scalebox{0.7}{
\begin{tabular}{@{} r cccc @{}}
\toprule
\multirow{2}{*}{\textbf{Sampler}} & \multicolumn{4}{c}{\textbf{AFHQ $256\times256$}} \\
\cmidrule(r){2-5}
& \textbf{PSNR$\uparrow$} & \textbf{LPIPS$\downarrow$} & \textbf{FID$\downarrow$} & \textbf{DISTS$\downarrow$} \\
\midrule
DPS-ODE & 18.60 & 0.413 & 33.79 & 0.236 \\
\midrule
PnP-Flow ($t = s$)  & \textbf{26.30} & 0.111 & 9.95 & 0.110 \\
PnP-Flow ($t = 1$)  & 25.37 & \textbf{0.105} & \textbf{9.47} & \textbf{0.094} \\
\bottomrule
\end{tabular}
}
\vspace{-0.2cm}
\end{table}

\subsubsection{Gaussian deblurring with Poisson noise}
\label{app:poisson}

Lastly, we consider Poisson noise, to show that the method is not limited to Gaussian measurement noise.
We normalize each image to $[0,1]$ and consider independent Poisson observations normalized by a peak rate $\lambda > 0$:
\begin{equation}
  y_i \,\big|\, \bbx \;\sim\; \mathrm{Poisson}\!\big(\lambda\,(\bbH\bbx)_i\big),
  \qquad i = 1, \ldots, N,
  \label{eq:poisson_obs}
\end{equation}
so that $\mathbb{E}[y_i \mid \bbx] = (\bbH\bbx)_i$ and $\mathrm{Var}[y_i \mid \bbx] = (\bbH\bbx)_i / \lambda$. 
The signal-to-noise ratio scales as $\sqrt{\lambda}$, where lower $\lambda$ corresponds to fewer photons and a noisier observation; we consider $\lambda = 20$.
Given this noise model, the corresponding negative log-likelihood, used as the data-fidelity term in Alg.~\ref{alg:pnp_fm}, is (up to a constant in $\bbx$)
\begin{equation}
  g(\bbx)
  \;=\; \sum_{i=1}^{N} \big[\,(\bbH\bbx)_i \,-\, y_i \log (\bbH\bbx)_i\,\big].
  \label{eq:poisson_nll}
\end{equation}
The results as a function of the number of steps is shown in Fig.~\ref{fig:ablation_num_steps_poisson}, while qualitative comparisons are in Fig.~\ref{fig:poisson_visual}.

\begin{figure}[H]
    \centering
    \includegraphics[width=0.7\linewidth]{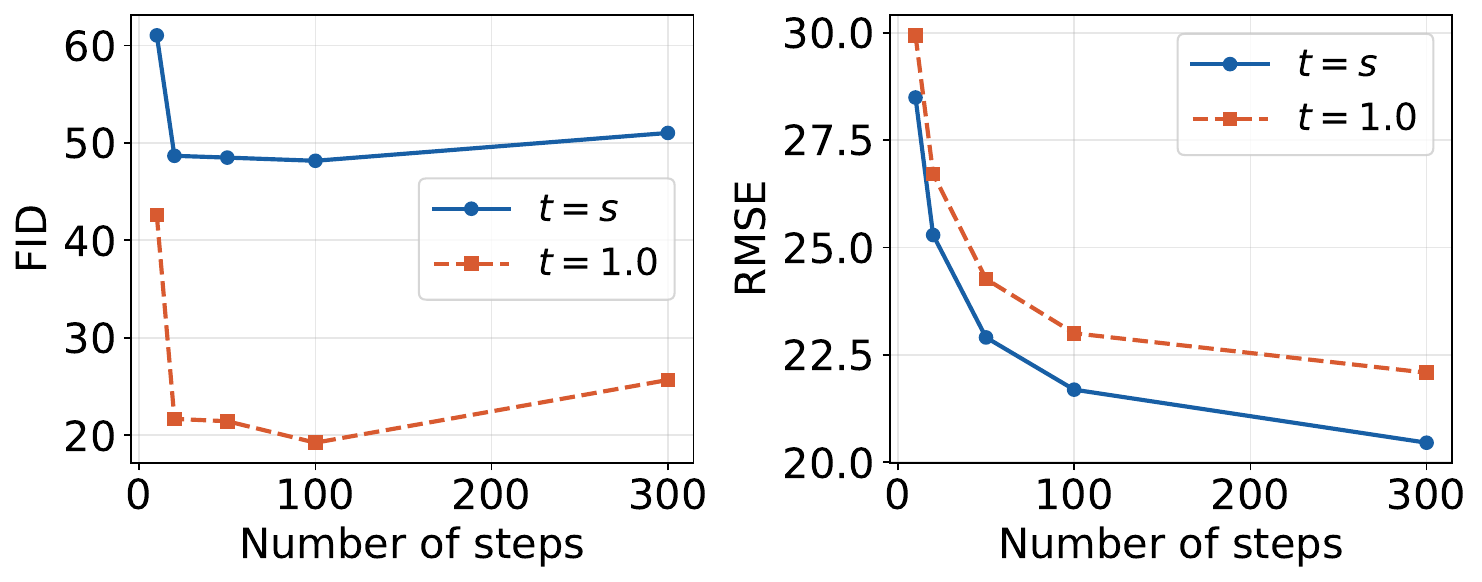}
    \caption{\small{Ablation of PnP-flow as a function of the number of steps with Poisson noise and Gaussian deblurring}}
    \label{fig:ablation_num_steps_poisson}
\end{figure}

\subsection{Distortion-perception of CelebA}
Fig.~\ref{fig:pd_celeba} reproduces the experiment of Section~\ref{subsec:dp_exp} on CelebA.
Again, varying $t$ produces a smooth, nearly monotonic trajectory in the DP plane: small lookaheads minimize RMSE, large lookaheads minimize FID, and intermediate values interpolate between the two.

\begin{figure}[H]
    \centering
    \includegraphics[width=1.0\linewidth]{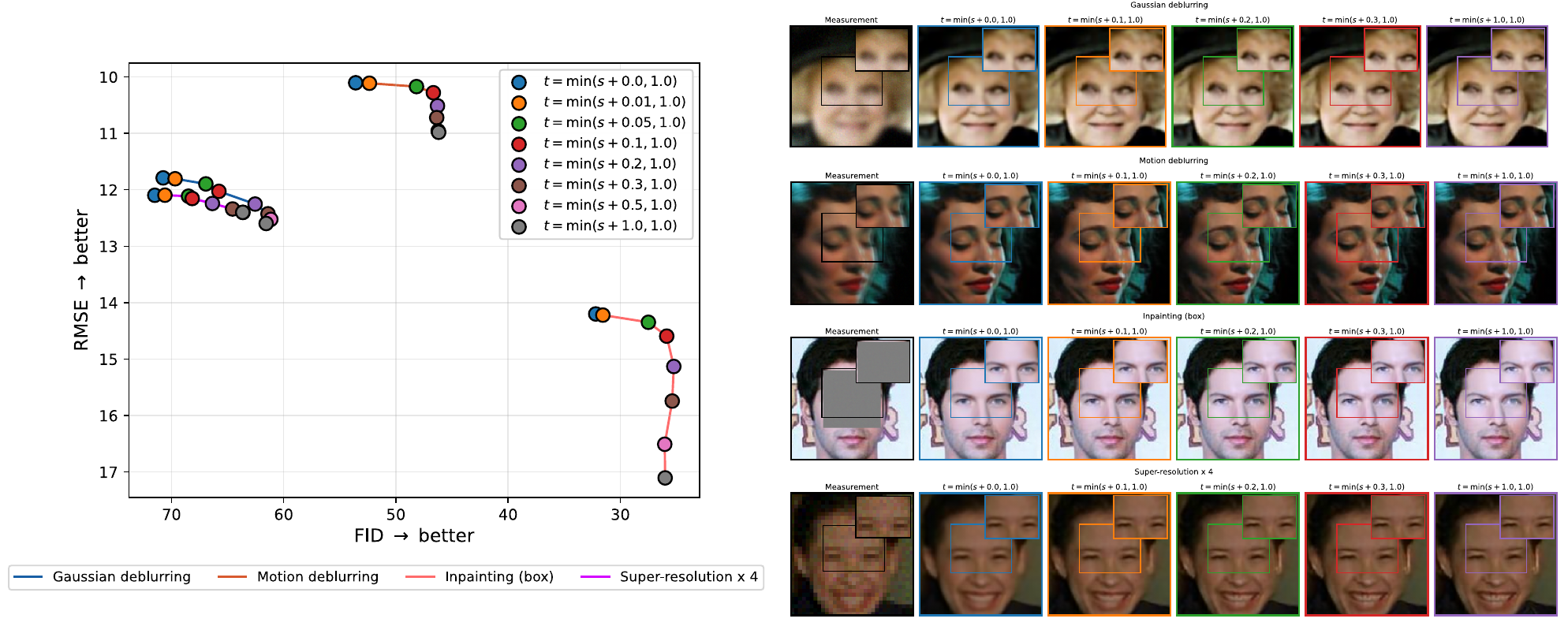}
    \caption{{\small \textbf{Distortion-perception tradeoff traced by a single flow map on CelebA.} Same setup as Fig.~\ref{fig:DP_exp}, but on CelebA $128\times128$. For each of four linear inverse problems (Gaussian deblurring, motion deblurring, box inpainting, and super-resolution), we sweep the lookahead $t$ of our average denoiser within a PnP solver and report RMSE (distortion, $\downarrow$) against FID (perception, $\downarrow$). Each marker corresponds to a different value of $t$; small lookaheads cluster near the low-distortion regime (top of each curve) while large lookaheads cluster near the low-FID regime (right). The same trained model traverses the DP plane in all four tasks. Reconstructions on the right illustrate the perceptual change as $t$ increases.}}
    \label{fig:pd_celeba}
\end{figure}

\begin{figure}[H]
    \centering
    \includegraphics[width=0.8\linewidth]{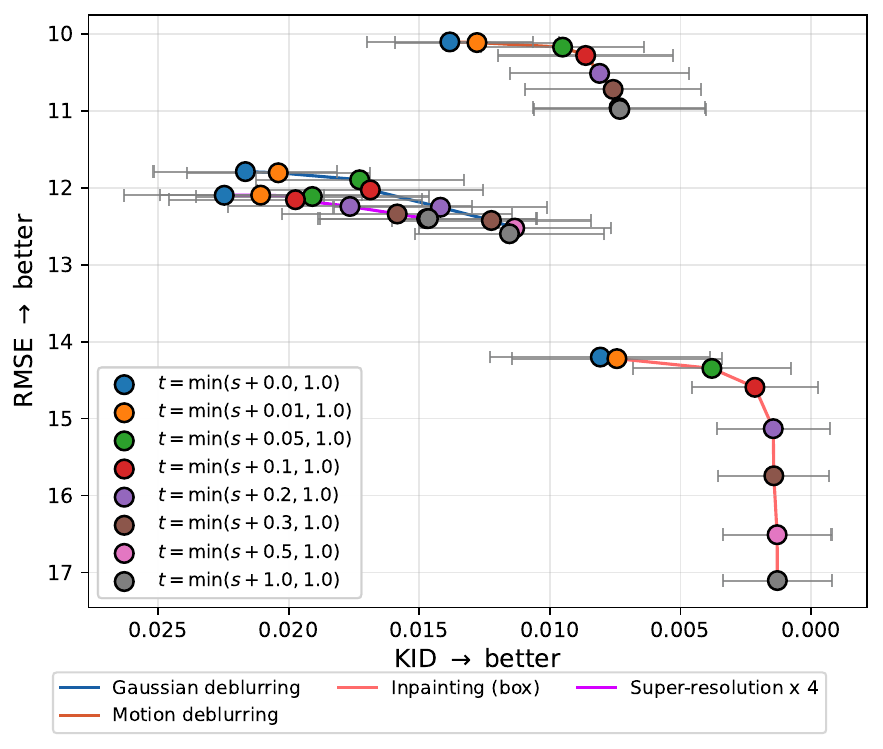}
    \caption{{\small \textbf{Distortion-perception tradeoff traced by a single flow map on CelebA.} Same setup as Fig.~\ref{fig:pd_celeba}, reporting KID instead of FID.}}
    \label{fig:kid_celeba}
\end{figure}

\subsection{Comparison with baselines using a flow matching model}
\label{app:flow_matching_comparison}

Table~\ref{table:comparison_w_flow_matching} compares all methods under both a flow matching backbone and a flow map backbone, allowing us to isolate the contribution of the flow map training objective from the choice of sampler.

\paragraph{Flow matching backbone.}
PnP-Flow ($t = s$) achieves the highest PSNR ($25.36$ dB), but at the cost of perceptual quality relative to Flow-priors (LPIPS $0.359$ vs.\ $0.205$). PnP-Flow ($t = 1$) is inapplicable here, as it requires the consistency property enforced only by flow map training.

\paragraph{Impact of the backbone.}
The results are largely consistent across backbones, with DPS-ODE achieving a lower FID with the flow matching model.
However, the flow map backbone with PnP-flow ($t = 1$) still yields notably lower FID scores throughout and is required to unlock PnP-flow ($t = 1$).

\begin{table}[H]
\centering
\caption{\small{Quantitative results for motion deblurring w/$\sigma = 0.05$ across AFHQ dataset. Best results are in \textbf{bold}.}}
\label{table:comparison_w_flow_matching}
\setlength{\tabcolsep}{4pt}
\scalebox{0.7}{
\begin{tabular}{@{} r cccc @{\hspace{15pt}} cccc @{}}
\toprule
\multirow{2}{*}{\textbf{Sampler}} 
& \multicolumn{4}{c}{\textbf{Flow Matching backbone}} 
& \multicolumn{4}{c}{\textbf{Flow Map backbone}} \\
\cmidrule(r){2-5} \cmidrule(l){6-9}
& \textbf{PSNR$\uparrow$} & \textbf{LPIPS$\downarrow$} & \textbf{FID$\downarrow$} & \textbf{DISTS$\downarrow$}
& \textbf{PSNR$\uparrow$} & \textbf{LPIPS$\downarrow$} & \textbf{FID$\downarrow$} & \textbf{DISTS$\downarrow$} \\
\midrule
Flow-priors 
& 23.20 & {0.205} & 18.48 & \underline{0.153} 
& \underline{23.14} & \underline{0.207} & 18.96 & \underline{0.154} \\

D-Flow 
&  &  &   & 
& 25.52 & 0.230 & \underline{19.92} & 0.156 \\

OT-ODE 
& 17.2 & {0.335} & 68.71  & 0.224
& 17.73 & 0.348 & 77.53 & 0.234 \\

DPS-ODE 
& 20.62 & 0.262 & 25.71 & 0.162
& 20.49 & 0.292 & 36.71 & 0.179 \\

\midrule
PnP-Flow ($t = s$)  
& 25.36 &  0.359  & 20.90 &  0.206
& \textbf{25.40} & 0.347 & 20.37 & \underline{0.201} \\

PnP-Flow ($t = 1$) 
& - & - & - & -
& 23.31 & \textbf{0.207} & \textbf{14.04} & \textbf{0.137} \\
\bottomrule
\end{tabular}
}
\vspace{-0.2cm}
\end{table}

\subsection{Comparison with training-free DP-traversal methods}
\label{app:dp_baselines}

We compare our lookahead sweep with two training-free methods that traverse the DP plane with a single pre-trained model: \citet{wang2025traversing}, which scales the noise injected in the reverse SDE by a parameter $\lambda$, and MAP-RPS~\citep{zhang2026stagewise}, which sets the time $t_0$ at which a re-noised posterior-sampling stage starts from a MAP initialization.

\paragraph{Setup.}
Since no public code is available for \citet{wang2025traversing}, we implemented their sampler; for MAP-RPS, we adapted the official code. Both baselines share the same backbone: a flow matching model with the same architecture, data, and training pipeline as our flow map, converted into a score model through the affine relation between velocity and score for Gaussian paths. Their hyperparameters are tuned starting from the configurations reported in the respective papers. For our method, all solver hyperparameters are fixed and only the lookahead offset $\Delta$ varies; we use $N = 50$ iterations for Gaussian deblurring and $N = 300$ for random inpainting. All methods are evaluated on the same $50$ test images.

\begin{table}[H]
\centering
\caption{\small{Comparison with training-free DP-traversal methods on AFHQ $256\times256$, at three values of each method's knob. RMSE and FID are both lower-is-better.}}
\label{tab:dp_baselines}
\resizebox{\linewidth}{!}{%
\begin{tabular}{llcccc}
\toprule
& & \multicolumn{2}{c}{Gaussian deblurring ($\sigma=0.05$)} & \multicolumn{2}{c}{Random inpainting ($\sigma=0.01$)} \\
\cmidrule(lr){3-4} \cmidrule(lr){5-6}
Method & Knob values & RMSE$\downarrow$ & FID$\downarrow$ & RMSE$\downarrow$ & FID$\downarrow$ \\
\midrule
\citet{wang2025traversing} & $\lambda \in \{0, 0.5, 1\}$ & 23.85 / 21.84 / 23.24 & 138.78 / 49.57 / 27.52 & 39.58 / 17.46 / 18.12 & 357.35 / 43.59 / 25.75 \\
MAP-RPS & $t_0 \in \{0, 50, 100\}$ & 15.36 / 15.48 / 15.86 & 41.80 / 31.18 / 26.80 & 21.44 / 19.31 / 18.77 & 55.52 / 32.82 / 27.33 \\
Ours & $t \in \{s, s{+}0.4, 1\}$ & 15.94 / 17.61 / 18.03 & 37.83 / 24.63 / 24.48 & 15.47 / 18.91 / 18.94 & 27.84 / 24.22 / 24.23 \\
\bottomrule
\end{tabular}}
\end{table}

\begin{figure}[H]
    \centering
    \includegraphics[width=\linewidth]{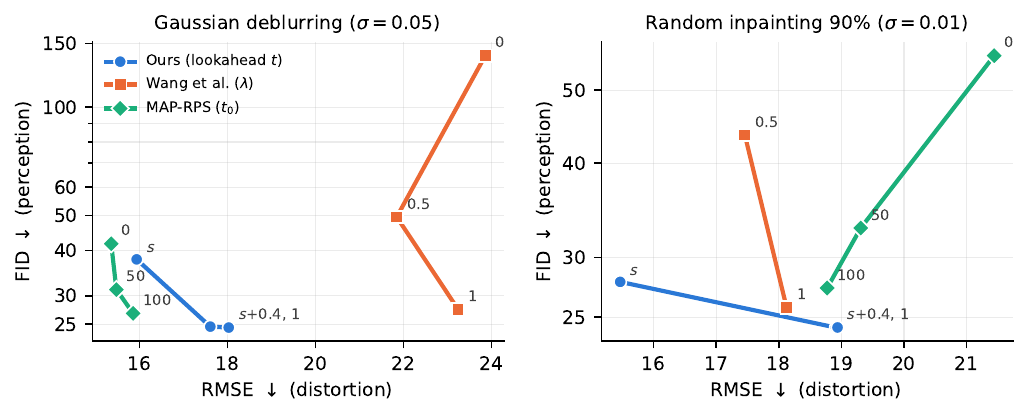}
    \caption{\small{Operating points of Table~\ref{tab:dp_baselines} in the RMSE--FID plane (FID on a log scale), labeled by each method's knob. For random inpainting, \citet{wang2025traversing} at $\lambda = 0$ lies outside the axes (RMSE $39.58$, FID $357.35$).}}
    \label{fig:dp_baselines}
\end{figure}

\paragraph{Results.}
Table~\ref{tab:dp_baselines} and Fig.~\ref{fig:dp_baselines} report the results. On Gaussian deblurring, our operating points Pareto-dominate those of \citet{wang2025traversing}. On random inpainting, their $\lambda=0$ endpoint collapses (FID $357$; the deterministic reverse mean fails on the mask, and we could not recover it by tuning), and their sweep does not reach the low-distortion end (best RMSE $17.46$ versus $15.47$ for ours). Compared with MAP-RPS, the curves cross on Gaussian deblurring: MAP-RPS attains lower distortion and we attain lower FID. On random inpainting, our curve lies below and to the left of theirs, i.e., it is better on both axes.
In terms of runtime, processing a batch of $5$ images takes $12$\,s for our method, $40$\,s for \citet{wang2025traversing}, and $60$\,s for MAP-RPS, since both baselines backpropagate through the network at every step.
Beyond performance, the two families differ in where the knob acts. The baselines modify the sampler and retune it per operating point (e.g., \citet{wang2025traversing} retune their guidance step per $\lambda$ and per task, and MAP-RPS adjusts its sampling stage per $t_0$), whereas our sweep uses a single solver configuration for all $t$ and changes only the conditioning of the denoiser (see also Section~\ref{sec:related}).

\paragraph{Other DP-traversal methods.}
\citet{wang2026training} extend the approach of \citet{wang2025traversing} to rate--distortion--perception traversal with a score-scaled probability-flow ODE, and prove Gaussian optimality for the conditional problem, complementary to our Theorem~\ref{thm:optimal}.
The posterior-sampling conditional GAN of \citet{ohayon2021high} is a natural baseline, but no model is available for AFHQ or CelebA. Both of its mechanisms (scaling the latent noise and averaging samples) traverse between a posterior sampler and its conditional mean, placing it in the same conditional family as the two baselines above, both of which include it as a baseline in their experiments. OFTSR~\citep{zhu2026oftsr}, InDI~\citep{delbracio2023inversion}, and PMRF~\citep{ohayonposterior} are trained on paired data for a specific degradation, whereas our unconditional prior is reused across operators without retraining; a fair comparison would require training a task-specific model for each of our settings, which we leave for future work.

\subsection{Endpoint interpolation versus lookahead}
\label{app:interp_vs_lookahead}

Appendix~\ref{app:multivariate} shows that, for Gaussian targets, the convex combination~\eqref{eq:endpoint_interp} of the two endpoint denoisers traces the optimal DP curve exactly, even for anisotropic covariances. For images and inverse problems, however, this certificate no longer applies: its optimality relies on $D_{s,1}$ being the minimum-distortion perfect-perception estimator, i.e., on the coupling induced by the flow-map ODE being the optimal transport coupling, which need not hold in general. The two constructions then trace different paths between the same endpoints: the convex combination follows the pixel-space segment between the endpoint outputs, while the intermediate-$t$ evaluation is a nonlinear network output. Moreover, the convex combination requires two network evaluations per iteration.

We compare both constructions within the same PnP solver on AFHQ Gaussian deblurring ($200$ images, $\sigma = 0.05$), replacing the denoiser in Algorithm~\ref{alg:pnp_fm} by $(1-\lambda)D_{s,1} + \lambda D_{s,s}$ for the interpolation (Fig.~\ref{fig:interp_vs_lookahead}).
At the same parameter value ($\lambda = \Delta = 0.5$), the two constructions land at very different operating points: the lookahead is already deep in the perceptual regime (KID $1.68 \times 10^{-3}$, RMSE $17.26$), whereas the interpolation remains near the distortion end (KID $6.78 \times 10^{-3}$, RMSE $15.51$); matching the latter requires $\Delta \approx 0.1$. At matched distortion (RMSE $15.54$ versus $15.51$), the lookahead attains a better KID ($3.64$ versus $6.78 \times 10^{-3}$; LPIPS $0.354$ versus $0.394$) at half the runtime ($\approx 8.3$ versus $\approx 16.4$\,s per batch), and its curve lies below the interpolation curve at every intermediate operating point, for both metrics.
The lookahead therefore parametrizes the perceptual regime much faster than the interpolation, at no loss in quality and half the cost. A systematic comparison of both constructions is an interesting direction for future work.

\begin{figure}[H]
    \centering
    \includegraphics[width=\linewidth]{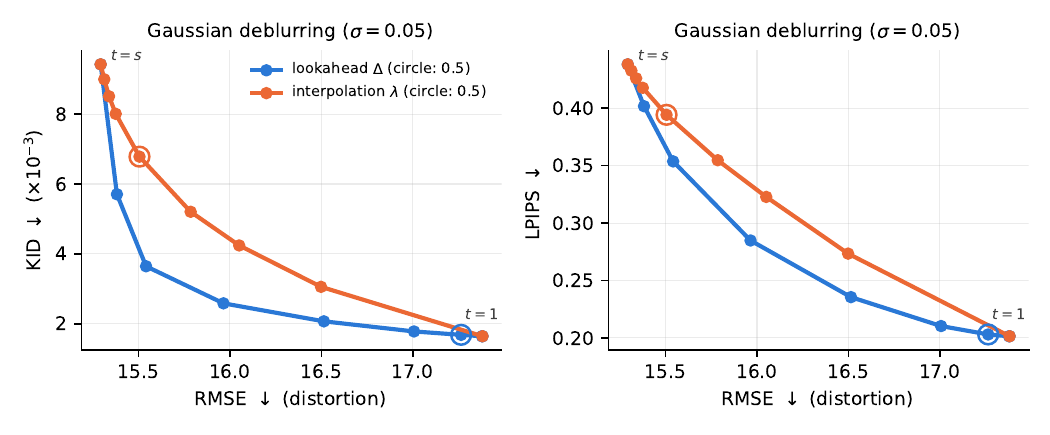}
    \caption{\small{Endpoint interpolation versus lookahead within the same PnP solver (AFHQ Gaussian deblurring, $200$ images; KID against the AFHQ test set), with $\lambda \in \{0, 0.1, 0.2, 0.3, 0.5, 0.7, 0.8, 0.9, 1\}$ and $\Delta \in \{0, 0.05, 0.1, 0.2, 0.3, 0.4, 0.5, 0.7, 1\}$. Both curves share the endpoints $t = s$ and $t = 1$; circles mark $\lambda = \Delta = 0.5$.}}
    \label{fig:interp_vs_lookahead}
\end{figure}

\subsection{Renoising magnitude}
\label{app:renoise_ablation}

In Algorithm~\ref{alg:pnp_fm}, the injected noise $(1-s_k)\bbepsilon$ has a single valid magnitude: it reprojects the iterate onto the interpolant level $s_k$ at which the denoiser is trained. To verify that it cannot replace the lookahead as a DP knob, we scale it by a factor $c$, i.e., $\tilde\bbz^k = c\,(1-s_k)\,\bbepsilon + s_k\,\bbz^k$, with $\Delta = 0$ and all other hyperparameters fixed, on AFHQ Gaussian deblurring and random inpainting ($200$ images, same solver settings as in Appendix~\ref{app:dp_baselines}). Fig.~\ref{fig:renoise_ablation} compares the resulting operating points with the lookahead sweep.
Reducing the noise ($c < 1$) degrades perceptual quality on both tasks without a meaningful gain in distortion (e.g., at $c = 0.75$, KID $11.04$ versus $9.43 \times 10^{-3}$ on deblurring, and $4.93$ versus $2.96 \times 10^{-3}$ on inpainting). Increasing it ($c > 1$) moves deblurring toward the perceptual regime, but every operating point is dominated by the lookahead sweep (e.g., at $c = 1.5$, RMSE $17.21$ and KID $3.10 \times 10^{-3}$, versus $17.26$ and $1.68 \times 10^{-3}$ for $\Delta = 0.5$), and the iteration degrades on both axes at $c = 2$ (RMSE $22.40$). On inpainting, $c > 1$ breaks the reconstruction (RMSE $28.90$ at $c = 1.25$ and $67.15$ at $c = 1.5$). LPIPS leads to the same conclusions (Fig.~\ref{fig:renoise_ablation}, bottom). The renoising magnitude therefore does not provide a useful traversal, and the lookahead remains the only degree of freedom.

\begin{figure}[H]
    \centering
    \includegraphics[width=\linewidth]{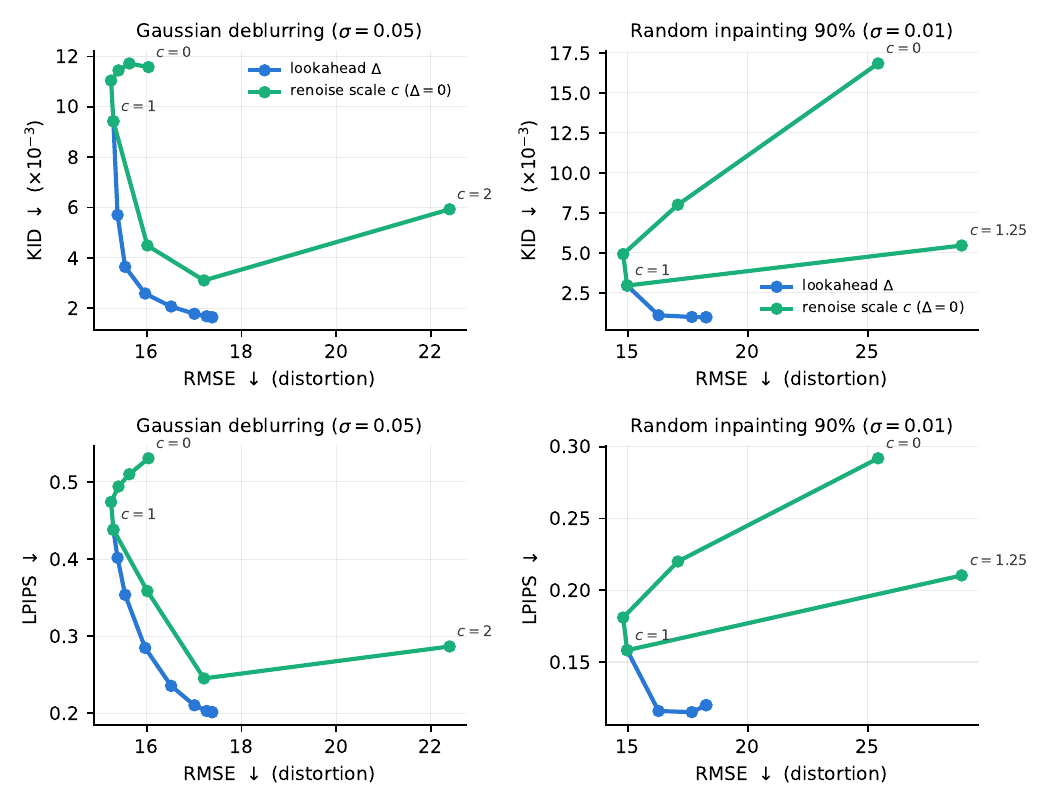}
    \caption{\small{Scaling the injected noise by $c \in \{0, 0.25, 0.5, 0.75, 1, 1.25, 1.5, 2\}$ at $\Delta = 0$ (deblurring; $c \in \{0, 0.5, 0.75, 1, 1.25, 1.5\}$ for inpainting), compared with the lookahead sweep ($200$ images). Top: KID; bottom: LPIPS. $c = 1$ coincides with $\Delta = 0$. For inpainting, $c = 1.5$ lies outside the axes (RMSE $67.15$, KID $130.4 \times 10^{-3}$, LPIPS $0.589$).}}
    \label{fig:renoise_ablation}
\end{figure}

\subsection{Measurement consistency along the lookahead}
\label{app:measurement_residual}

A perceptual gain obtained by drifting away from the measurements would show up as a growing residual. We measure $\|\bby - \bbH\hat\bbx(t)\|$ along the lookahead sweep ($200$ images per setting), as a root mean square over the observed entries, normalized by the noise level $\sigma$ (Fig.~\ref{fig:measurement_residual}).
For Gaussian deblurring and random inpainting, the residual does not depend on $t$: it stays at the noise level for deblurring ($1.00\,\sigma$) and below it for random inpainting ($0.41$--$0.43\,\sigma$). For the block masks of Appendix~\ref{app:failures_limitations}, it increases slightly for small lookaheads and saturates by $\Delta \approx 0.3$; the largest increase ($0.23\,\sigma$, for $1\times1$ blocks) corresponds to about $0.3$ gray levels on the observed pixels, whereas the RMSE increases by $4.3$ gray levels over the same sweep. The perceptual gains therefore do not come from weakened data consistency.

\begin{figure}[H]
    \centering
    \includegraphics[width=0.9\linewidth]{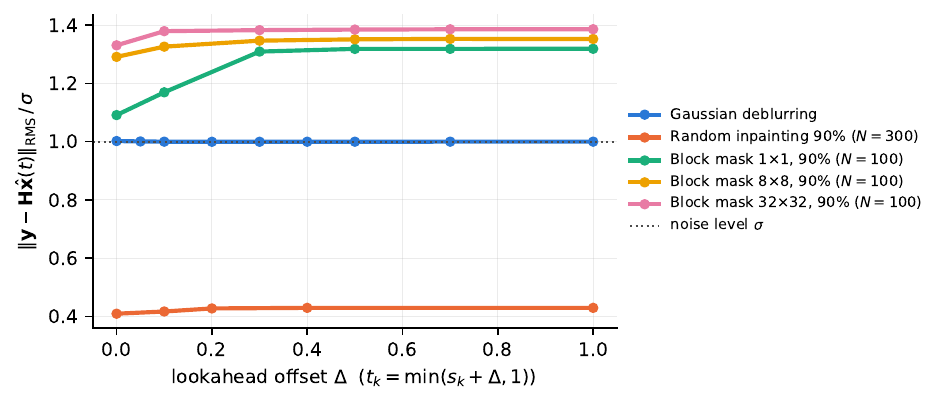}
    \caption{\small{Measurement residual along the lookahead sweep on AFHQ, normalized by the noise level $\sigma$ (RMS over observed entries). The residual of the clean image equals $1.00\,\sigma$ in all settings.}}
    \label{fig:measurement_residual}
\end{figure}

\subsection{Convexity of the denoising DP curves and flow-map accuracy}
\label{app:convexity}

In Fig.~\ref{fig:pd_curves}, the DP curves are convex at low noise and nearly linear at the highest noise level. Since the convexity of the ideal DP curve does not depend on the noise level, we attribute this to a deviation of the learned model rather than to the DP characterization: reaching the perceptual endpoint requires traversing the long interval $[s,1]$, where the single-step flow map is least accurate, so the deviation should grow at large $\xi$ (small $s$) and depend on the accuracy of the flow map.

To test this, we evaluate the average denoiser with $K \in \{1, 2, 4, 8\}$ flow-map steps over $[s,t]$ and measure $(i)$ the FID at the perceptual endpoint and $(ii)$ the convexity of the DP curve, quantified by the normalized signed area between the curve and the chord joining its endpoints (positive for convex, negative for concave). Results are in Table~\ref{tab:convexity}.

\begin{table}[H]
\centering
\caption{\small{Convexity of the denoising DP curve on CelebA (normalized signed area; positive means convex) and FID at the perceptual endpoint, as a function of the number $K$ of flow-map steps used to evaluate the average denoiser.}}
\label{tab:convexity}
\small
\begin{tabular}{ccccccc}
\toprule
& \multicolumn{4}{c}{Signed area} & \multicolumn{2}{c}{Endpoint FID} \\
\cmidrule(lr){2-5} \cmidrule(lr){6-7}
$\xi$ & $K=1$ & $K=2$ & $K=4$ & $K=8$ & $K=1$ & $K=8$ \\
\midrule
0.1 & $+0.23$ & $+0.17$ & $+0.18$ & $+0.20$ & 16.3 & 16.8 \\
0.3 & $+0.08$ & $+0.05$ & $+0.15$ & $+0.14$ & 36.9 & 34.1 \\
0.5 & $-0.06$ & $+0.06$ & $-0.00$ & $-0.04$ & 52.0 & 46.1 \\
0.7 & $-0.17$ & $-0.11$ & $-0.15$ & $-0.17$ & 68.4 & 57.8 \\
\bottomrule
\end{tabular}
\end{table}

Additional flow-map steps improve the endpoint FID at moderate and high noise (e.g., at $\xi=0.7$, FID decreases from $68.4$ to $62.7$, $59.9$, and $57.8$ for $K = 1, 2, 4, 8$), which confirms the single-step map as a source of approximation error; at $\xi=0.1$, the endpoint is nearly saturated and the change is within evaluation noise. The added accuracy restores convexity at moderate noise ($\xi = 0.3$ becomes clearly convex), but a residual concavity remains at the largest noise level that does not diminish with $K$. We attribute it to the error of the trained flow map in the heavy-noise regime, rather than to single-step discretization or to the DP characterization itself.

\subsection{Visual results}
\label{app:visual_results}

\begin{figure}[H]
    \centering
    \includegraphics[width=\linewidth]{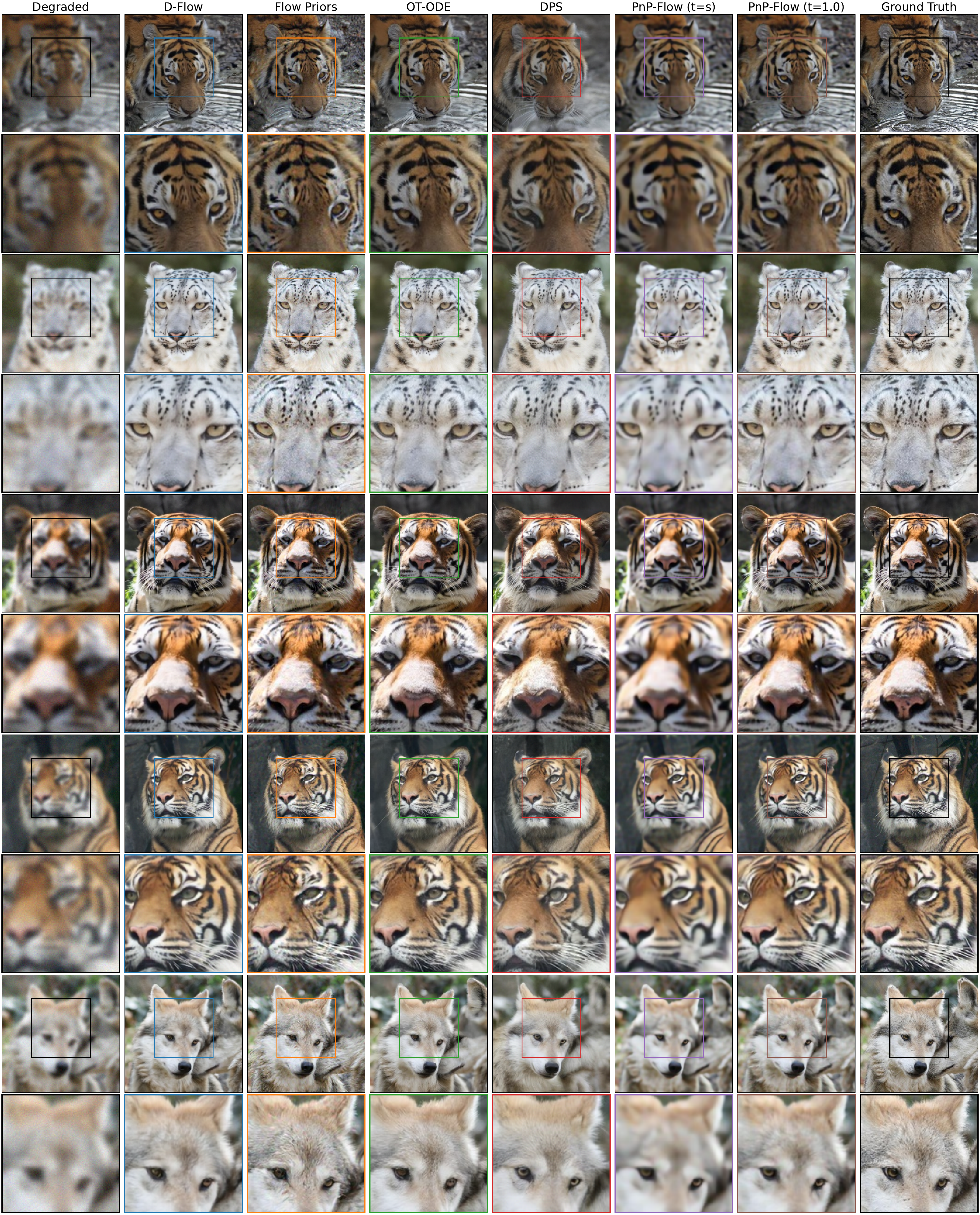}
    \caption{Gaussian deblurring with kernel size $61\times 61$ and $\sigma_b =3 $, and $\sigma =  0.05$}
    \label{fig:afhq_gauss}
\end{figure}

\begin{figure}[H]
    \centering
    \includegraphics[width=\linewidth]{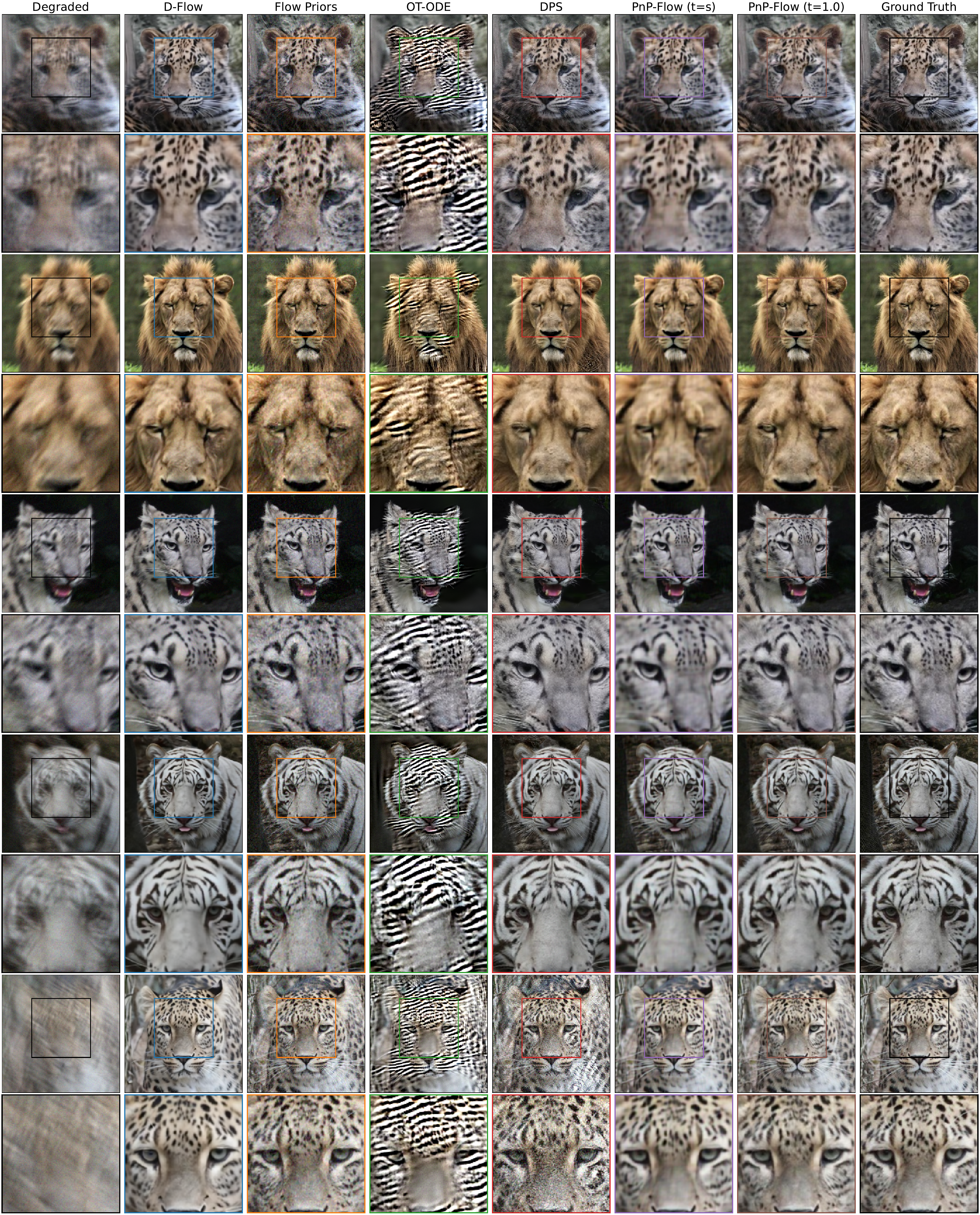}
    \caption{Motion deblurring with kernel size $61\times 61$ and $\sigma_b = 1 $, and $\sigma =  0.1$}
    \label{fig:afhq_motion}
\end{figure}

\begin{figure}[H]
    \centering
    \includegraphics[width=1.0\linewidth]{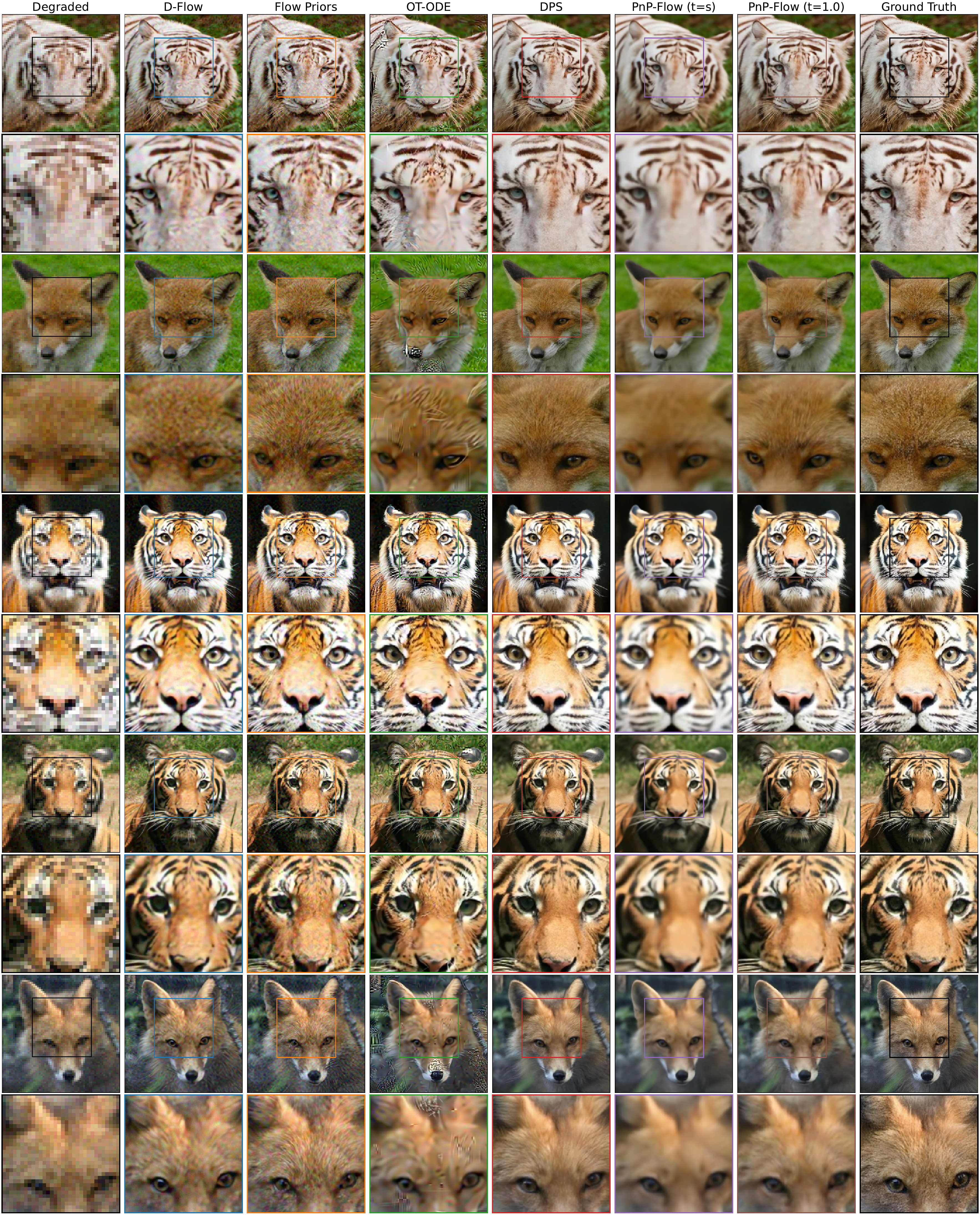}
    \caption{Super resolution $\times 4$, and $\sigma =  0.05$}
    \label{fig:afhq_sr}
\end{figure}

\begin{figure}[H]
    \centering
    \includegraphics[width=1.0\linewidth]{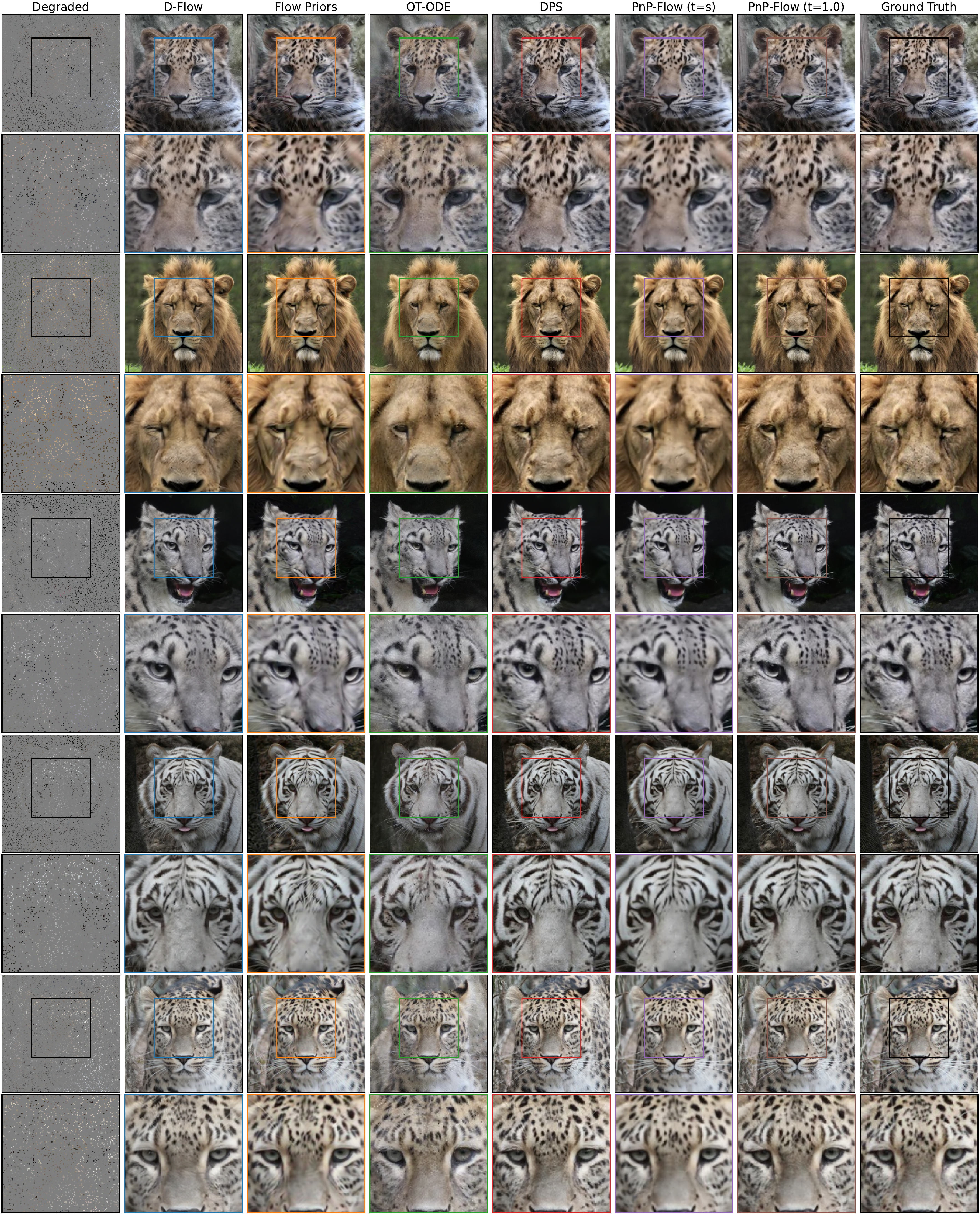}
    \caption{Random inpainting with $90\%$ drop rate, and $\sigma = 0.01$}
    \label{fig:afhq_rand_inp}
\end{figure}

\begin{figure}[H]
    \centering
    \includegraphics[width=1.0\linewidth]{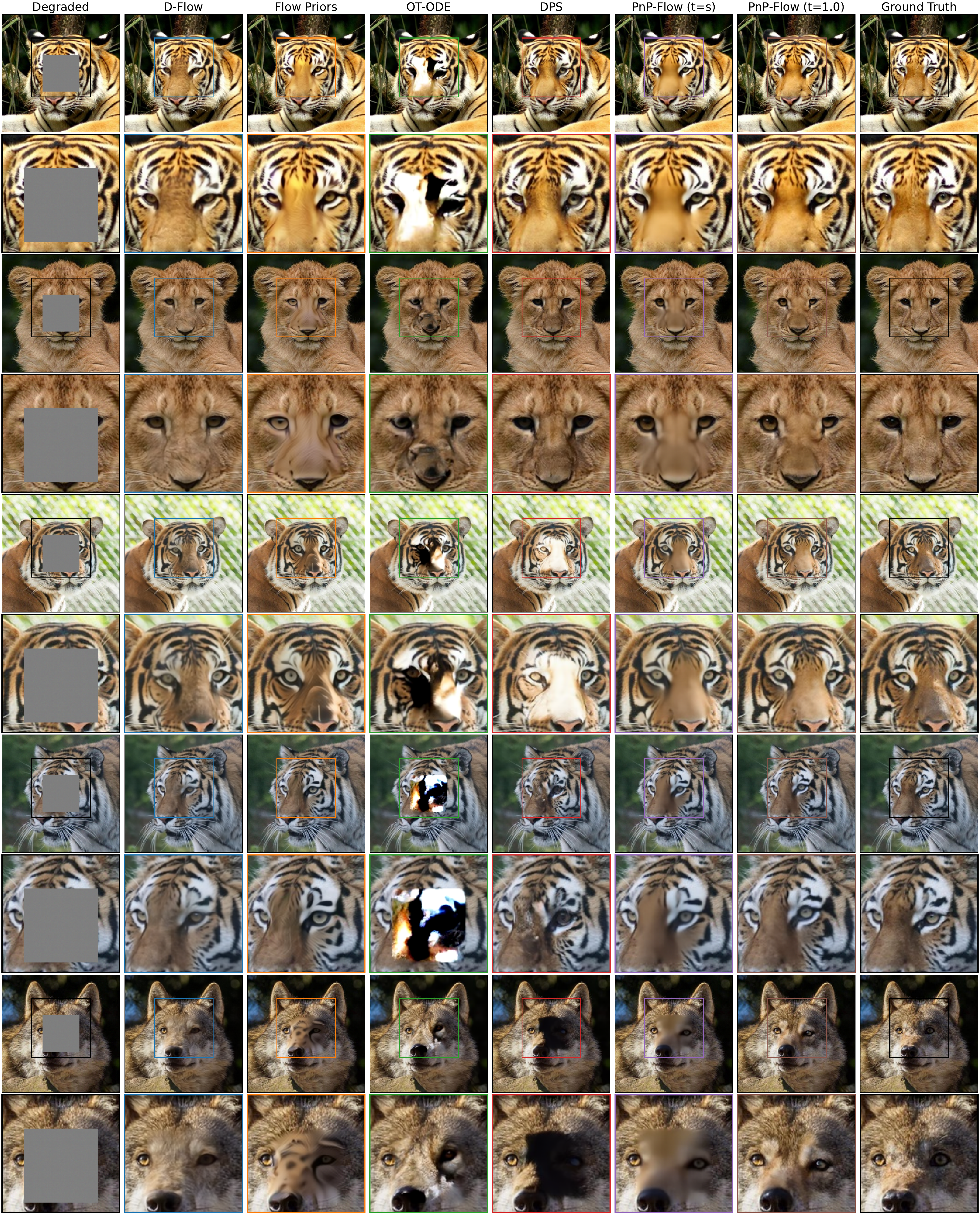}
    \caption{Box inpainting with mask $80\times 80$, and $\sigma = 0.01$}
    \label{fig:box_inp}
\end{figure}

\begin{figure}[H]
    \centering
    \includegraphics[width=0.75\linewidth]{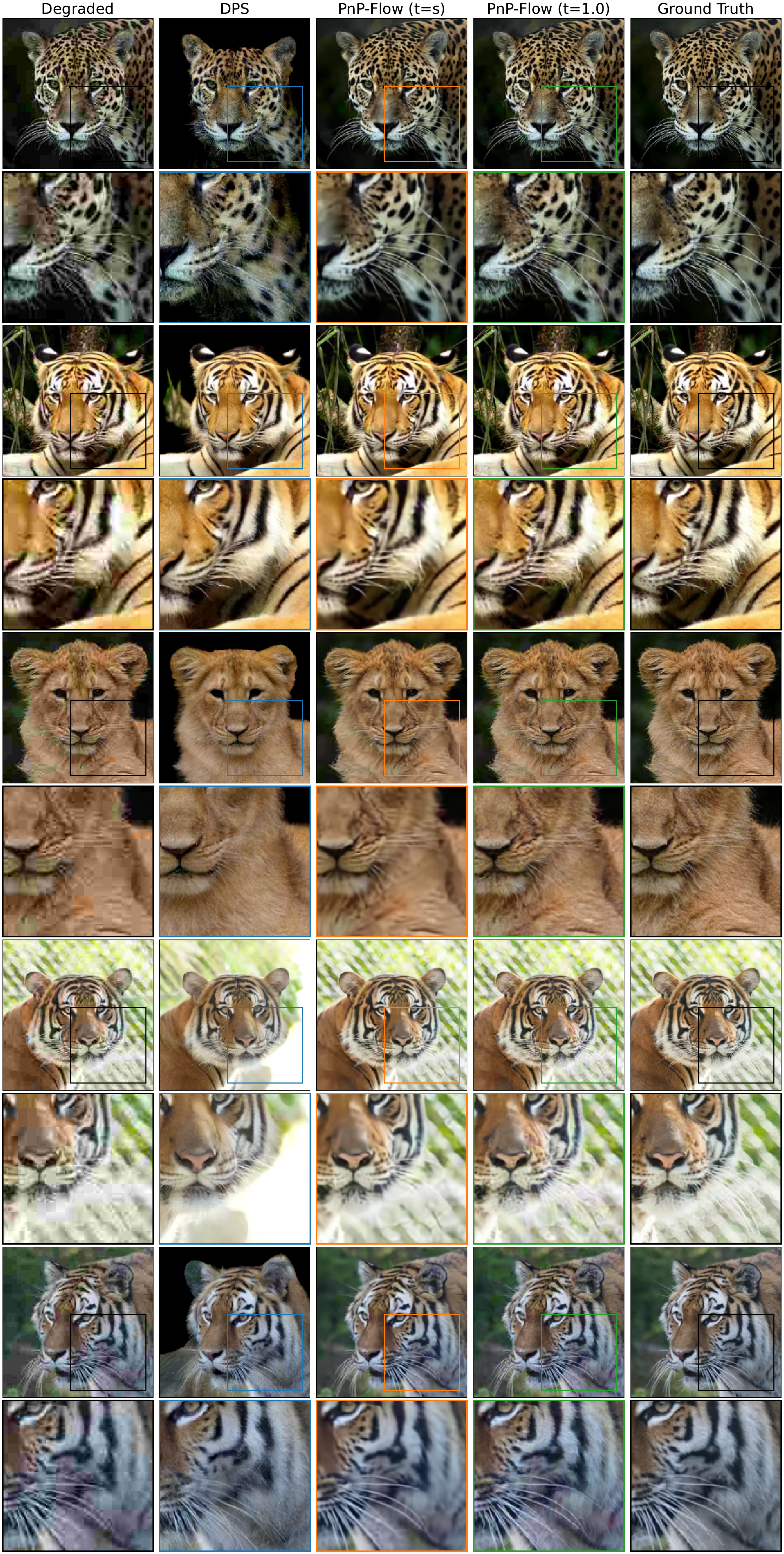}
    \caption{JPEG restoration.}
    \label{fig:jpeg}
\end{figure}

\begin{figure}[H]
    \centering
    \includegraphics[width=1.0\linewidth]{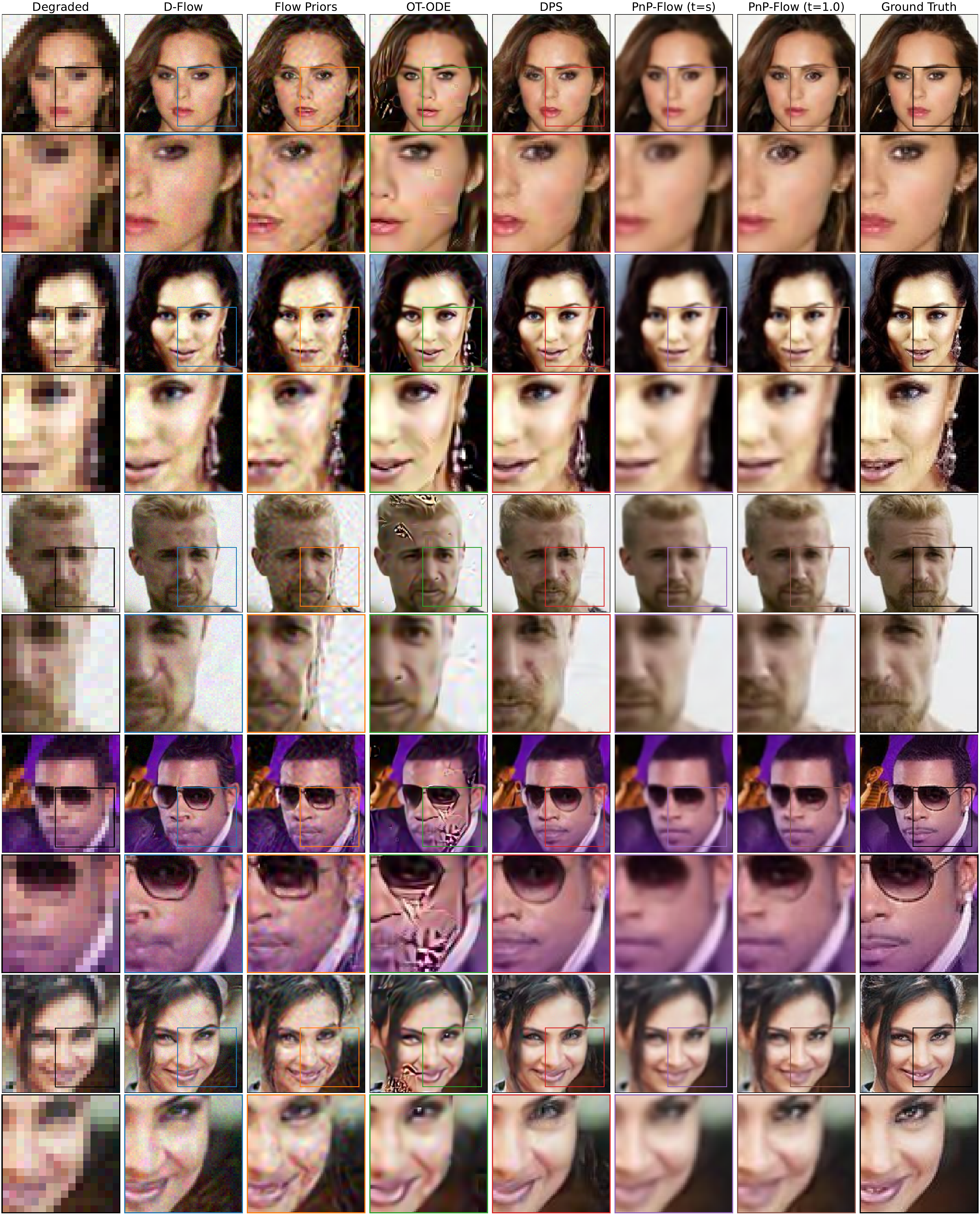}
    \caption{CelebA - Super resolution $\times 4$, and $\sigma =  0.05$}
    \label{fig:sr_celeba}
\end{figure}

\begin{figure}[H]
    \centering
    \includegraphics[width=1.0\linewidth]{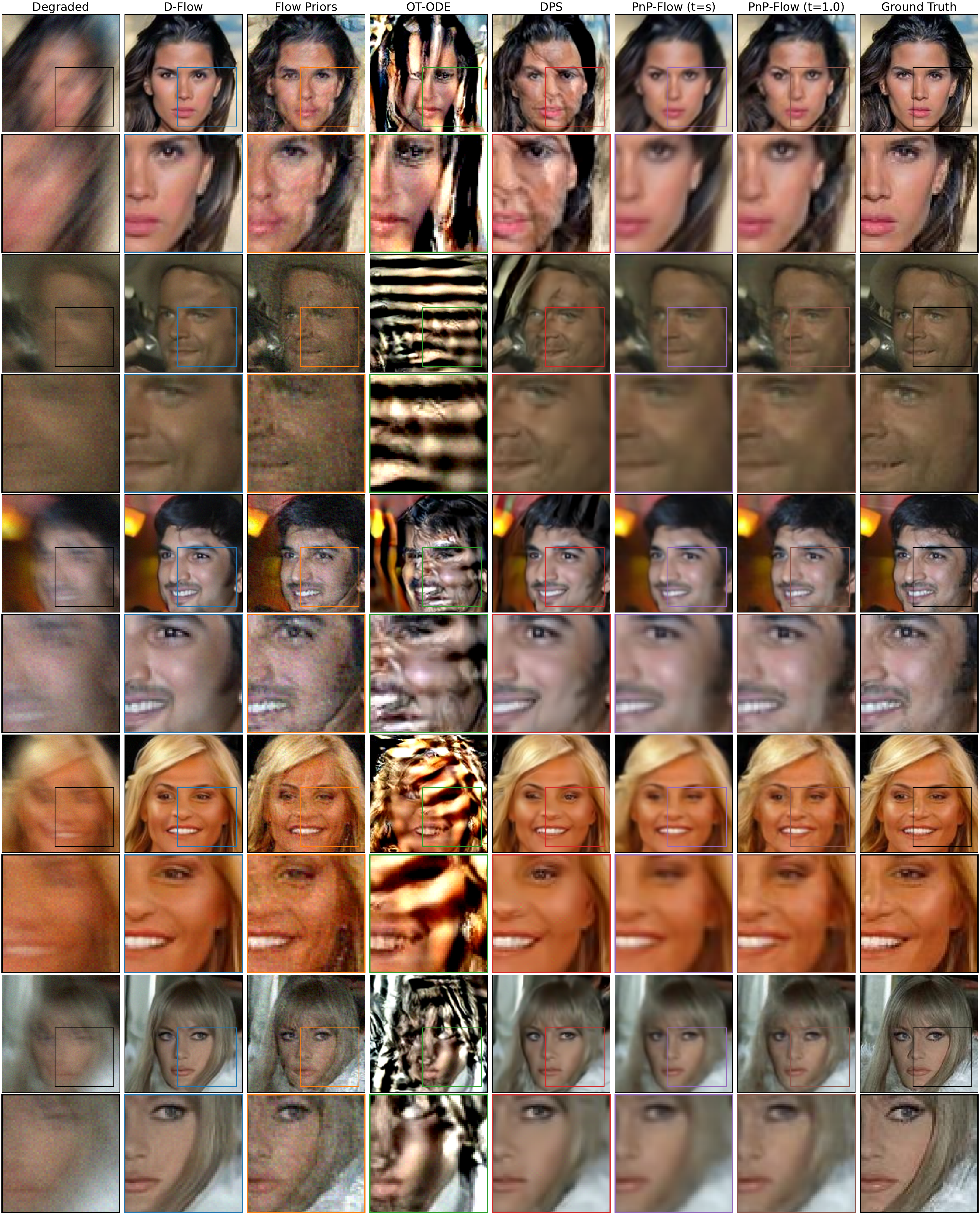}
    \caption{CelebA - Motion deblurring}
    \label{fig:motion_celeba}
\end{figure}

\begin{figure}[H]
    \centering
    \includegraphics[width=0.6\linewidth]{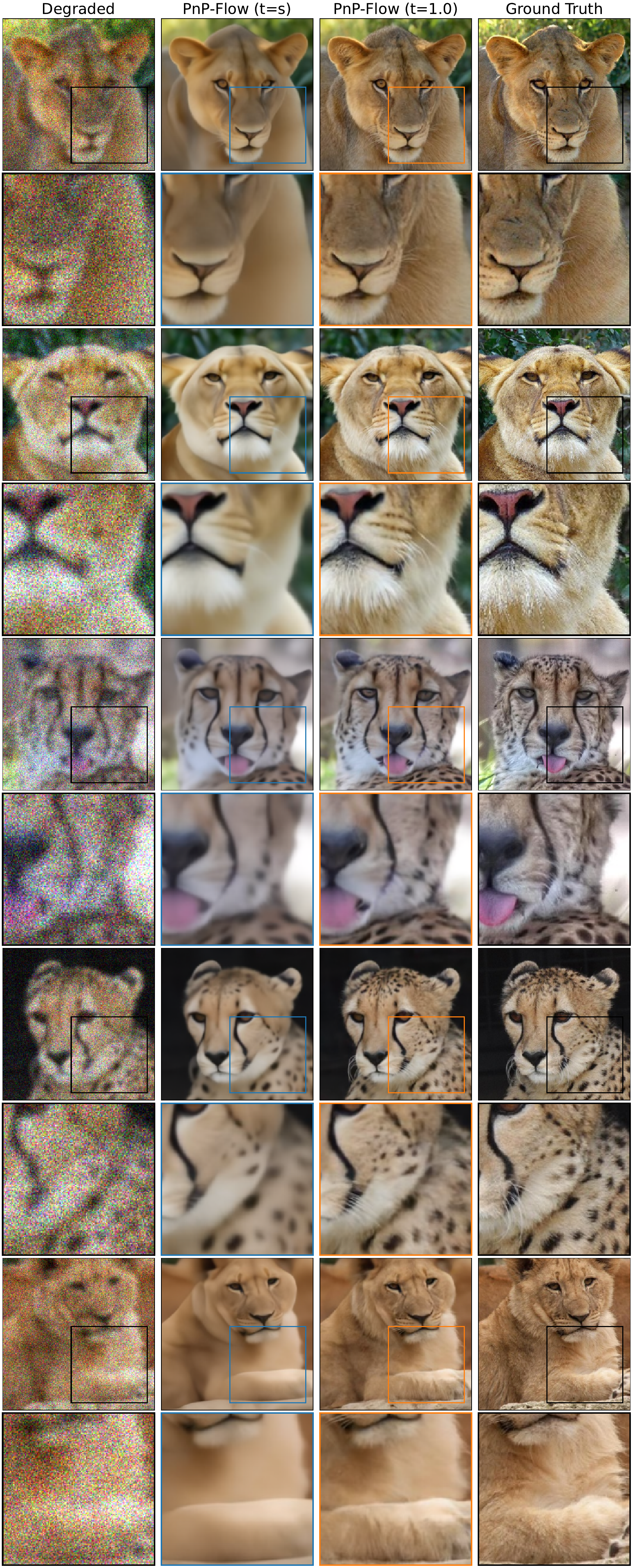}
    \caption{Gaussian deblurring with Poisson noise}
    \label{fig:poisson_visual}
\end{figure}

\begin{figure}[H]
    \centering
    \includegraphics[width=1.0\linewidth]{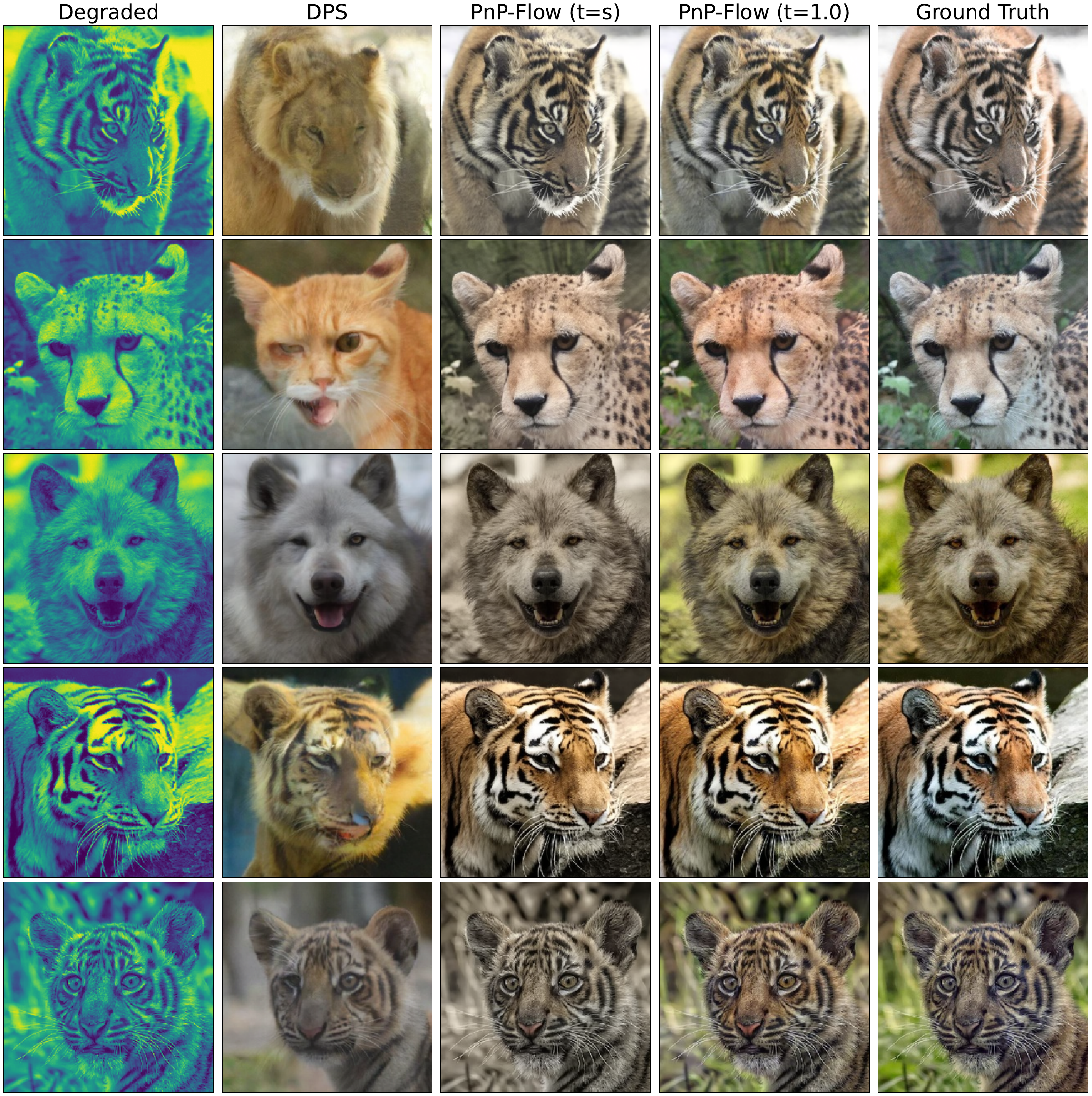}
    \caption{Colorization}
    \label{fig:colorization}
\end{figure}

\begin{figure}[H] % Use figure* if you are in a two-column format so it spans the whole top page
    \centering
    
    % --- Top Row: Three Plots ---
    \begin{subfigure}[b]{1\textwidth}
        \centering
        \includegraphics[width=\textwidth]{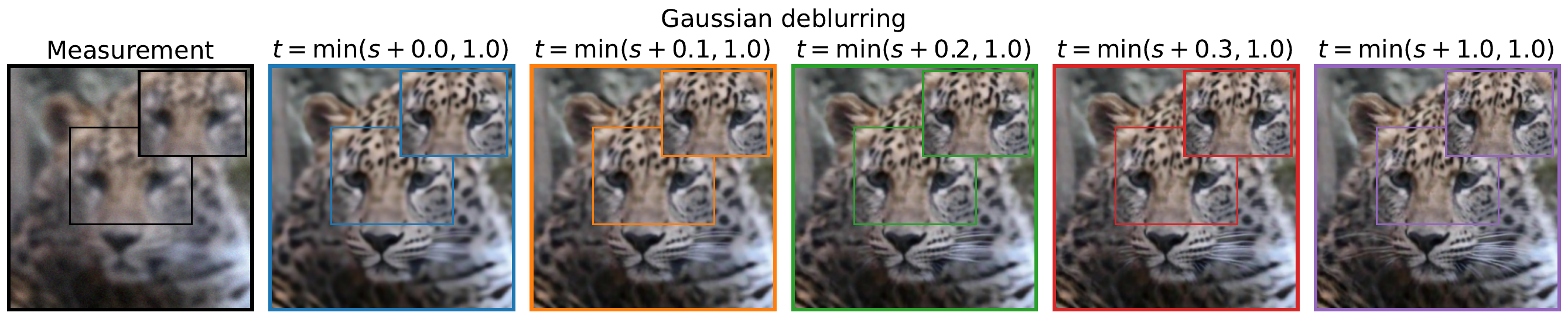}
    \end{subfigure}
    \begin{subfigure}[b]{1\textwidth}
        \centering
        \includegraphics[width=\textwidth]{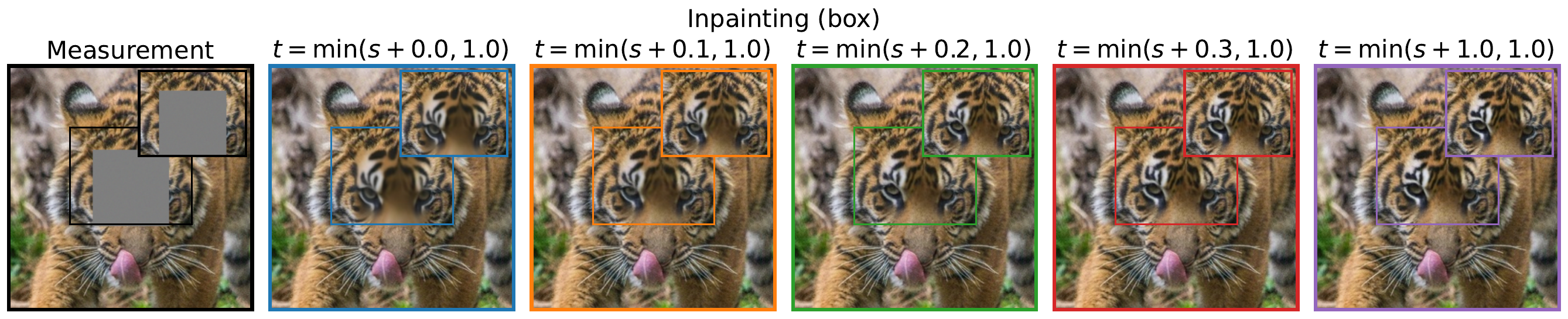}
    \end{subfigure}
    \begin{subfigure}[b]{1\textwidth}
        \centering
        \includegraphics[width=\textwidth]{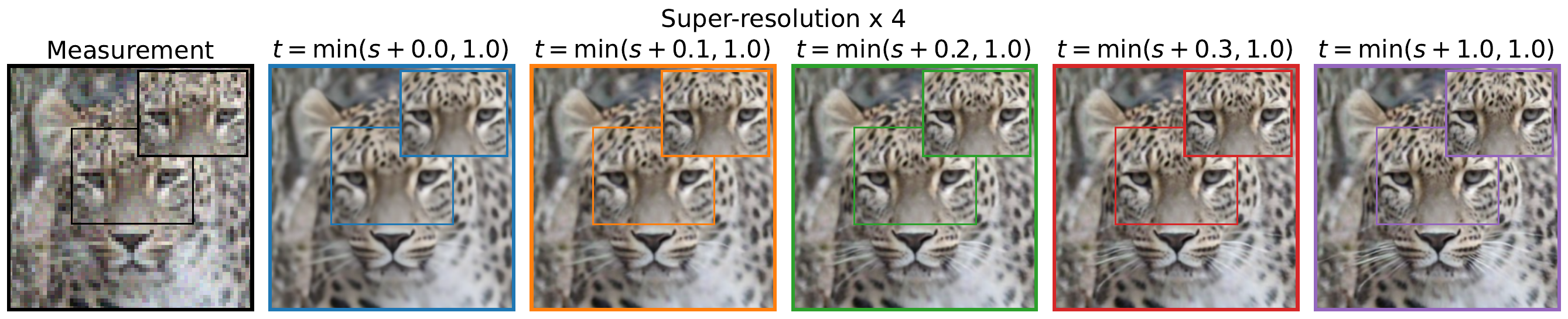}
    \end{subfigure}
    \begin{subfigure}[b]{1\textwidth}
        \centering
        \includegraphics[width=\textwidth]{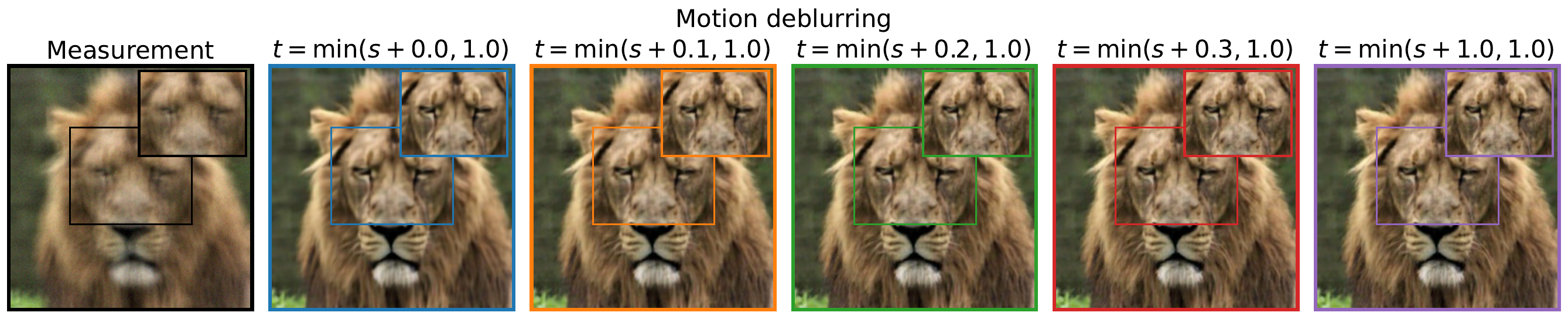}
    \end{subfigure}
    \begin{subfigure}[b]{1\textwidth}
        \centering
        \includegraphics[width=\textwidth]{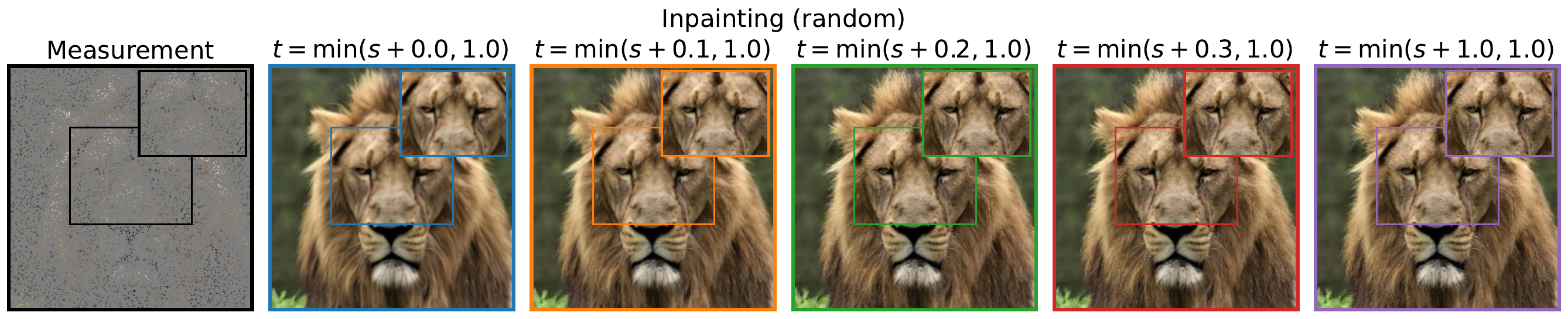}
    \end{subfigure}
    % --- Main Caption ---
    \caption{Larger version of the images in Fig.~\ref{fig:DP_exp}.}
    \label{fig:trajectory_afhq}
\end{figure}

\subsection{Visualization of the iterations for different lookaheads}

\begin{figure}[H]
    \centering
    \includegraphics[width=0.75\linewidth]{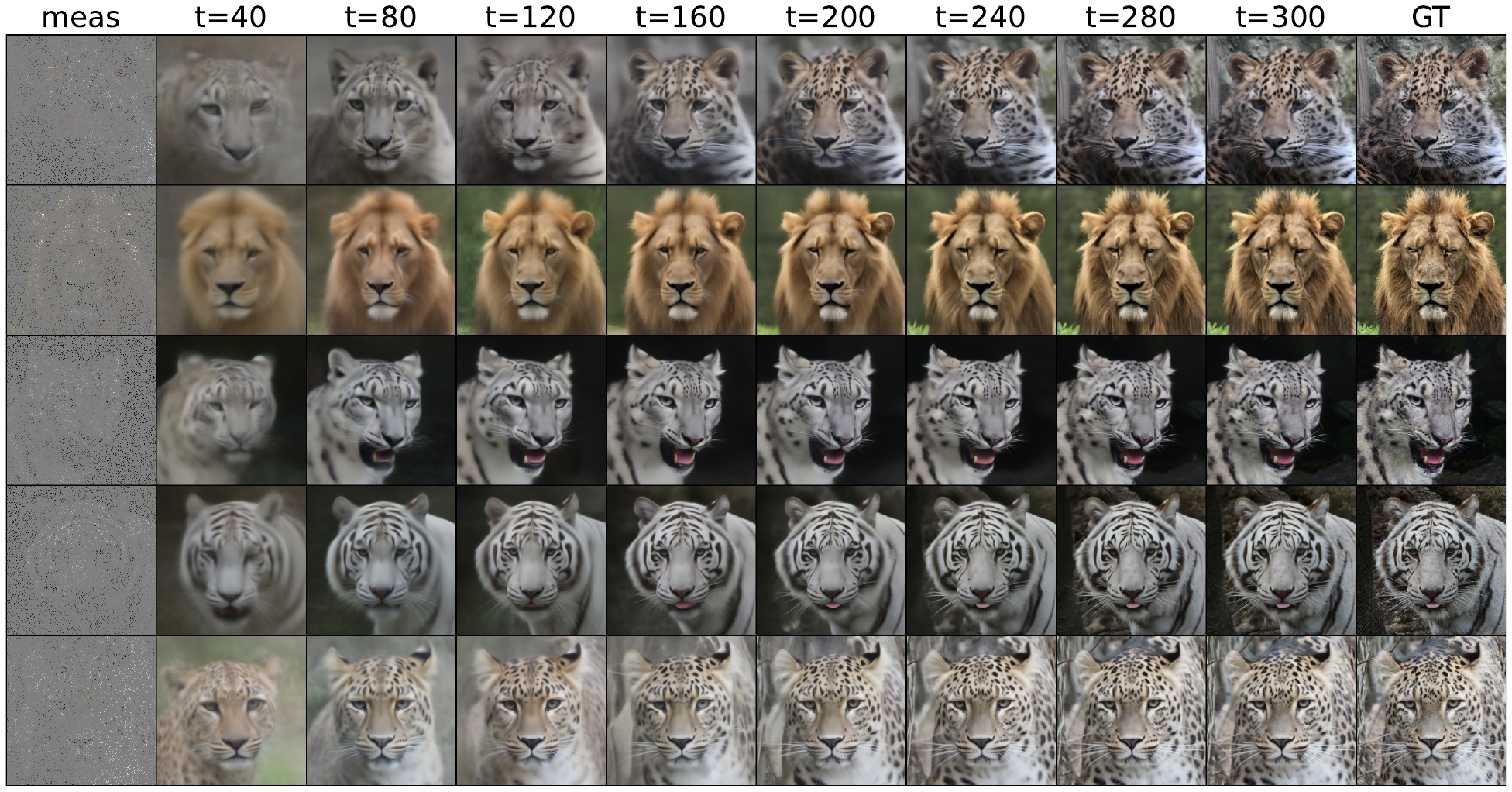}
    \caption{Progression of random inpainting with lookahead $t=s$}
    \label{fig:inp_prog_look0}
\end{figure}

\begin{figure}[H]
    \centering
    \includegraphics[width=0.75\linewidth]{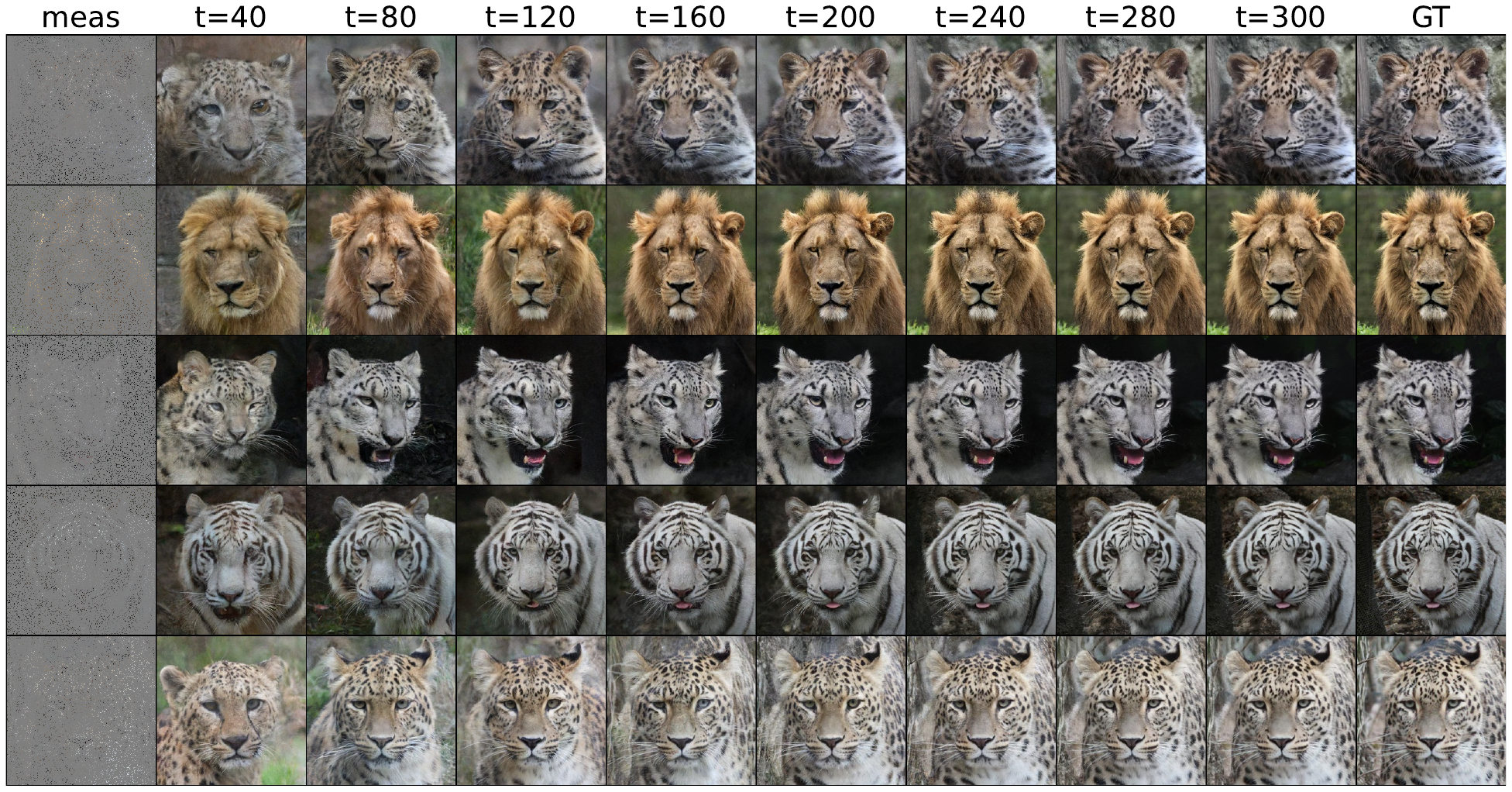}
    \caption{Progression of random inpainting with lookahead $t=1$}
    \label{fig:inp_prog_look1}
\end{figure}

\begin{figure}[H]
    \centering
    \includegraphics[width=0.85\linewidth]{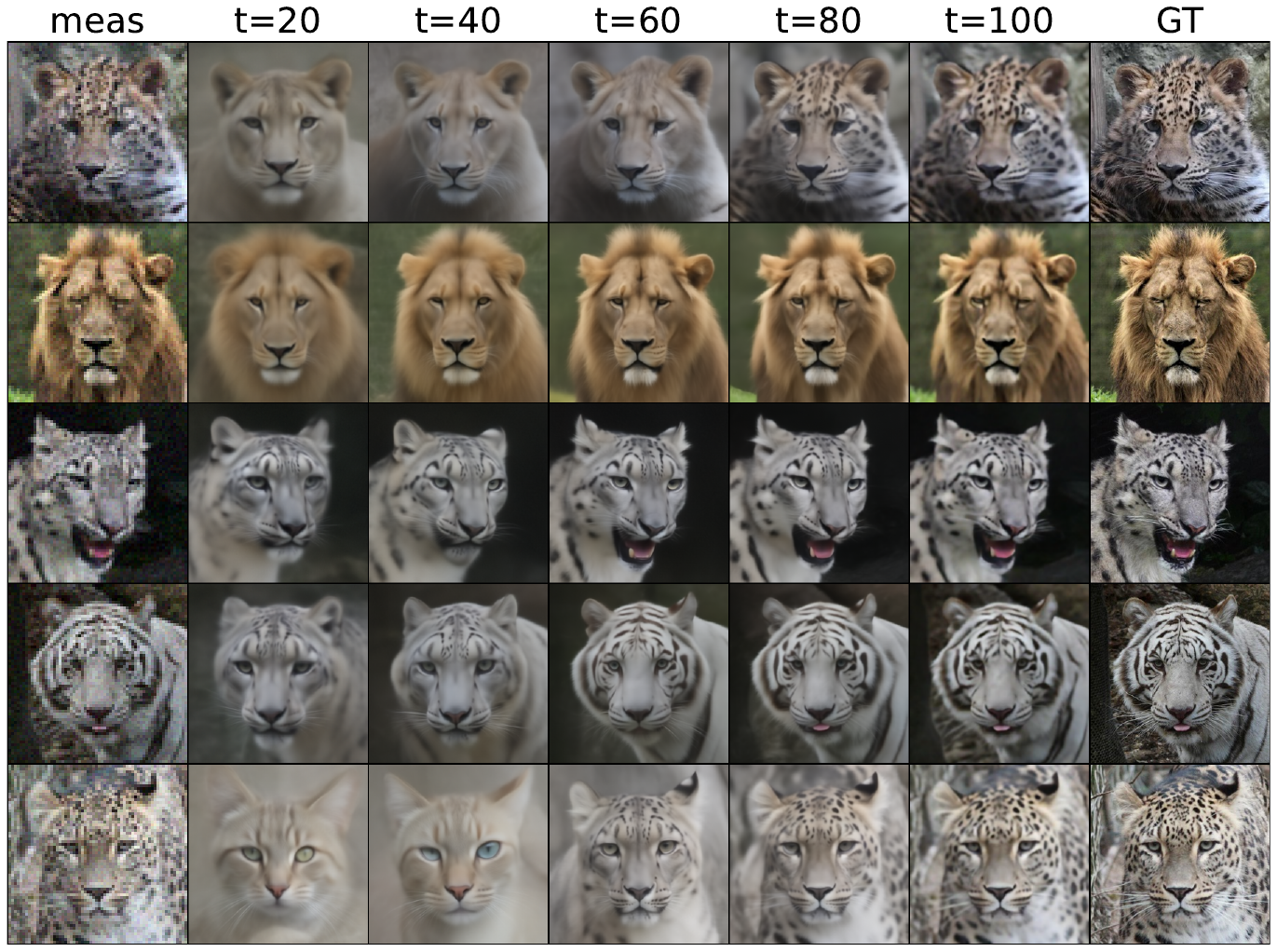}
    \caption{Progression of SR $\times 4$ with lookahead $t=s$}
    \label{fig:sr_prog_look0}
\end{figure}

\begin{figure}[H]
    \centering
    \includegraphics[width=0.85\linewidth]{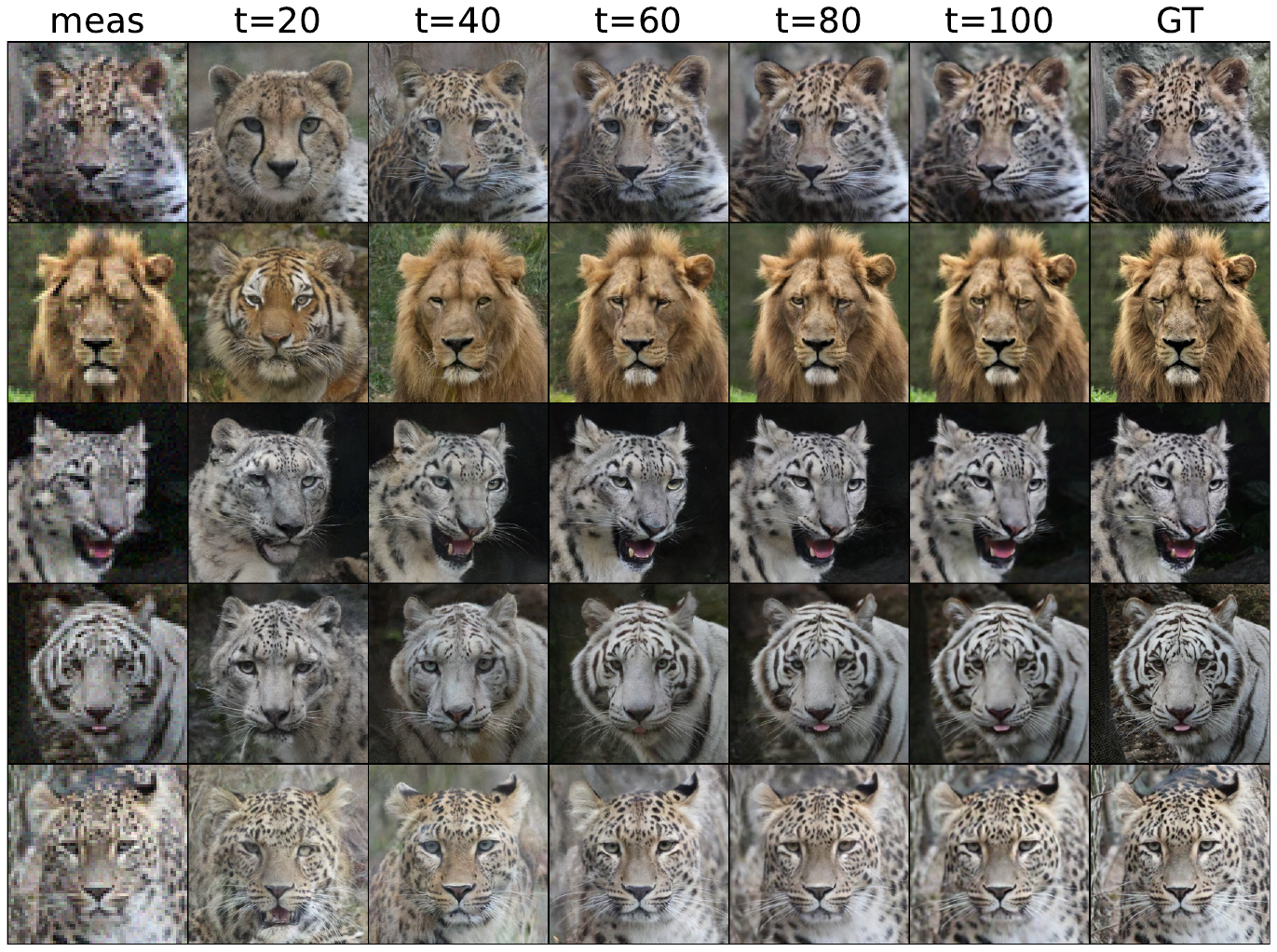}
    \caption{Progression of SR $\times 4$ with lookahead $t=1$}
    \label{fig:sr_prog_look1}
\end{figure}

% \subsection{Ablation experiments}

\subsection{Performance as a function of the number of PnP iterations}
\label{app:num_steps_ablation}

We study how the two endpoint lookaheads, $t = s$ and $t = 1$, depend on the number of PnP iterations $N$. Results for AFHQ are reported on super-resolution $\times 4$ (Fig.~\ref{fig:ablation_num_steps_sr}), Gaussian deblurring (Fig.~\ref{fig:ablation_num_steps_gd}) and box inpainting (Fig.~\ref{fig:ablation_num_steps_boxinp});
we also include for CelebA in Fig.~\ref{fig:ablation_num_steps_md_celeba}.

The two regimes behave qualitatively differently. 
At $t = s$, both RMSE and FID decrease monotonically with $N$ and plateau. 
At $t = 1$, the perceptual trajectory is non-monotonic: FID drops sharply within the first few tens of iterations, reaches its minimum, and then drifts back up as $N$ continues to grow. 
The effect is most pronounced for Gaussian deblurring, where the best FID at $t = 1$ ($\approx 11$ at $N \approx 20$) is reached an order of magnitude earlier than the best FID at $t = s$ ($\approx 23.5$ at $N \approx 300$).
This asymmetry is consistent with the measured growth of the local Lipschitz constant with $t$ (Appendix~\ref{app:lipschitz}), which reflects the transition from an averaging MMSE estimator at $t = s$ to a less regular, perceptually oriented estimator at $t = 1$, so that approximation errors of the flow map compound more strongly over iterations at large lookaheads.
RMSE remains well behaved in both cases, indicating that the drift is along the perceptual axis.

The practical takeaway is that the two endpoints call for different iteration schemes: the perceptual endpoint ($t = 1$) is best run for fewer steps than the distortion endpoint ($t = s$). 
The values of $N$ used in our main experiments (Appendix~\ref{app:hyperparamters}) reflect this asymmetry.
This behavior is also consistent with the use of flow maps and consistency models as denoisers for few-step restoration~\citep{spagnoletti2025latino, gulle2025consistency}.

\begin{figure}[H]
    \centering
    \includegraphics[width=0.7\linewidth]{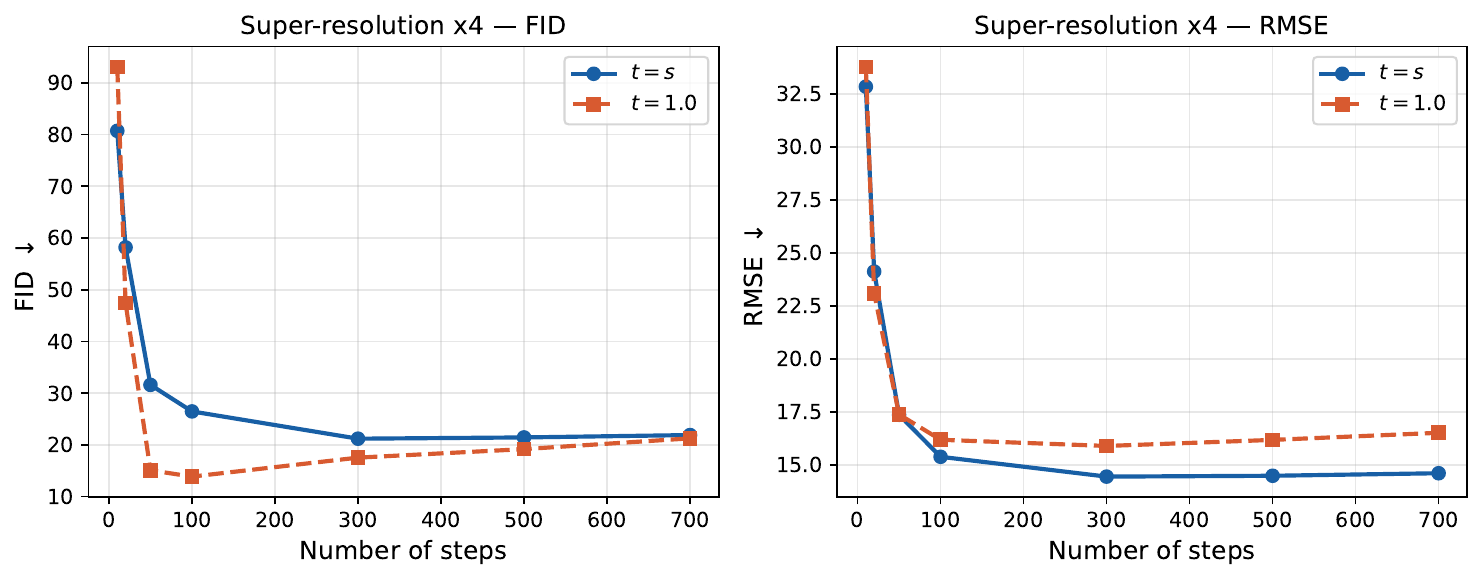}
    \caption{{AFHQ - Super-resolution $\times 4$.} Ablation of PnP-flow as a function of the number of steps.}
    \label{fig:ablation_num_steps_sr}
\end{figure}

\begin{figure}[H]
    \centering
    \includegraphics[width=0.7\linewidth]{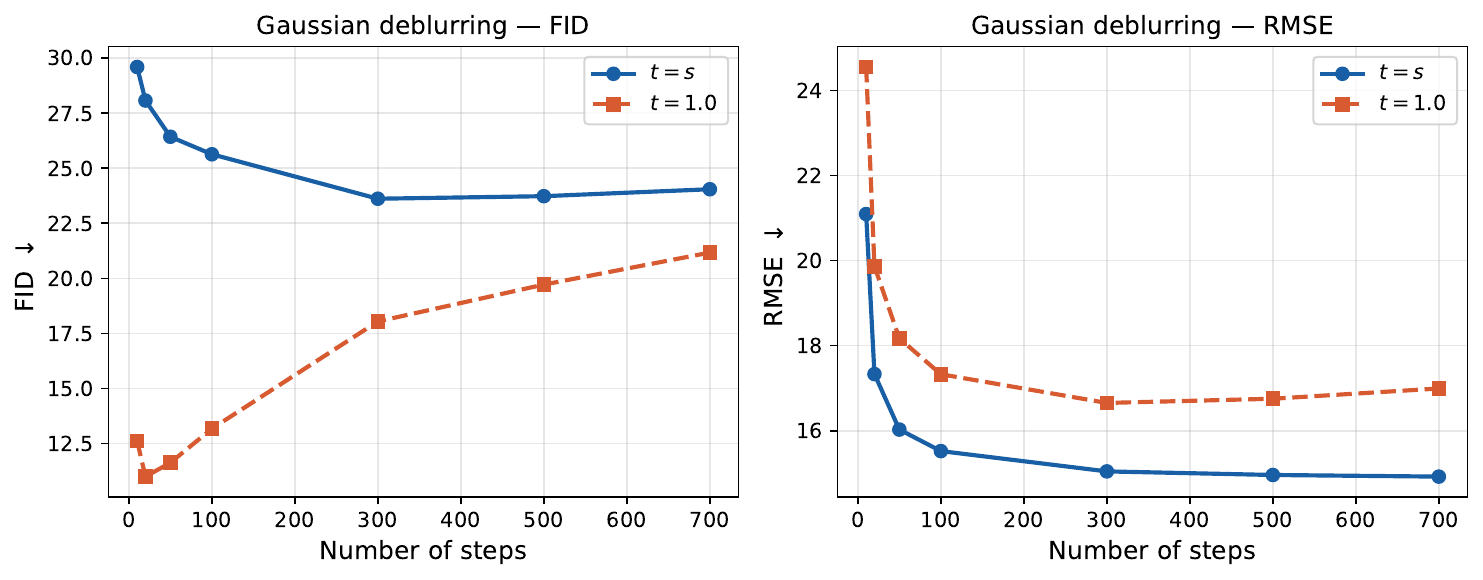}
    \caption{{AFHQ - Gaussian deblurring.} Ablation of PnP-flow as a function of the number of steps.}
    \label{fig:ablation_num_steps_gd}
\end{figure}

\begin{figure}[H]
    \centering
    \includegraphics[width=0.7\linewidth]{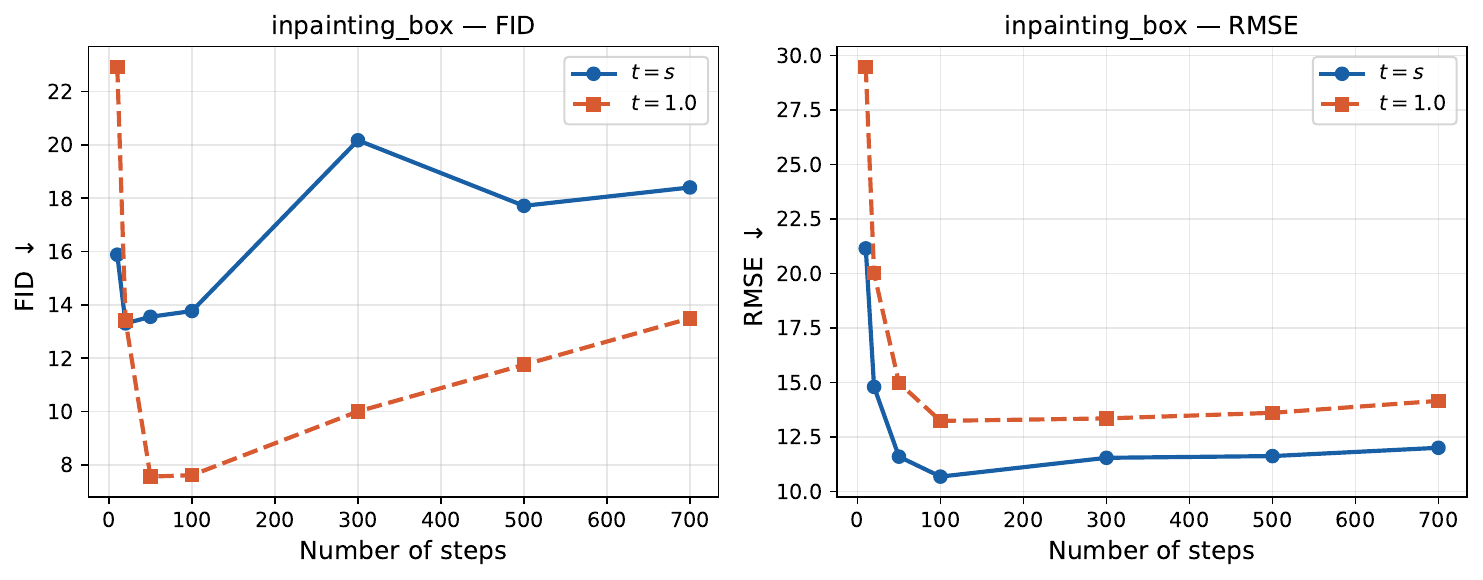}
    \caption{{AFHQ - Box inpainting.} Ablation of PnP-flow as a function of the number of steps.}
    \label{fig:ablation_num_steps_boxinp}
\end{figure}

\begin{figure}[H]
    \centering
    \includegraphics[width=0.7\linewidth]{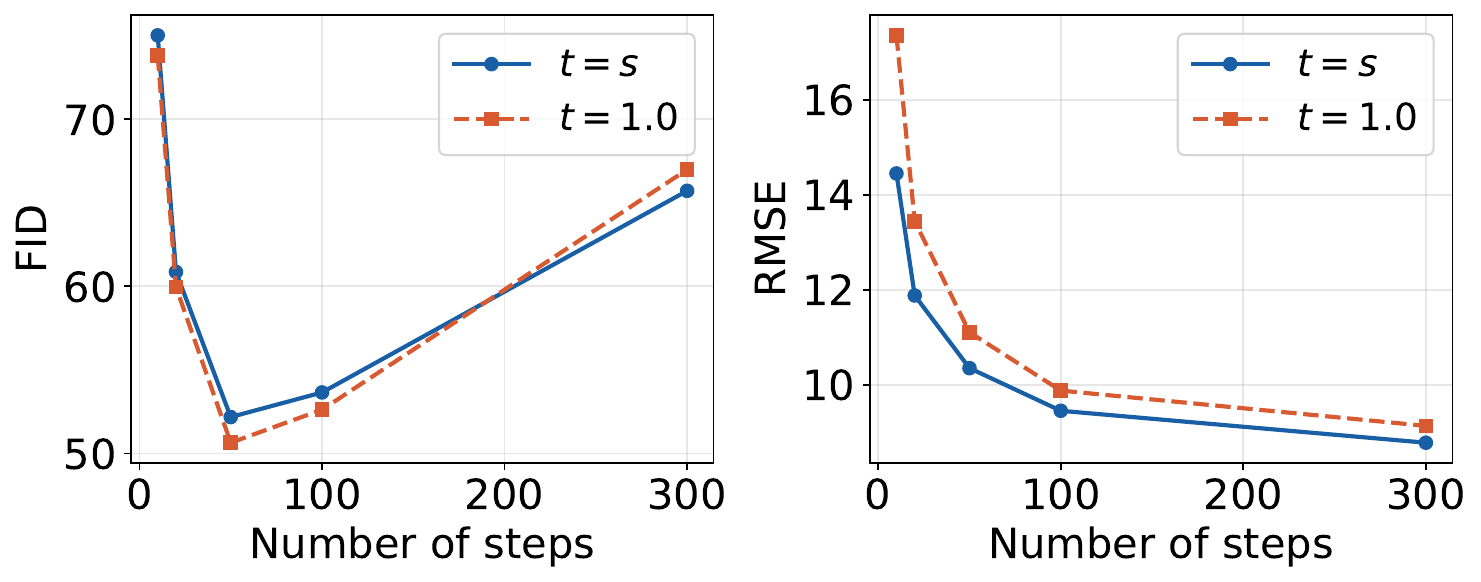}
    \caption{CelebA - Motion deblurring. Ablation of PnP-flow as a function of the number of steps.}
    \label{fig:ablation_num_steps_md_celeba}
\end{figure}

\subsection{Failure modes and limitations}
\label{app:failures_limitations}

For some degradations, the DP traversal is not monotone. We observe two such cases, box inpainting (Fig.~\ref{fig:DP_exp}) and phase retrieval, which point to two distinct mechanisms.

\paragraph{Phase retrieval.}
For phase retrieval with oversampling factor $4$, PnP with the perceptual endpoint ($t=1$) outperforms the MMSE endpoint ($t=s$) only after a certain number of iterations (Fig.~\ref{fig:ablation_phase}), and the overall performance of both is limited (Table~\ref{table:non_linear_phase} and Fig.~\ref{fig:phase}).

\begin{figure}[H]
    \centering
    \includegraphics[width=0.5\linewidth]{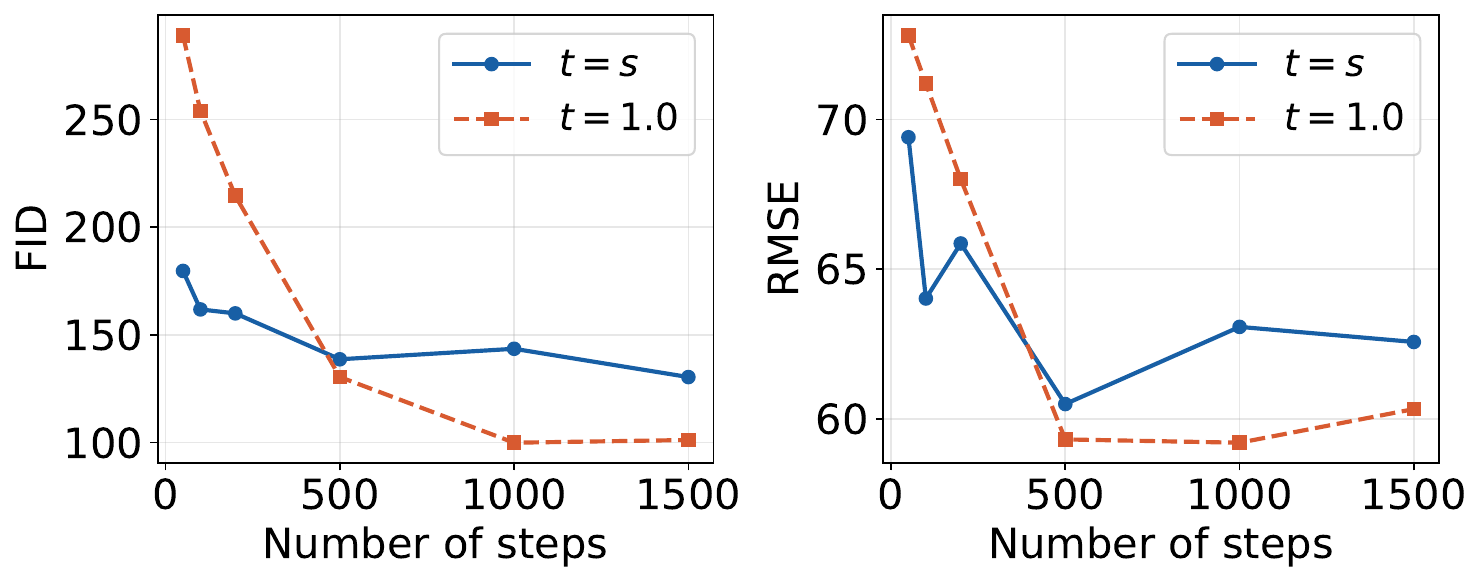}
    \caption{Performance as a function of the number of steps}
    \label{fig:ablation_phase}
\end{figure}

\begin{table}[H]
\centering
\caption{\small{Quantitative results for Phase retrieval w/$\sigma = 0.01$ on AFHQ datasets. Best results are in \textbf{bold}.}}
\label{table:non_linear_phase}
\setlength{\tabcolsep}{4pt}
\scalebox{0.7}{
\begin{tabular}{@{} r cccc @{}}
\toprule
\multirow{2}{*}{\textbf{Sampler}} & \multicolumn{4}{c}{\textbf{AFHQ $256\times256$}} \\
\cmidrule(r){2-5}
& \textbf{PSNR$\uparrow$} & \textbf{LPIPS$\downarrow$} & \textbf{FID$\downarrow$} & \textbf{DISTS$\downarrow$} \\
\midrule
PnP-Flow ($t = s$)  & \textbf{17.36} & 0.454 & 143.55 & 0.316 \\
PnP-Flow ($t = 1$)  & 20.88 & \textbf{0.317} & \textbf{99.95} & \textbf{0.240} \\
\bottomrule
\end{tabular}
}
\vspace{-0.2cm}
\end{table}

\begin{figure}
    \centering
    \includegraphics[width=0.5\linewidth]{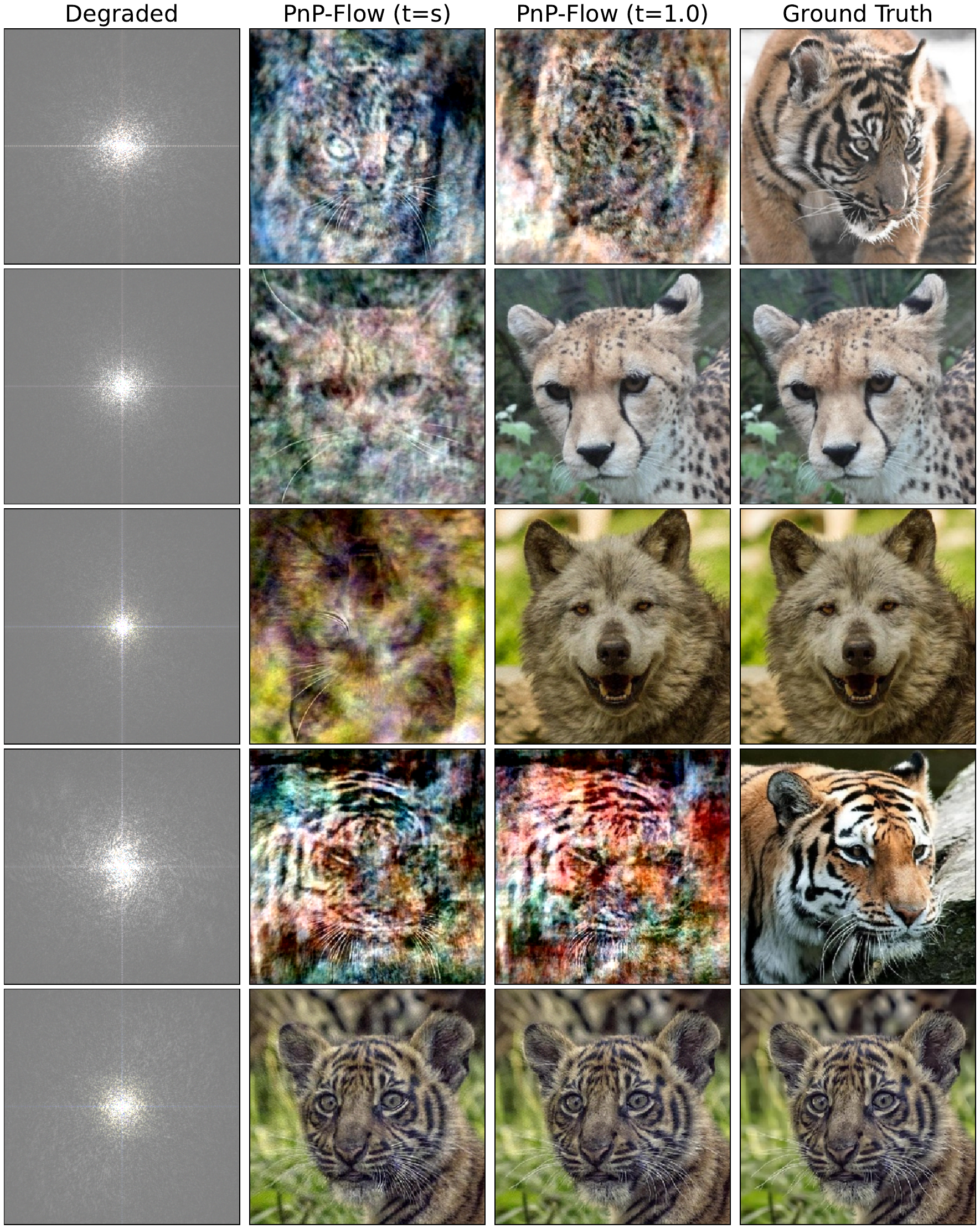}
    \caption{Reconstruction for phase retrieval}
    \label{fig:phase}
\end{figure}

\paragraph{Mechanism 1: spatially coherent null spaces.}
The size of the null space of the forward operator alone does not explain the non-monotonicity: random inpainting with $90\%$ missing pixels has a much larger null space than an $80\times80$ box, yet its DP curve is nearly monotone. We hypothesize that the \emph{spatial coherence} of the null space is the dominant factor. For a random mask, the missing content is local and the prior interpolates it from dense surrounding context. A box mask instead forces the prior to synthesize coherent content over a contiguous region. As $t$ increases, the denoiser shifts from posterior-mean-like toward sample-like outputs precisely in this region; the synthesized content is plausible under the prior, but the data-fidelity step cannot correct it, since the residual is blind to the null space. Flow-map approximation error, which grows with the lookahead (Appendix~\ref{app:lipschitz}), then determines how far this synthesis drifts from the observed context.

To test this hypothesis, we fix the fraction of missing pixels and vary the block size of the mask from $1\times1$ (random pixels) to $8\times8$ and $32\times32$ blocks, tracing the DP curve over the lookahead on AFHQ (Table~\ref{tab:mask_coherence}). RMSE is monotone in $t$ in all conditions, so any non-monotonicity is purely perceptual. For the random mask, the FID curve is monotone and then flat (reversal $+0.02$). As coherence increases, the curve becomes non-monotone, with a reversal of $+1.29$ for $8\times8$ blocks and $+2.34$ for $32\times32$ blocks; KID shows the same ordering (reversals of approximately $0$, $0.9\cdot10^{-3}$, and $1.5\cdot10^{-3}$). Coherence also degrades the MMSE endpoint itself (FID $16.5$ versus $\approx 77$ at $t=s$), consistent with synthesis being intrinsically harder than interpolation. At fixed missingness, the monotonicity of the DP curve thus degrades continuously with mask coherence.

\begin{table}[H]
\centering
\caption{\small{FID ($\downarrow$) along the lookahead sweep for inpainting on AFHQ with a fixed fraction of missing pixels and increasing block size of the mask. The last column reports the reversal, i.e., the FID at $t=1$ minus the minimum FID along the sweep.}}
\label{tab:mask_coherence}
\small
\begin{tabular}{lccccccc}
\toprule
Mask & $t=s$ & $t=0.1$ & $t=0.3$ & $t=0.5$ & $t=0.7$ & $t=1$ & Reversal \\
\midrule
$1\times1$ (random) & 16.48 & 13.02 & 12.99 & 12.99 & 13.01 & 13.01 & $+0.02$ \\
$8\times8$ & 76.71 & 42.59 & 29.46 & 29.32 & 29.99 & 30.61 & $+1.29$ \\
$32\times32$ & 76.81 & 50.76 & 45.98 & 45.35 & 43.96 & 46.30 & $+2.34$ \\
\bottomrule
\end{tabular}
\end{table}

\paragraph{Mechanism 2: non-convex data fidelity.}
In phase retrieval, the forward operator is nonlinear, so $g(\bbx) = \tfrac{1}{2}\|\bby - f(\bbx)\|^2$ is non-convex with spurious stationary points. Since the local Lipschitz constant of the iteration grows with $t$ (Appendix~\ref{app:lipschitz}), large lookaheads are precisely where the iteration is most likely to move between basins; the failure stems from the interaction of the solver dynamics with the non-convex landscape.

\paragraph{Summary.}
Monotone traversal is expected when $(i)$ the forward operator is linear, so that the data fidelity is convex, and $(ii)$ the null space is spatially incoherent with image structure, so that the prior interpolates rather than synthesizes. Degradations that violate $(ii)$, such as large contiguous masks, are prone to the first mechanism, and those that violate $(i)$, such as phase retrieval and other nonlinear operators, to the second. This is consistent with our findings across the evaluated tasks.

\section{Proofs and theoretical results}
\label{app_general:proofs}

\subsection{Average denoiser.}
\label{app:avg_denoiser}

Recall the definition of the average denoiser
\begin{equation}
D_{s,t}(\bbx) = \bbx + (1 - s)\,\bbv(\bbx, s, t),
\end{equation}
where $\bbv(\bbx,s,t)$ denotes the average velocity of the flow map over $[s,t]$, and the flow map is parameterized as
\begin{equation}
X_{s,t}(\bbx) = \bbx + (t - s)\,\bbv(\bbx,s,t).
\end{equation}

\begin{proposition}[Basic properties]
For all $0 \le s \le u \le t \le 1$, the average denoiser satisfies:

\begin{enumerate}
\item \textbf{Flow-map relation.}
\begin{equation}
X_{s,t}(\bbx)
= \frac{1 - t}{1 - s}\,\bbx + \frac{t - s}{1 - s}\,D_{s,t}(\bbx).
\end{equation}

\item \textbf{Diagonal consistency (MMSE).}
\begin{equation}
D_{s,s}(\bbx) = D_s(\bbx) = \mathbb{E}[\bbx_1 \mid \bbx_s = \bbx].
\end{equation}

\item \textbf{Semigroup structure.}
\begin{equation}
D_{s,t}(\bbx)
= \gamma\, D_{s,u}(\bbx)
+ (1 - \gamma)\, D_{u,t}(X_{s,u}(\bbx)),
\end{equation}
with
\begin{equation}
\gamma = \frac{(1 - t)(u - s)}{(1 - u)(t - s)} \in [0,1].
\end{equation}

\end{enumerate}
\end{proposition}

\paragraph{Proof.}

\textbf{(1) Flow-map relation.}
Using the parameterization,
\[
D_{s,t}(\bbx) - \bbx = (1 - s)\,\bbv(\bbx,s,t),
\quad
X_{s,t}(\bbx) - \bbx = (t - s)\,\bbv(\bbx,s,t).
\]
Eliminating $\bbv$ gives
\[
X_{s,t}(\bbx)
= \bbx + \frac{t - s}{1 - s}\big(D_{s,t}(\bbx) - \bbx\big)
= \frac{1 - t}{1 - s}\,\bbx + \frac{t - s}{1 - s}\,D_{s,t}(\bbx).
\]

\textbf{(2) Diagonal consistency.}
By definition of the flow map,
\[
\lim_{t \to s} \bbv(\bbx,s,t) = \bbv(\bbx,s),
\]
the instantaneous velocity. Using Tweedie's relation,
\[
D_s(\bbx) = \bbx + (1 - s)\,\bbv(\bbx,s),
\]
which matches $D_{s,t}$ at $t=s$.

\textbf{(3) Semigroup structure.}
From the flow map semigroup property,
\[
X_{s,t}(\bbx) = X_{u,t}(\bbX_{s,u}(\bbx)).
\]
Substitute the affine forms from (i) for each map and match coefficients of $\bbx$ and $D_{s,t}$.
After rearrangement, this yields the convex combination:
\[
D_{s,t}(\bbx)
= \gamma\, D_{s,u}(\bbx)
+ (1 - \gamma)\, D_{u,t}(X_{s,u}(\bbx)),
\]
with $\gamma$ as stated.

\subsection{Full derivation for the Gaussian case (Theorem~\ref{thm:optimal})}
\label{app:gaussian}

We provide the complete derivation of the average denoiser gain and
its connection to the optimal DP estimator for the scalar Gaussian
target.

\subsubsection{Setup}

Let $p_1 = \ccalN(0, \sigma_p^2)$ with $\sigma_p^2 \leq 1$, and let
$x_\tau = \tau x_1 + (1-\tau)\epsilon$ with $x_1 \sim p_1$,
$\epsilon \sim \ccalN(0,1)$, and $x_1 \perp \epsilon$. Then
$x_\tau \sim \ccalN(0, \sigma_\tau^2)$ with
$\sigma_\tau^2 := \tau^2\sigma_p^2 + (1-\tau)^2$. The instantaneous
denoiser and velocity are linear:
\begin{equation}
    D_\tau(x) = a_\tau\,x, \quad a_\tau = \frac{\tau\sigma_p^2}{\sigma_\tau^2},
    \qquad
    v_\tau(x) = b_\tau\,x, \quad b_\tau = \frac{a_\tau - 1}{1-\tau} 
    = \frac{\tau(\sigma_p^2+1) - 1}{\sigma_\tau^2}.
\end{equation}

By linearity, the ODE $\dot x_\tau = v_\tau(x_\tau) = b_\tau x_\tau$
admits the closed-form solution $x_t = \Phi(t,s)\,x_s$ with
$\Phi(t,s) = \exp\!\bigl(\int_s^t b_\tau\,d\tau\bigr)$. The average
velocity over $[s,t]$ then satisfies $x_t = x_s + (t-s)\,v_{s,t}(x_s)$
with $v_{s,t}(x_s) = [\Phi(t,s) - 1]/(t-s) \cdot x_s$, and the
average denoiser $D_{s,t}(x) = x + (1-s)\,v_{s,t}(x)$ becomes
$D_{s,t}(x_s) = \Lambda(s,t)\,x_s$ with:
\begin{equation}
    \Lambda(s,t) = \frac{(1-s)\,\Phi(t,s) - (1-t)}{t-s}.
    \label{eq:app_Lambda_abstract}
\end{equation}

We now introduce a few useful Lemmas.

% \subsubsection{Closed form for $\Phi(t,s)$ and $\Lambda(s,t)$}

\begin{lemma}[Closed form for $\Phi(t,s)$ and $\Lambda(s,t)$]
    \label{lem:app_int_B}
    $\Phi(t,s) = \sigma_t/\sigma_s$.
\end{lemma}

\begin{proof}
A direct computation gives
$\frac{d}{d\tau}\tfrac{1}{2}\log\sigma_\tau^2 =
\dot\sigma_\tau^2/(2\sigma_\tau^2) =
[\tau(\sigma_p^2+1) - 1]/\sigma_\tau^2 = b_\tau$.
Hence $\int_s^t b_\tau\,d\tau = \log(\sigma_t/\sigma_s)$ and
$\Phi(t,s) = \sigma_t/\sigma_s$.
\end{proof}

Substituting Lemma~\ref{lem:app_int_B}
into~\eqref{eq:app_Lambda_abstract} yields the closed form stated
in~\eqref{eq:Lambda}:
\begin{equation}
    \Lambda(s,t) = \frac{(1-s)\sigma_t - (1-t)\sigma_s}{\sigma_s(t-s)}.
\end{equation}

\paragraph{Boundary values.}
At $t = 1$, $\sigma_1 = \sigma_p$ gives directly
$\Lambda(s,1) = \sigma_p/\sigma_s$. At $t = s$, applying L'Hôpital
to~\eqref{eq:Lambda} and using
$\dot\sigma_\tau = [\tau(\sigma_p^2+1)-1]/\sigma_\tau$:
\begin{equation}
    \Lambda(s,s) = \frac{(1-s)\dot\sigma_s + \sigma_s}{\sigma_s}
    = 1 + \frac{(1-s)[s(\sigma_p^2+1)-1]}{\sigma_s^2}
    = \frac{s\sigma_p^2}{\sigma_s^2} = a_s,
\end{equation}
where the last equality uses
$\sigma_s^2 + (1-s)[s(\sigma_p^2+1)-1] = s\sigma_p^2$, which one
verifies by expanding $\sigma_s^2 = s^2\sigma_p^2 + (1-s)^2$.

% \subsubsection{Monotonicity of $\Lambda(s,t)$}

\begin{lemma}[Monotonicity of $\Lambda(s,t)$]
    \label{lem:app_monotone}
    For fixed $s \in (0,1)$, $t \mapsto \Lambda(s,t)$ is strictly
    increasing on $[s,1]$.
\end{lemma}

\begin{proof}
Differentiating~\eqref{eq:Lambda} in $t$ and writing
$\sigma_t - \sigma_s = \int_s^t \dot\sigma_u\,du$:
\begin{equation}
    \partial_t \Lambda(s,t) = \frac{1-s}{\sigma_s(t-s)^2}
    \left[(t-s)\dot\sigma_t - \int_s^t \dot\sigma_u\,du\right].
\end{equation}
Writing $\sigma_\tau^2 = c\tau^2 - 2\tau + 1$ with $c = \sigma_p^2 + 1$, a direct calculation gives
$\ddot\sigma_\tau = \big[c\,\sigma_\tau^2 - (c\tau - 1)^2\big]/\sigma_\tau^3 = \sigma_p^2/\sigma_\tau^3 > 0$ for all
$\tau \in (0,1)$, so $\dot\sigma_\tau$ is strictly increasing.
Therefore $\int_s^t \dot\sigma_u\,du < (t-s)\dot\sigma_t$ for $s < t$,
giving $\partial_t\Lambda > 0$.
\end{proof}

\paragraph{On the assumption $\sigma_p^2 \leq 1$.}
The condition $\sigma_p^2 \le 1$ is a normalization convention, stating that the target variance does not exceed that of the standard normal source, as for data scaled to $[-1,1]$. The arguments above do not use it: the convexity $\ddot\sigma_\tau = \sigma_p^2/\sigma_\tau^3 > 0$ holds for any $\sigma_p > 0$, and $\sigma_p/\sigma_s > a_s$ is equivalent to $\sigma_s^2 > s^2\sigma_p^2$, which holds for every $s < 1$. Hence Lemma~\ref{lem:app_monotone}, Proposition~\ref{prop:app_optimal}, and Theorem~\ref{thm:optimal} remain valid for any $\sigma_p > 0$.

\subsubsection{Connection to the optimal DP estimator}

\paragraph{The DP optimal estimator from~\citet{freirich2021theory}.}
The MMSE estimator $\hat{X}_0 = a_s x_s$ has standard deviation
$\hat\sigma = a_s\sigma_s = s\sigma_p^2/\sigma_s$, and its
perceptual gap to the prior is $P^* = W_2(p_{\hat{X}_0}, p_1) =
|\sigma_p - \hat\sigma|$. The $W_2$-optimal transport map from
$\ccalN(0,\hat\sigma^2)$ to $p_1 = \ccalN(0,\sigma_p^2)$ is the
scaling $T^* = \sigma_p/\hat\sigma$. By
\citet[Theorem~3]{freirich2021theory}, the minimum-distortion
deterministic estimator at perception level $P = \alpha P^*$
($\alpha \in [0,1]$) is the displacement interpolant
\begin{equation}
    \hat{X}_\alpha^{\mathrm{det}}
    = \bigl[\alpha + (1-\alpha)T^*\bigr] a_s\, x_s
    =: \Gamma_s(\alpha)\,x_s,
    \label{eq:app_Gamma}
\end{equation}
with $\Gamma_s(1) = a_s$ (MMSE) and
$\Gamma_s(0) = T^* a_s = \sigma_p/\sigma_s$ (perfect perception). The
corresponding distortion-perception function is
$D(P) = D^* + (P^* - P)^2$.

\begin{proposition}[Average denoiser is exactly optimal]
    \label{prop:app_optimal}
    For any $s \in (0,1)$ and $t \in [s,1]$, $\Lambda(s,t) =
    \Gamma_s(\alpha(s,t))$ where
    \begin{equation}
        \alpha(s,t) = \frac{\sigma_p/\sigma_s - \Lambda(s,t)}
        {\sigma_p/\sigma_s - a_s}
    \end{equation}
    decreases monotonically from $1$ at $t=s$ to $0$ at $t=1$.
\end{proposition}

\begin{proof}
By~\eqref{eq:app_Gamma}, $\Gamma$ is affine in $\alpha$ and maps
$[0,1]$ bijectively onto $[a_s,\sigma_p/\sigma_s]$ (since
$\sigma_p/\sigma_s > a_s$ whenever $s < 1$). By the boundary
values of $\Lambda$ and Lemma~\ref{lem:app_monotone}, $\Lambda(s,\cdot)$
is continuous and strictly increasing on $[s,1]$ with
$\Lambda(s,s) = a_s = \Gamma_s(1)$ and
$\Lambda(s,1) = \sigma_p/\sigma_s = \Gamma_s(0)$. Inverting the affine
relation $\Gamma_s(\alpha) = \Lambda(s,t)$ yields $\alpha(s,t)$ as
stated, and the monotonicity follows because $\Gamma$ is decreasing
in $\alpha$ while $\Lambda$ is increasing in $t$.
\end{proof}

\paragraph{Summary.}
As $t$ increases from $s$ to $1$, the average denoiser gain
$\Lambda(s,t)$ traces the affine segment from the MMSE gain $a_s$
to the perfect-perception gain $\sigma_p/\sigma_s$. This segment
parametrizes the $W_2$ geodesic from the MMSE output distribution
$p_{\hat{X}_0}$ to the target $p_1$, and by
\citet[Theorem~3]{freirich2021theory}, every point on it achieves
the minimum distortion at its perception level.
A visual comparison on images is shown in Fig.~\ref{fig:visual}.

\subsubsection{Multivariate Gaussian targets}
\label{app:multivariate}

\paragraph{Isotropic and correlated targets.}
Let $p_1 = \ccalN(\mathbf{0}, \bbSigma)$ with $\bbSigma = \bbU \bbLambda_{\Sigma} \bbU^\top$ and $\bbLambda_\Sigma = \mathrm{diag}(\sigma_{p,1}^2, \ldots, \sigma_{p,d}^2)$. Since the source is isotropic, $\bbU^\top \bbx_0 \sim \ccalN(\mathbf{0}, \bbI)$, so the rotated variable $\tilde{\bbx}_\tau = \bbU^\top \bbx_\tau$ is the linear interpolant between $\ccalN(\mathbf{0},\bbI)$ and $\ccalN(\mathbf{0},\bbLambda_\Sigma)$, whose coordinates are independent. The velocity, flow map, and average denoiser are equivariant under this rotation, so the problem reduces to $d$ independent scalar problems, each described by Appendix~\ref{app:gaussian} with its own $\sigma_{p,i}$, gains $a_{s,i}$ and $\Lambda_i(s,t)$, and perception level
\begin{equation}
    \alpha_i(s,t) = \frac{\sigma_{p,i}/\sigma_{s,i} - \Lambda_i(s,t)}{\sigma_{p,i}/\sigma_{s,i} - a_{s,i}}.
\end{equation}
In particular, for isotropic targets ($\sigma_{p,i} \equiv \sigma_p$), all coordinates share the same $\alpha(s,t)$ and Theorem~\ref{thm:optimal} holds verbatim.

\paragraph{Anisotropic targets.}
The optimal DP estimator at perception level $P$~\citep[Theorem~3]{freirich2021theory} applies the \emph{same} interpolation weight $\alpha = P/P^*$ in every direction, whereas a shared lookahead induces the direction-dependent weights $\alpha_i(s,t)$. Consequently:
\begin{itemize}
    \item the endpoints are exactly optimal for any covariance: $t=s$ gives $\alpha_i \equiv 1$ (MMSE) and $t=1$ gives $\alpha_i \equiv 0$ (perfect perception);
    \item intermediate lookaheads are exactly optimal if and only if all $\alpha_i(s,t)$ coincide, which holds when the spectrum is flat;
    \item otherwise, the deviation from optimality is governed by the spread $\max_i \alpha_i(s,t) - \min_i \alpha_i(s,t)$, which grows with the spread of the spectrum and vanishes as it flattens.
\end{itemize}
A single scalar lookahead, as used in Algorithm~\ref{alg:pnp_fm}, therefore does not realize the optimal DP curve for anisotropic targets at intermediate operating points.

\paragraph{Recovering the optimal curve by endpoint interpolation.}
Exact optimality is nevertheless attainable from the same trained model. The endpoint $D_{s,s}$ is the MMSE estimator, and $D_{s,1}$ coincides with the optimal perfect-perception estimator~\citep[Theorem~2]{freirich2021theory}: the $W_2$-optimal map from the distribution of the MMSE estimate to $p_1$ scales direction $i$ by $\sigma_{p,i}/(a_{s,i}\sigma_{s,i})$, which applied to $a_{s,i} \tilde{x}_{s,i}$ gives exactly $(\sigma_{p,i}/\sigma_{s,i})\,\tilde{x}_{s,i} = \Lambda_i(s,1)\,\tilde{x}_{s,i}$. Since the optimal DP estimators form the convex combination of these two endpoints~\citep[Theorem~3]{freirich2021theory}, the family
\begin{equation}
    \hat{\bbx}_\lambda = (1-\lambda)\,D_{s,1}(\bbx_s) + \lambda\,D_{s,s}(\bbx_s), \qquad \lambda \in [0,1],
    \label{eq:endpoint_interp}
\end{equation}
traces the optimal DP curve exactly for any Gaussian target, with $\alpha_i = \lambda$ in every direction, at the cost of two network evaluations instead of one. For isotropic targets, $\Lambda(s,t)$ is itself a convex combination of the endpoint gains, so the lookahead and~\eqref{eq:endpoint_interp} trace the same curve under different parametrizations. Appendix~\ref{app:interp_vs_lookahead} compares both constructions on images.

\paragraph{Empirical (discrete) target distributions.}
For an empirical target $p_1 = \frac{1}{n}\sum_{i} \delta_{\bbx^{(i)}}$, the instantaneous denoiser is a softmax-weighted average of the atoms $\bbx^{(i)}$. At $t=s$, $D_{s,s}$ returns this average (the empirical analogue of the blurry MMSE estimate), while at $t=1$ it returns the endpoint of the ODE, a single atom, so that the perception index is zero. The qualitative DP behavior therefore persists (see also Fig.~\ref{fig:teaser}). However, we cannot claim optimality at intermediate lookaheads: the proof of Theorem~\ref{thm:optimal} relies on the linearity of $D_{s,t}$, which fails here, and the path induced by the ODE need not follow the $W_2$ geodesic.

% [Full proof from the Gaussian analysis document.]
 
% \subsection{Mixture-of-Gaussians details}
% %-----------------------------------------------------------------------------
 
% For non-Gaussian priors, the velocity field is nonlinear and the flow
% map is no longer a scalar multiplication. We validate the method
% numerically on a symmetric mixture of Gaussians
% $p_1 = \tfrac{1}{2}\ccalN(\mu, \sigma_1^2) +
% \tfrac{1}{2}\ccalN(-\mu, \sigma_1^2)$ with $\mu=2$, $\sigma_1=1$.
 
% The conditional mean involves a softmax over the two component
% likelihoods, introducing a $\tanh$-like nonlinearity in the velocity
% field. We compute the average denoiser by numerically integrating the
% ODE and evaluate the DP curve by Monte Carlo with 30,000 samples.
 
% Fig.~\ref{fig:mixture_pd} shows the results. The average denoiser
% traces a smooth, monotone path through the DP plane, close to the
% optimal Freirich frontier in the moderate-perception regime. As the
% lookahead increases, the denoiser map transitions from a smooth MMSE
% estimator to a sharp classifier-like map that separates the mixture
% components (Fig.~\ref{fig:mixture_maps}).
 
% \begin{figure}[t]
%     \centering
%     \includegraphics[width=0.48\textwidth]{pd_curve_mixture.pdf}
%     \hfill
%     \includegraphics[width=0.48\textwidth]{denoiser_maps.pdf}
%     \caption{\textbf{Left}: DP tradeoff for the mixture-of-Gaussians
%     prior. Dashed: optimal curve; solid: average denoiser. \textbf{Right}:
%     denoiser maps $D_{s,t}(x_s)$ for different lookaheads.}
%     \label{fig:mixture_pd}
% \end{figure}

\begin{figure}[H] % Use figure* if you are in a two-column format so it spans the whole top page
    \centering
        % Note: width=\textwidth here fills the 0.62 allocated above
    \includegraphics[width=\textwidth]{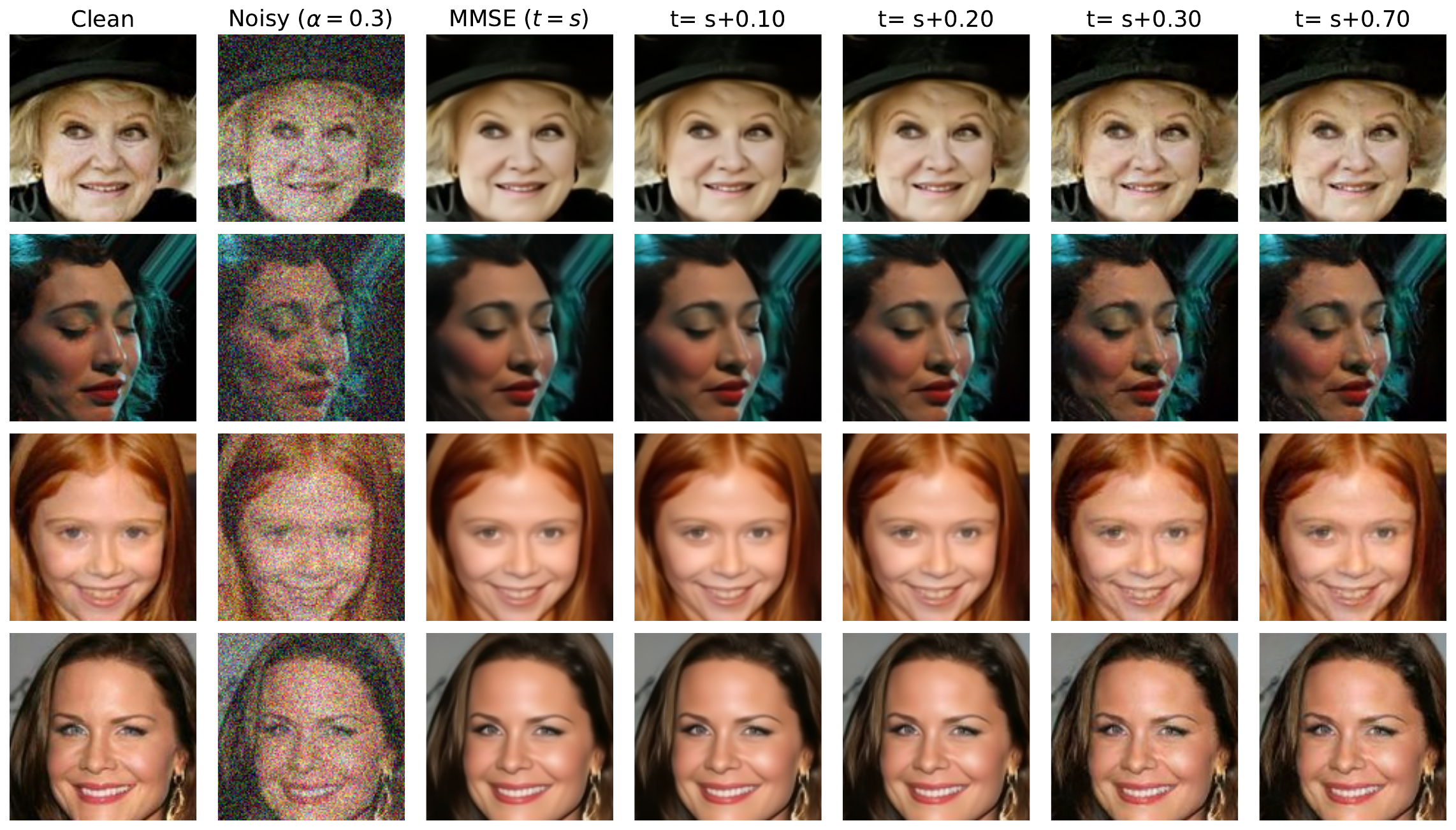} 
    % --- Main Caption ---
    \caption{Visual comparison for $\xi = 0.3$. From left to right: clean image, noisy input, and reconstructions at increasing lookahead. The estimator smoothly transitions from a blurry conditional mean (MMSE) to a sharp sample, progressively restoring texture, contrast, and color saturation.}
    \label{fig:visual}
\end{figure}

\subsection{Experiment on Mixture-of-Gaussians}
\label{app:mog_pd}

\subsubsection{MMSE denoiser for a mixture of Gaussians}

\paragraph{Setup.}
Let $p_0 = \mathcal{N}(0, \mathbf{I})$ and $p_1 = \frac{1}{2}\mathcal{N}(\bbmu_a, \sigma_c^2 \mathbf{I}) + \frac{1}{2}\mathcal{N}(\bbmu_b, \sigma_c^2 \mathbf{I})$. Under the linear interpolant $\bbx_t = (1-t)\bbx_0 + t \bbx_1$, we seek to compute the MMSE denoiser:
\begin{equation}
    D_t(\bbx) = \mathbb{E}[\bbx_1 \mid \bbx_t = \bbx].
\end{equation}

\paragraph{Step 1: Marginal distribution of $\bbx_t$.}
Conditioned on $\bbx_1$ originating from component $k \in \{a, b\}$, we have $\bbx_1 \sim \mathcal{N}(\bbmu_k, \sigma_c^2 \mathbf{I})$ and $\bbx_0 \sim \mathcal{N}(0, \mathbf{I})$. Thus:
\begin{equation}
    \bbx_t \mid \text{component } k \;\sim\; \mathcal{N}\!\left(t\,\bbmu_k,\; \sigma_t^2 \mathbf{I}\right), \quad \text{where } \sigma_t^2 = (1-t)^2 + t^2 \sigma_c^2.
\end{equation}
The marginal distribution is the mixture:
\begin{equation}
    p(\bbx_t) = \frac{1}{2}\,\mathcal{N}(\bbx_t;\, t\bbmu_a,\, \sigma_t^2 \mathbf{I}) + \frac{1}{2}\,\mathcal{N}(\bbx_t;\, t\bbmu_b,\, \sigma_t^2 \mathbf{I}).
\end{equation}

\paragraph{Step 2: Posterior weights.}
By Bayes' rule, the posterior probability $w_k(\bbx, t)$ that $\bbx_1$ originated from component $k$ given $\bbx_t = \bbx$ is:
\begin{equation}
    w_k(\bbx, t) = \frac{\mathcal{N}(\bbx;\, t\bbmu_k,\, \sigma_t^2 \mathbf{I})}{\mathcal{N}(\bbx;\, t\bbmu_a,\, \sigma_t^2 \mathbf{I}) + \mathcal{N}(\bbx;\, t\bbmu_b,\, \sigma_t^2 \mathbf{I})}.
\end{equation}
In log-space, this is proportional to $-\frac{\|\bbx - t\bbmu_k\|^2}{2\sigma_t^2}$, and the weights are normalized via softmax such that $w_a + w_b = 1$.

\paragraph{Step 3: Conditional mean within components.}
Given that $\bbx_1$ comes from component $k$, the posterior mean of $\bbx_1$ given $\bbx_t = \bbx$ is:
\begin{equation}
    \mathbb{E}[\bbx_1 \mid \bbx_t = \bbx, \text{component } k] = \bbmu_k + \kappa_t(\bbx - t\,\bbmu_k),
\end{equation}
where $\kappa_t = \frac{t\,\sigma_c^2}{\sigma_t^2}$ is the regression coefficient (Tweedie's formula applied component-wise).

\paragraph{Step 4: Full MMSE denoiser.}
By the law of total expectation, the denoiser is the weighted average:
\begin{equation}
    D_t(\bbx) = w_a(\bbx,t)\,\left[\bbmu_a + \kappa_t(\bbx - t\bbmu_a)\right] + w_b(\bbx,t)\,\left[\bbmu_b + \kappa_t(\bbx - t\bbmu_b)\right].
\end{equation}
This is a softmax-weighted average of the within-component posterior means. Its behavior is intuitive:
\begin{itemize}
    \item When $\bbx$ is close to $t\bbmu_a$, $w_a \approx 1$ and $D_t(\bbx) \approx \bbmu_a + \kappa_t(\bbx - t\bbmu_a)$ (assigned to component $a$).
    \item When $\bbx$ is equidistant from $t\bbmu_a$ and $t\bbmu_b$, $w_a \approx w_b \approx 0.5$, and $D_t(\bbx)$ averages the two, producing an output between the modes (the posterior averaging artifact).
\end{itemize}

\noindent We compute the average denoiser by ODE integration using the instantaneous velocity field.
The nonlinearity introduced by the softmax weights makes the MoG fundamentally different from the single Gaussian case: the denoiser is non-linear in $\bbx$, the flow map lacks a closed-form solution, and the posterior averaging artifact (blurring between modes) emerges as a direct consequence of the mixture structure.
However, given the simplicity of the instantaneous denoiser, we can approximate the average denoiser numerically.

\subsubsection{DP analysis}

In this section, we consider a controlled setting based on a 2D mixture of Gaussians. 
For this distribution, the marginal velocity admits a closed-form expression, allowing both the average denoiser $D_{s,t}$ and the optimal DP estimator of~\citet{freirich2021theory} to be evaluated with high precision, bypassing the requirement of a learned model.

\paragraph{Setup.}
We consider $\bbmu_a = (-2, 0)$, $\bbmu_b = (2, 0)$, and $\sigma_c = 0.3$, and compute the average denoiser $D_{s,t}$ by integrating $\dot{x}_\tau = v(x_\tau, \tau)$ from $\tau = s$ to $\tau = t$ with a fourth-order Runge--Kutta scheme (200 steps). 

\paragraph{Optimal DP construction.}
To construct the optimal DP curve, we follow \citet[Theorem~3]{freirich2021theory}. 
At perception level $P \in [0, P^*]$, the optimal estimator is given by~\eqref{eq:optimal_pd}.

We approximate this construction empirically by drawing $N=5000$ samples, computing $\hat{\bbx}_0^{(i)} = D_s(\bbx_s^{(i)})$, and solving the discrete optimal transport problem between the empirical distributions $\{\hat{\bbx}_0^{(i)}\}$ and $\{\bbx_1^{(j)}\} \sim p_1$ using~\citep{flamary2021pot}. 
Let $\pi$ denote the optimal coupling; we then form the interpolated samples as
\[
\hat{\bbx}_P^{(i)} = \left(1 - \frac{P}{P^*}\right)\hat{\bbx}_0^{(i)} + \frac{P}{P^*} \bbx_1^{(\pi(i))}.
\]

Distortion is measured as MSE with respect to the ground truth, while perception is quantified via the squared 2-Wasserstein distance $W_2^2(\hat{p}, p_1)$ between empirical distributions.

\paragraph{Results.}
Fig.~\ref{fig:mog_dp} compares the resulting curves at three noise levels $s \in {0.3, 0.5, 0.7}$, corresponding to heavy, moderate, and light noise (recall that smaller $s$ implies heavier noise, with $\xi := 1 - s$). Three observations stand out.

\begin{enumerate}

\item \textbf{Endpoints coincide at MMSE.}
At $t = s$ (equivalently, $\alpha = 1$), both curves reduce to the MMSE estimator by construction, yielding identical distortion–perception coordinates across all panels.

\item \textbf{Near-optimal behavior in the interior.}
For intermediate lookaheads, the average denoiser traces a DP curve that lies slightly above the Freirich optimum, with a small and nearly uniform gap. At $s = 0.3$ (heavy noise), this gap remains below $0.02$ in MSE across the full perception range. This discrepancy is not a numerical artifact: it persists when using higher-order integration (RK4) and when increasing the sample size from $N=2000$ to $N=5000$. Instead, it reflects a structural limitation—while the flow-map trajectory matches the $W_2$ geodesic in the Gaussian case, it deviates slightly for nonlinear velocity fields (here induced by the softmax-weighted mixture).

\item \textbf{Endpoint behavior at $t=1$.}
At the perceptual endpoint, the distribution $D_{s,1}(p_{x_s})$ closely approximates $p_1$ but does not match it exactly, resulting in a small residual $W_2^2$ gap. This gap decreases as noise increases (from $9.4 \times 10^{-3}$ at $s=0.5$ to $3.8 \times 10^{-3}$ at $s=0.3$, compared to a sampling floor of $2.2 \times 10^{-3}$). This is consistent with the intuition that the integrable $1/(1-\tau)$ singularity in the velocity field has a stronger effect when the final integration interval is short.

\end{enumerate}

Taken together, these results show that the average denoiser is \emph{near-optimal} in a controlled setting: the resulting curves are smooth, monotone, and convex, match the MMSE endpoint exactly, and exhibit only small, well-characterized deviations from the Freirich optimum.
This complements the exact optimality established in the Gaussian case (Theorem~\ref{thm:optimal}) and supports the behavior observed at image scale (Fig.~\ref{fig:pd_curves}).

\begin{figure}[H]
    \centering
    \includegraphics[width=\textwidth]{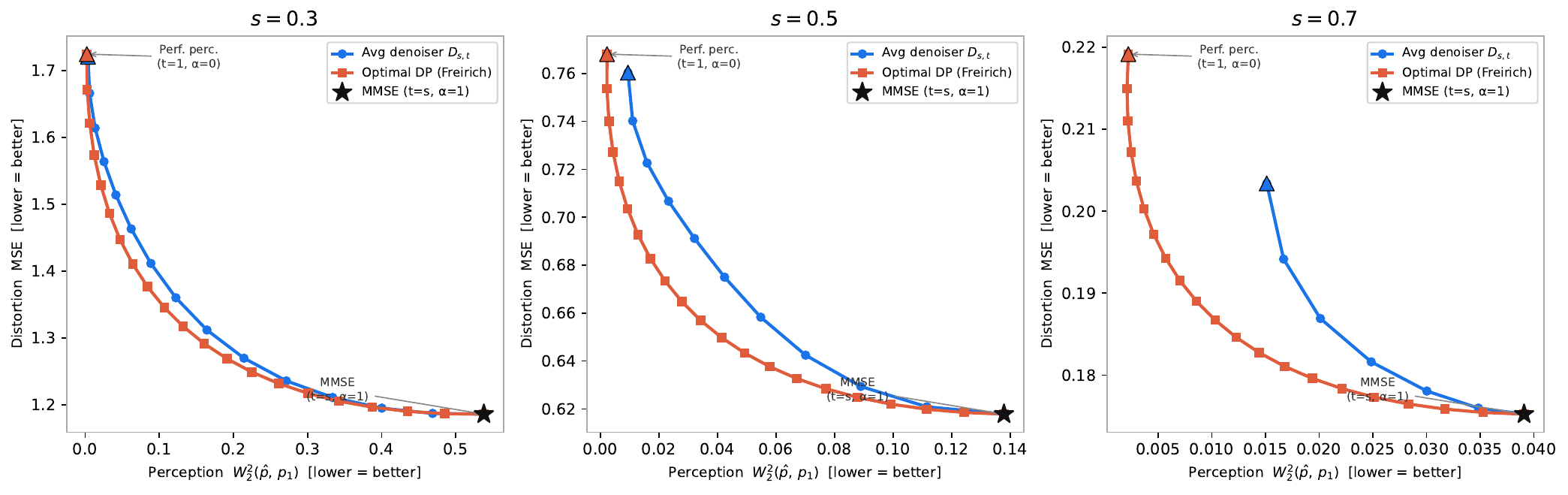}
    \caption{Distortion--perception comparison on a 2D mixture of
    Gaussians at three noise levels $s \in \{0.3, 0.5, 0.7\}$
    (heavy, moderate, light noise). Blue: average denoiser
    $D_{s,t}$ swept over lookahead $t \in [s, 1]$. Red: optimal DP estimator of~\citet{freirich2021theory} swept over perception level
    $\alpha \in [0, 1]$, computed via exact W\textsubscript{2}-optimal
    transport between the MMSE pushforward and $p_1$. Both curves share
    the MMSE endpoint exactly (\ding{72}); small residual gaps at the
    perfect-perception extreme (\ding{115}) reflect the velocity
    singularity near $\tau = 1$. The near-optimality of
    $D_{s,t}$ in the interior confirms that the flow-map trajectory
    closely approximates the W\textsubscript{2} geodesic even in a
    non-Gaussian setting, bridging Theorem~\ref{thm:optimal} to the
    image experiments in Section~\ref{sec:learned}.}
    \label{fig:mog_dp}
\end{figure}

\subsection{Empirical analysis of the PnP iteration}
\label{app:lipschitz}

\paragraph{Convergence.}
We do not provide convergence guarantees for Algorithm~\ref{alg:pnp_fm}. Empirically, convergence is enforced by the annealing schedule ($s_k \to 1$ shrinks the denoiser's contribution), not by contraction: the measured local Lipschitz constant of the PnP operator exceeds one and grows with $t$ (see below).

Below, we characterize the behavior of the iteration empirically.
To measure regularity, we consider linear inverse problems $f(\bbx)=\bbH\bbx$ and the deterministic variant of the iteration at a fixed noise level $s$ (i.e., $\bbepsilon = \mathbf{0}$), in which the renoising step reduces to $R_s(\bbz) = s\,\bbz$ and one iteration is the operator
\begin{equation}
T_t(\bbx) := D_{s,t}\bigl(R_s(F_\lambda(\bbx))\bigr), 
\quad \text{with} \quad 
F_\lambda(\bbx) = \bbx - \lambda \nabla g(\bbx),
\quad R_s(\bbz) = s\,\bbz.
\label{eq:Tt_operator}
\end{equation}
Let $L_D(t)$ denote the Lipschitz constant of $D_{s,t}$ and $L_F = \|\bbI - \lambda\,\bbH^\top\bbH\|$ that of $F_\lambda$, so that $T_t$ is $s\,L_D(t)\,L_F$-Lipschitz. The gradient step satisfies $L_F \le 1$ for $\lambda \in (0,\,2/\|\bbH^\top\bbH\|]$, with $L_F = 1$ when $\bbH^\top\bbH$ is singular, as in inpainting.

\paragraph{The composite bound exceeds one.}
Fig.~\ref{fig:lipschitz_deblur} reports $s\,L_D(t)\,L_F$ on CelebA $128 \times 128$ for Gaussian deblurring (kernel size $45$, standard deviation $3.0$, $\lambda = 1.0$, giving $L_F = 1.000$). The denoiser constant $L_D(t)$ is estimated by power iteration on the Jacobian of $D_{s,t}$ at noisy samples $\bbx_s = s\,\bbx_1 + (1-s)\,\bbepsilon$, taking the worst case over $20$ samples. Across all $(s, t)$ tested, the bound exceeds $1$ and increases with $t$, so $T_t$ is not a global contraction.

\paragraph{Local regularity at the iterates.}
We also measure the local Lipschitz constant $\hat\rho = \sigma_{\max}(\partial T_t/\partial \bbx)$ of the full composite operator by power iteration on its Jacobian, evaluated at noisy samples $\bbx_s = s\,\bbx_1 + (1-s)\,\bbepsilon$ and at the actual iterates of the algorithm (iterations $k=5$ and $k=250$, $8$ images), for Gaussian deblurring and for a well-conditioned control ($\bbH = \bbI$, $\lambda = 0.5$); see Fig.~\ref{fig:rho_hat}. We observe that:
\begin{enumerate}[label=(\roman*)]
    \item at noisy samples, $\hat\rho$ increases monotonically with the lookahead at every noise level; e.g., at $s=0.5$, $\hat\rho$ grows from $1.35$ ($t=s$) to $1.85$, $2.64$, and $3.17$ ($t=1$);
    \item at all operating points, $\hat\rho \in [1.3, 9.5]$ at the iterates, including the well-conditioned control, so the iterates do not stabilize because of local contraction; correspondingly, the fixed-$s$ iteration does not converge (the ratio of successive residuals is $1.000 \pm 10^{-4}$); at the iterates, $\hat\rho$ is also larger at $t=1$ than at $t=s$ in five of six settings, but not monotone in between;
    \item the deployed algorithm, with the annealed schedule $s_k = k/N$, stabilizes approximately geometrically, with a ratio of successive residuals between $0.958$ and $0.962$ that does not depend on $t$.
\end{enumerate}
Together, these measurements indicate that the stability of Algorithm~\ref{alg:pnp_fm} is enforced by the annealing schedule, since the correction applied by the denoiser vanishes as $s_k \to 1$, and not by contractivity of the denoiser. Global convergence of PnP without contractivity can be established under other assumptions, such as averagedness~\citep{pesquet2021learning} or gradient-step denoisers in a non-convex optimization framework~\citep{hurault2022gradient}; extending such analyses to flow-map denoisers is left for future work. The growth of $\hat\rho$ with $t$ is also consistent with the asymmetry between the two endpoints observed in Appendix~\ref{app:num_steps_ablation}.

\begin{figure}[H]
    \centering
    \includegraphics[width=\linewidth]{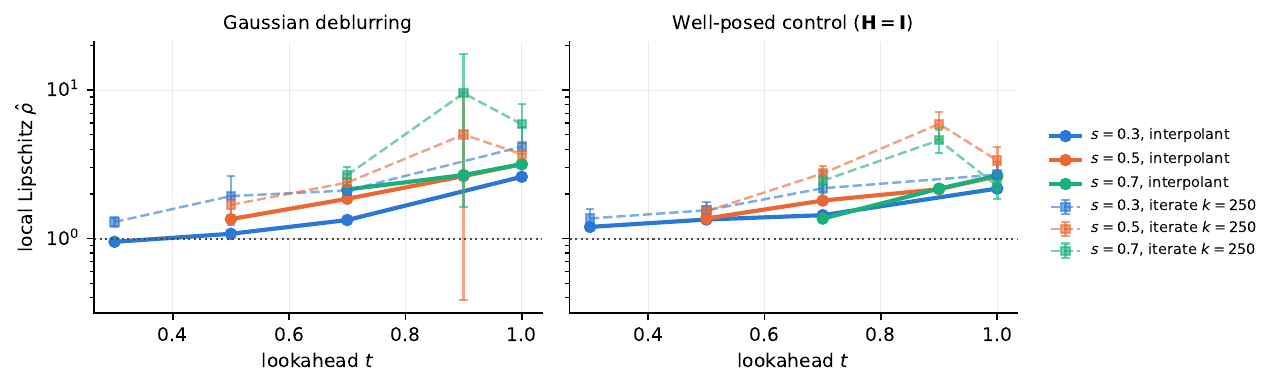}
    \caption{\small{Local Lipschitz constant $\hat\rho = \sigma_{\max}(\partial T_t/\partial \bbx)$ of the deterministic PnP operator on CelebA, as a function of the lookahead $t$, for three noise levels $s$ (log scale; mean over $8$ images). Solid: at noisy samples $\bbx_s$; dashed: at the iterates of the fixed-$s$ iteration ($k=250$, error bars: one standard deviation). The dotted line marks $\hat\rho = 1$.}}
    \label{fig:rho_hat}
\end{figure}

\begin{figure}[H]
    \centering
    \includegraphics[width=1.0\linewidth]{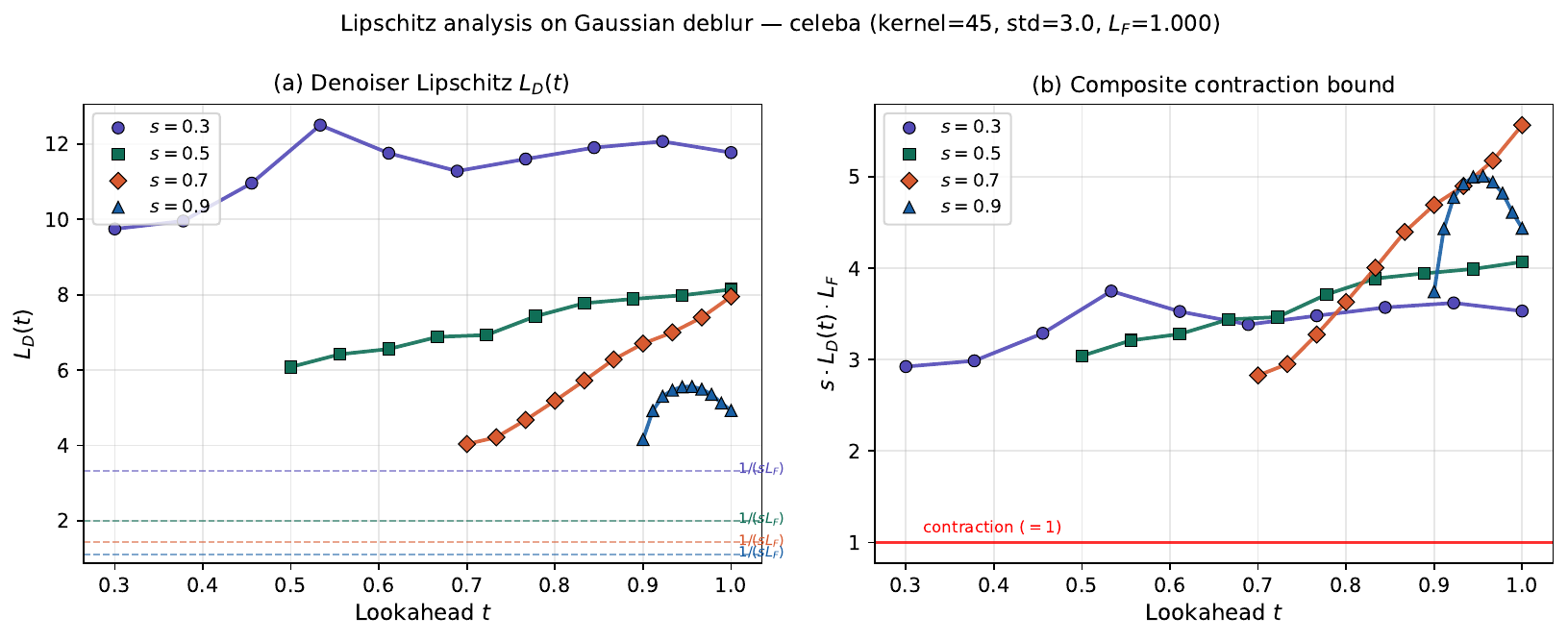}
    \caption{Composite bound $s\,L_D(t)\,L_F$ for the deterministic PnP operator on CelebA Gaussian deblurring. The bound exceeds $1$ for all $(s,t)$, so the operator is not a global contraction. $L_D(t)$ is estimated via power iteration on $\partial D_{s,t}/\partial \bbx$ at noisy samples (worst case over $20$). The bound increases with $t$, with the steepest growth at low noise (large $s$).}
    \label{fig:lipschitz_deblur}
\end{figure}

\subsection{Gaussian linear inverse problems: fixed point of the iteration}
\label{app:fixed_point}

\paragraph{What does the iteration target?}
PnP is derived from proximal splitting and targets a MAP solution, and plugging in our denoiser does not change this: the lookahead does not switch between optimization and posterior sampling, but changes the implicit regularizer, as a denoiser is the proximal operator of a smoothed version of the regularizer rather than of $-\log p$ itself~\citep{gribonval2021bayesian}. In the Gaussian linear case, the fixed point is a MAP estimator under a surrogate Gaussian prior whose variance grows with $t$ (see below); beyond this case we have no closed form, and we state the resulting behavior as an empirical observation.

We now extend the analysis of Section~\ref{sec:gaussian}, where the denoiser was treated as a standalone estimator, to linear inverse problems with an isotropic Gaussian prior $p_1 = \ccalN(\mathbf{0},\,\sigma_p^2\,\bbI)$, in order to characterize \emph{what} the deterministic iteration~\eqref{eq:Tt_operator} computes. In this case, the average denoiser is the linear map $D_{s,t}(\bbx) = \Lambda(s,t)\,\bbx$, with the scalar gain~\eqref{eq:Lambda}, and the operator $T_t$ is affine,
\begin{equation}
    T_t(\bbx) = m(t)\,\bigl(\bbx - \lambda \bbH^\top(\bbH\bbx - \bby)\bigr), \qquad m(t) := s\,\Lambda(s,t).
\end{equation}
Its fixed point admits the closed form
\begin{equation}
    \bbx^\star(t) \;=\; \bigl(c(t)\,\bbI + \lambda\,\bbH^\top\bbH\bigr)^{-1}\, \lambda\,\bbH^\top\bby, \qquad c(t) \;=\; \frac{1}{m(t)} - 1,
    \label{eq:tikhonov}
\end{equation}
i.e., a Tikhonov-regularized estimator whose regularization weight $c(t)$ depends on the lookahead. We have $m(t) \in (0,1)$, so $c(t) > 0$; since $\Lambda(s,t)$ is strictly increasing in $t$ (Lemma~\ref{lem:app_monotone}), $c(t)$ is strictly decreasing, so the regularization progressively weakens as the lookahead increases.

\paragraph{Interpretation as a MAP estimator.}
Equivalently, $\bbx^\star(t)$ minimizes $\tfrac{1}{2\sigma_{\rm obs}^2}\|\bbH\bbx - \bby\|^2 + \tfrac{1}{2\tau^2(t)}\|\bbx\|^2$ with $\tau^2(t) = \lambda\sigma_{\rm obs}^2/c(t)$, where $\sigma_{\rm obs}^2$ is the measurement noise variance. It is therefore the MAP estimator under a \emph{surrogate} Gaussian prior $\ccalN(\mathbf{0}, \tau^2(t)\bbI)$, whose variance grows with $t$. It does not coincide with the MAP (Wiener) estimator under the true prior, which would require $c(t)/\lambda = \sigma_{\rm obs}^2/\sigma_p^2$: the flow-map denoiser operates at the interpolant noise level $\sigma_s$ of $\bbx_s$ rather than at the measurement noise level, so it is not the proximal operator of $-\log p_1$ with the corresponding weight. This is an instance of the general fact that a denoiser is the proximal operator of a smoothed regularizer rather than of the negative log-prior itself~\citep{gribonval2021bayesian}. Hence, in this setting, the lookahead does not move the output away from a MAP-type solution; it changes the implicit regularizer, and thus acts as a regularization parameter.

\subsubsection{Proof of Equation~\eqref{eq:tikhonov}}
\label{app:proof_tikhonov}

By definition, the fixed point $\bbx^\star(t)$ satisfies $\bbx^\star(t) = T_t(\bbx^\star(t))$, i.e.,
\begin{equation*}
    \bbx^\star(t) = m(t)\bbx^\star(t) - m(t)\lambda \bbH^\top\bbH\bbx^\star(t) + m(t)\lambda \bbH^\top\bby.
\end{equation*}
Since $m(t) > 0$, dividing by $m(t)$ and rearranging gives
\begin{equation*}
    \biggl[\left(\frac{1}{m(t)} - 1\right)\bbI + \lambda \bbH^\top\bbH\biggr]\bbx^\star(t) = \lambda \bbH^\top\bby.
\end{equation*}
The matrix on the left-hand side is positive definite because $c(t) = 1/m(t) - 1 > 0$, and inverting it yields~\eqref{eq:tikhonov}.

\paragraph{Empirical validation.}
To corroborate the closed form~\eqref{eq:tikhonov}, we consider a 2D Gaussian prior $p_1 = \ccalN(\mathbf{0}, \sigma_p^2\,\bbI)$ with $\sigma_p = 2$, and two linear forward models: an isotropic rotation $\bbH_1 = 0.8\,R(\pi/6)$ and an anisotropic operator $\bbH_2 = R(\pi/6)\,\mathrm{diag}(1.0,\,0.3)\,R(\pi/4)$, whose Gram matrix $\bbH_2^\top\bbH_2$ has eigenvalues $\{1.0,\,0.09\}$ and couples the coordinates.
Observations are $\bby = \bbH_k\,\bbx_1 + \eta$ with
$\eta \sim \ccalN(\mathbf{0}, \sigma_{\rm obs}^2\,\bbI)$, $\sigma_{\rm obs} = 0.1$.

We run deterministic PnP-Flow ($\bbepsilon=\mathbf{0}$) with $\lambda=0.1$ and $0.5$, anchor $s=0.45$, and lookaheads $t \in \{s,\, s+0.27,\, 1\}$. 
The iterates converge to $\bbx^\star(t)$ to machine precision ($\|\hat{\bbx}-\bbx^\star(t)\|_{\rm mean}<5\cdot10^{-16}$), validating~\eqref{eq:tikhonov}.
As $t$ increases, the effective regularization $c(t)/\lambda$ decreases ($0.93 \to 0.59 \to 0.43$), yielding a continuum of Tikhonov estimators with progressively weaker shrinkage toward the prior mean.
The result is shown in Fig.~\ref{fig:gaussian_pnp}.

\begin{figure}[H] % Use figure* if you are in a two-column format so it spans the whole top page
    \centering
    
    % --- Top Row: Three Plots ---
    \begin{subfigure}[b]{1.0\textwidth}
        \centering
        \includegraphics[width=1.01\textwidth]{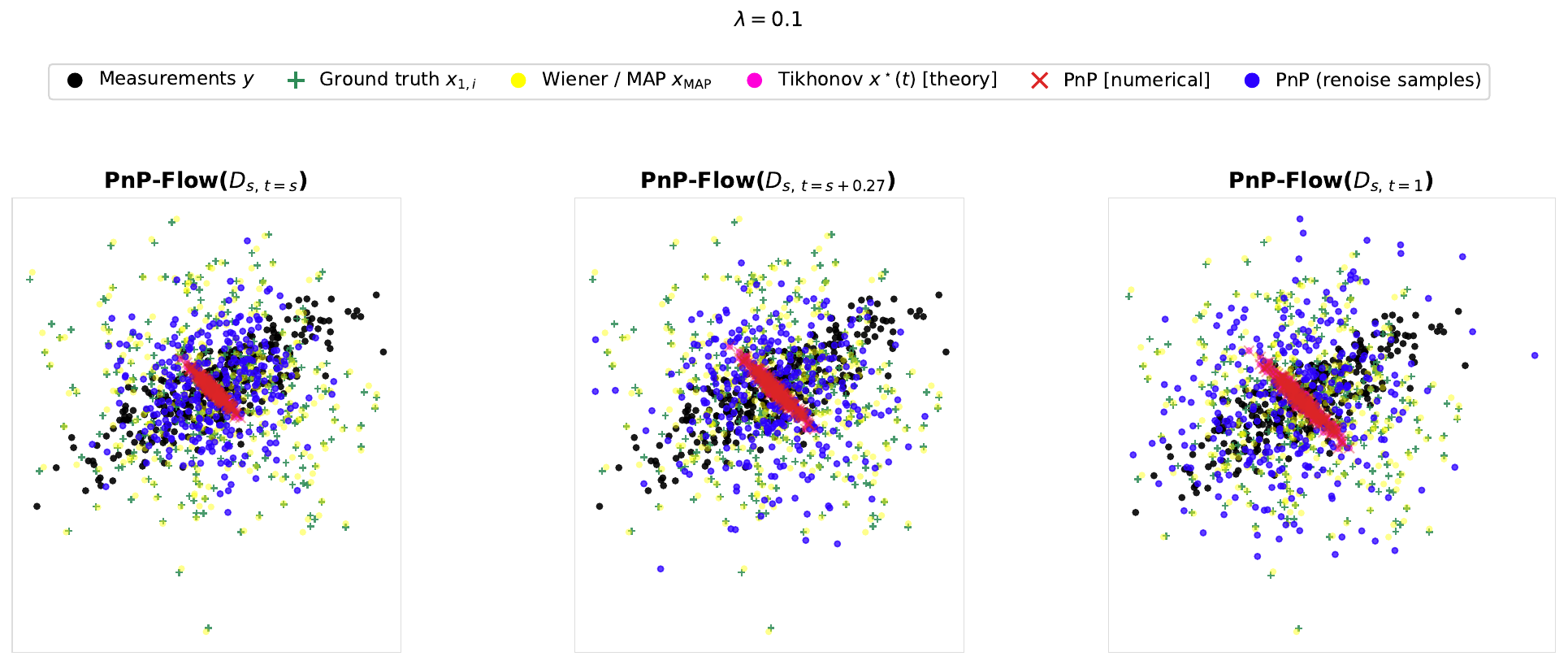}
        \caption{}
    \end{subfigure}
    % \vspace{0.5cm} % Space between the top 4 plots and this row
        \begin{subfigure}[b]{1.0\textwidth}
            \centering
            \includegraphics[width=1.01\textwidth]{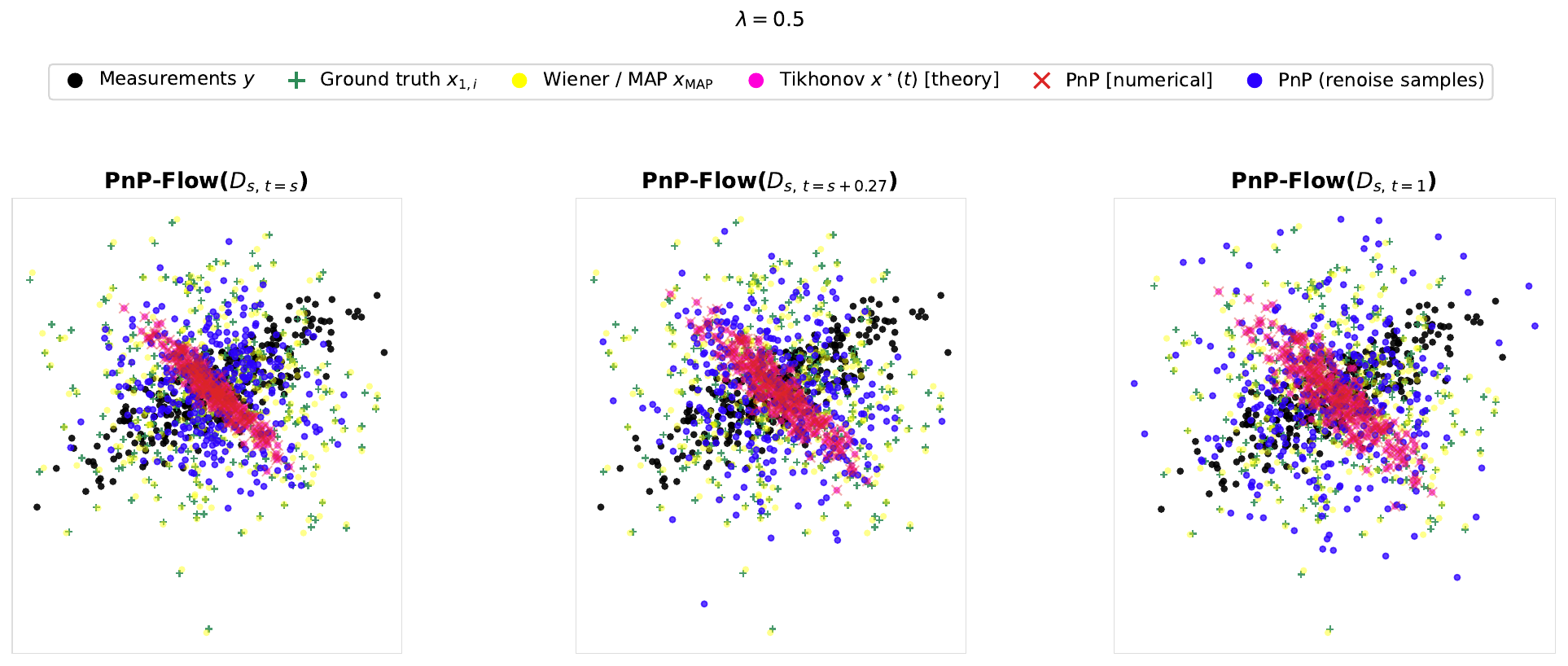}
            \caption{}
        \end{subfigure}
    % --- Main Caption ---
    \caption{{\small  Two forward models are tested: \textbf{(a)} isotropic rotation $\bbH_1 = 0.8\,R(\pi/6)$ with 
$\lambda = 0.1$, and \textbf{(b)} anisotropic operator 
$\bbH_2 = R(\pi/6)\,\mathrm{diag}(1.0, 0.3)\,R(\pi/4)$ with $\lambda = 0.5$. 
Black: measurements $\bby$. Gold crosses: ground truth $\bbx_{1,i}$. Yellow: Wiener/MAP estimator. 
Red dots: Tikhonov fixed point $\bbx^\star(t)$ predicted by Eq.~\eqref{eq:tikhonov}. 
Red $\times$: PnP iterates after convergence (matching theory to machine precision, 
$\|\hat{\bbx} - \bbx^\star(t)\|_{\rm mean} < 5\cdot 10^{-16}$). 
Blue dots: samples from the stochastic PnP variant, illustrating the variance restoration 
discussed in Section~\ref{sec:learned}. As $t$ increases (left to right), the Tikhonov regularization 
weakens and the iterates spread further from the prior mean toward the Wiener solution.}}
    \label{fig:gaussian_pnp}
\end{figure}

In this Gaussian inverse problem, the lookahead acts as a \emph{regularization knob} rather than as a parameterization of the optimal DP curve. Since this regime is over-regularized ($c(t)/\lambda$ exceeds $\sigma_{\rm obs}^2/\sigma_p^2$), increasing $t$ reduces bias and moves $\bbx^\star(t)$ toward the Wiener solution, which also decreases the MSE. This does not contradict the image experiments: exact DP optimality holds only in the denoising setting, while in inverse problems the interaction with the forward model breaks this guarantee.

\paragraph{Effect of renoising in the Gaussian setting.}
To isolate the role of the renoising step, we repeat the experiment with
the full stochastic PnP iteration. While the deterministic case converges
to a single Tikhonov solution, the stochastic iteration produces a
distribution of estimates whose spread is largely independent of the step
size $\lambda$.

Across different $\lambda$, the sample clouds consistently exhibit
meaningful variance and align with the geometry of the data distribution,
even when the mean estimate remains over-regularized. This indicates that
the renoising step, together with the denoiser, primarily controls the
variance of the estimator, while $\lambda$ governs its mean through the
data-consistency update.

As the lookahead increases, this variability is progressively amplified,
yielding the distortion--perception behavior observed in the experiment.

\section{Broader impacts}
\label{app:social_impacts}
Image restoration methods that generate content under measurement constraints can produce plausible but incorrect details, particularly at the perceptual end of the DP tradeoff and for degradations with large null spaces, such as inpainting. In sensitive applications, such as medical imaging, forensics, or security, such hallucinations could lead to wrong decisions. The lookahead studied in this work makes this tradeoff explicit and controllable, since users can select a low-distortion operating point when fidelity is critical, but reconstructions should still be validated before being used in high-stakes settings.

\newpage
\input{checklist.tex}

\end{document}

%% file: checklist.tex
\section*{NeurIPS Paper Checklist}

\begin{enumerate}

\item {\bf Claims}
    \item[] Question: Do the main claims made in the abstract and introduction accurately reflect the paper's contributions and scope?
    \item[] Answer: \answerYes{} % Replace by \answerYes{}, \answerNo{}, or \answerNA{}.
    \item[] Justification: The abstract and introduction state the three regimes of our claims: exact DP optimality for Gaussian denoising with isotropic targets (Theorem~\ref{thm:optimal}), empirical DP traversal for natural-image denoising (Section~\ref{sec:learned}), and empirical DP traversal without optimality guarantees for general inverse problems (Sections~\ref{sec:pnp_flow_map} and~\ref{sec:experiments}).
    We support this claim with an extensive ablation and comparison with baselines.
    % \justificationTODO{}
    \item[] Guidelines:
    \begin{itemize}
        \item The answer \answerNA{} means that the abstract and introduction do not include the claims made in the paper.
        \item The abstract and/or introduction should clearly state the claims made, including the contributions made in the paper and important assumptions and limitations. A \answerNo{} or \answerNA{} answer to this question will not be perceived well by the reviewers. 
        \item The claims made should match theoretical and experimental results, and reflect how much the results can be expected to generalize to other settings. 
        \item It is fine to include aspirational goals as motivation as long as it is clear that these goals are not attained by the paper. 
    \end{itemize}

\item {\bf Limitations}
    \item[] Question: Does the paper discuss the limitations of the work performed by the authors?
    \item[] Answer: \answerYes{} % Replace by \answerYes{}, \answerNo{}, or \answerNA{}.
    \item[] Justification: We discuss limitations in Section~\ref{sec:conclusions} and Appendix~\ref{app:failures_limitations}, including the scope of the Gaussian result (Section~\ref{sec:gaussian}, Appendix~\ref{app:multivariate}), the absence of convergence guarantees (Appendix~\ref{app:lipschitz}), non-monotone DP behavior, and training cost (Appendix~\ref{app:training_details}).
    \item[] Guidelines:
    \begin{itemize}
        \item The answer \answerNA{} means that the paper has no limitation while the answer \answerNo{} means that the paper has limitations, but those are not discussed in the paper. 
        \item The authors are encouraged to create a separate ``Limitations'' section in their paper.
        \item The paper should point out any strong assumptions and how robust the results are to violations of these assumptions (e.g., independence assumptions, noiseless settings, model well-specification, asymptotic approximations only holding locally). The authors should reflect on how these assumptions might be violated in practice and what the implications would be.
        \item The authors should reflect on the scope of the claims made, e.g., if the approach was only tested on a few datasets or with a few runs. In general, empirical results often depend on implicit assumptions, which should be articulated.
        \item The authors should reflect on the factors that influence the performance of the approach. For example, a facial recognition algorithm may perform poorly when image resolution is low or images are taken in low lighting. Or a speech-to-text system might not be used reliably to provide closed captions for online lectures because it fails to handle technical jargon.
        \item The authors should discuss the computational efficiency of the proposed algorithms and how they scale with dataset size.
        \item If applicable, the authors should discuss possible limitations of their approach to address problems of privacy and fairness.
        \item While the authors might fear that complete honesty about limitations might be used by reviewers as grounds for rejection, a worse outcome might be that reviewers discover limitations that aren't acknowledged in the paper. The authors should use their best judgment and recognize that individual actions in favor of transparency play an important role in developing norms that preserve the integrity of the community. Reviewers will be specifically instructed to not penalize honesty concerning limitations.
    \end{itemize}

\item {\bf Theory assumptions and proofs}
    \item[] Question: For each theoretical result, does the paper provide the full set of assumptions and a complete (and correct) proof?
    \item[] Answer: \answerYes{} % Replace by \answerYes{}, \answerNo{}, or \answerNA{}.
    \item[] Justification: Theorem~\ref{thm:optimal} states its assumptions, and its complete proof, together with the supporting lemmas and the multivariate discussion, is in Appendix~\ref{app_general:proofs}.
    \item[] Guidelines:
    \begin{itemize}
        \item The answer \answerNA{} means that the paper does not include theoretical results. 
        \item All the theorems, formulas, and proofs in the paper should be numbered and cross-referenced.
        \item All assumptions should be clearly stated or referenced in the statement of any theorems.
        \item The proofs can either appear in the main paper or the supplemental material, but if they appear in the supplemental material, the authors are encouraged to provide a short proof sketch to provide intuition. 
        \item Inversely, any informal proof provided in the core of the paper should be complemented by formal proofs provided in appendix or supplemental material.
        \item Theorems and Lemmas that the proof relies upon should be properly referenced. 
    \end{itemize}

    \item {\bf Experimental result reproducibility}
    \item[] Question: Does the paper fully disclose all the information needed to reproduce the main experimental results of the paper to the extent that it affects the main claims and/or conclusions of the paper (regardless of whether the code and data are provided or not)?
    \item[] Answer: \answerYes{} % Replace by \answerYes{}, \answerNo{}, or \answerNA{}.
    \item[] Justification: Our algorithm is clearly described in Alg.~\ref{alg:pnp_fm}, and all the hyperparameters and datasets used are explicitly reported (see Appendix~\ref{app:hyperparamters}).
    The source code is also provided in the submission (it will be publicly available on GitHub if the paper gets accepted).
    \item[] Guidelines:
    \begin{itemize}
        \item The answer \answerNA{} means that the paper does not include experiments.
        \item If the paper includes experiments, a \answerNo{} answer to this question will not be perceived well by the reviewers: Making the paper reproducible is important, regardless of whether the code and data are provided or not.
        \item If the contribution is a dataset and\slash or model, the authors should describe the steps taken to make their results reproducible or verifiable. 
        \item Depending on the contribution, reproducibility can be accomplished in various ways. For example, if the contribution is a novel architecture, describing the architecture fully might suffice, or if the contribution is a specific model and empirical evaluation, it may be necessary to either make it possible for others to replicate the model with the same dataset, or provide access to the model. In general. releasing code and data is often one good way to accomplish this, but reproducibility can also be provided via detailed instructions for how to replicate the results, access to a hosted model (e.g., in the case of a large language model), releasing of a model checkpoint, or other means that are appropriate to the research performed.
        \item While NeurIPS does not require releasing code, the conference does require all submissions to provide some reasonable avenue for reproducibility, which may depend on the nature of the contribution. For example
        \begin{enumerate}
            \item If the contribution is primarily a new algorithm, the paper should make it clear how to reproduce that algorithm.
            \item If the contribution is primarily a new model architecture, the paper should describe the architecture clearly and fully.
            \item If the contribution is a new model (e.g., a large language model), then there should either be a way to access this model for reproducing the results or a way to reproduce the model (e.g., with an open-source dataset or instructions for how to construct the dataset).
            \item We recognize that reproducibility may be tricky in some cases, in which case authors are welcome to describe the particular way they provide for reproducibility. In the case of closed-source models, it may be that access to the model is limited in some way (e.g., to registered users), but it should be possible for other researchers to have some path to reproducing or verifying the results.
        \end{enumerate}
    \end{itemize}

\item {\bf Open access to data and code}
    \item[] Question: Does the paper provide open access to the data and code, with sufficient instructions to faithfully reproduce the main experimental results, as described in supplemental material?
    \item[] Answer: \answerYes{} % Replace by \answerYes{}, \answerNo{}, or \answerNA{}.
    \item[] Justification: The code to reproduce the experiments is publicly available at \url{https://github.com/nzilberstein/Flow-map-denoisers}.
    \item[] Guidelines:
    \begin{itemize}
        \item The answer \answerNA{} means that paper does not include experiments requiring code.
        \item Please see the NeurIPS code and data submission guidelines (\url{https://neurips.cc/public/guides/CodeSubmissionPolicy}) for more details.
        \item While we encourage the release of code and data, we understand that this might not be possible, so \answerNo{} is an acceptable answer. Papers cannot be rejected simply for not including code, unless this is central to the contribution (e.g., for a new open-source benchmark).
        \item The instructions should contain the exact command and environment needed to run to reproduce the results. See the NeurIPS code and data submission guidelines (\url{https://neurips.cc/public/guides/CodeSubmissionPolicy}) for more details.
        \item The authors should provide instructions on data access and preparation, including how to access the raw data, preprocessed data, intermediate data, and generated data, etc.
        \item The authors should provide scripts to reproduce all experimental results for the new proposed method and baselines. If only a subset of experiments are reproducible, they should state which ones are omitted from the script and why.
        \item At submission time, to preserve anonymity, the authors should release anonymized versions (if applicable).
        \item Providing as much information as possible in supplemental material (appended to the paper) is recommended, but including URLs to data and code is permitted.
    \end{itemize}

\item {\bf Experimental setting/details}
    \item[] Question: Does the paper specify all the training and test details (e.g., data splits, hyperparameters, how they were chosen, type of optimizer) necessary to understand the results?
    \item[] Answer: \answerYes{} % Replace by \answerYes{}, \answerNo{}, or \answerNA{}.
    \item[] Justification: Yes, see Appendix~\ref{app:hyperparamters}.
    \item[] Guidelines:
    \begin{itemize}
        \item The answer \answerNA{} means that the paper does not include experiments.
        \item The experimental setting should be presented in the core of the paper to a level of detail that is necessary to appreciate the results and make sense of them.
        \item The full details can be provided either with the code, in appendix, or as supplemental material.
    \end{itemize}

\item {\bf Experiment statistical significance}
    \item[] Question: Does the paper report error bars suitably and correctly defined or other appropriate information about the statistical significance of the experiments?
    \item[] Answer: \answerYes{} % Replace by \answerYes{}, \answerNo{}, or \answerNA{}.
    \item[] Justification: Yes, in Appendix~\ref{app:error_bars}.
    \item[] Guidelines:
    \begin{itemize}
        \item The answer \answerNA{} means that the paper does not include experiments.
        \item The authors should answer \answerYes{} if the results are accompanied by error bars, confidence intervals, or statistical significance tests, at least for the experiments that support the main claims of the paper.
        \item The factors of variability that the error bars are capturing should be clearly stated (for example, train/test split, initialization, random drawing of some parameter, or overall run with given experimental conditions).
        \item The method for calculating the error bars should be explained (closed form formula, call to a library function, bootstrap, etc.)
        \item The assumptions made should be given (e.g., Normally distributed errors).
        \item It should be clear whether the error bar is the standard deviation or the standard error of the mean.
        \item It is OK to report 1-sigma error bars, but one should state it. The authors should preferably report a 2-sigma error bar than state that they have a 96\% CI, if the hypothesis of Normality of errors is not verified.
        \item For asymmetric distributions, the authors should be careful not to show in tables or figures symmetric error bars that would yield results that are out of range (e.g., negative error rates).
        \item If error bars are reported in tables or plots, the authors should explain in the text how they were calculated and reference the corresponding figures or tables in the text.
    \end{itemize}

\item {\bf Experiments compute resources}
    \item[] Question: For each experiment, does the paper provide sufficient information on the computer resources (type of compute workers, memory, time of execution) needed to reproduce the experiments?
    \item[] Answer: \answerYes{} % Replace by \answerYes{}, \answerNo{}, or \answerNA{}.
    \item[] Justification: Yes, we provide the information used for sampling and training in Appendix~\ref{app:training_details}.
    \item[] Guidelines:
    \begin{itemize}
        \item The answer \answerNA{} means that the paper does not include experiments.
        \item The paper should indicate the type of compute workers CPU or GPU, internal cluster, or cloud provider, including relevant memory and storage.
        \item The paper should provide the amount of compute required for each of the individual experimental runs as well as estimate the total compute. 
        \item The paper should disclose whether the full research project required more compute than the experiments reported in the paper (e.g., preliminary or failed experiments that didn't make it into the paper). 
    \end{itemize}
    
\item {\bf Code of ethics}
    \item[] Question: Does the research conducted in the paper conform, in every respect, with the NeurIPS Code of Ethics \url{https://neurips.cc/public/EthicsGuidelines}?
    \item[] Answer: \answerYes{} % Replace by \answerYes{}, \answerNo{}, or \answerNA{}.
    \item[] Justification: We have read the NeurIPS Code of Ethics and ensured the paper satisfies every aspect.
    \item[] Guidelines:
    \begin{itemize}
        \item The answer \answerNA{} means that the authors have not reviewed the NeurIPS Code of Ethics.
        \item If the authors answer \answerNo, they should explain the special circumstances that require a deviation from the Code of Ethics.
        \item The authors should make sure to preserve anonymity (e.g., if there is a special consideration due to laws or regulations in their jurisdiction).
    \end{itemize}

\item {\bf Broader impacts}
    \item[] Question: Does the paper discuss both potential positive societal impacts and negative societal impacts of the work performed?
    \item[] Answer: \answerYes{} % Replace by \answerYes{}, \answerNo{}, or \answerNA{}.
    \item[] Justification: See Appendix~\ref{app:social_impacts}.
    \item[] Guidelines:
    \begin{itemize}
        \item The answer \answerNA{} means that there is no societal impact of the work performed.
        \item If the authors answer \answerNA{} or \answerNo, they should explain why their work has no societal impact or why the paper does not address societal impact.
        \item Examples of negative societal impacts include potential malicious or unintended uses (e.g., disinformation, generating fake profiles, surveillance), fairness considerations (e.g., deployment of technologies that could make decisions that unfairly impact specific groups), privacy considerations, and security considerations.
        \item The conference expects that many papers will be foundational research and not tied to particular applications, let alone deployments. However, if there is a direct path to any negative applications, the authors should point it out. For example, it is legitimate to point out that an improvement in the quality of generative models could be used to generate Deepfakes for disinformation. On the other hand, it is not needed to point out that a generic algorithm for optimizing neural networks could enable people to train models that generate Deepfakes faster.
        \item The authors should consider possible harms that could arise when the technology is being used as intended and functioning correctly, harms that could arise when the technology is being used as intended but gives incorrect results, and harms following from (intentional or unintentional) misuse of the technology.
        \item If there are negative societal impacts, the authors could also discuss possible mitigation strategies (e.g., gated release of models, providing defenses in addition to attacks, mechanisms for monitoring misuse, mechanisms to monitor how a system learns from feedback over time, improving the efficiency and accessibility of ML).
    \end{itemize}
    
\item {\bf Safeguards}
    \item[] Question: Does the paper describe safeguards that have been put in place for responsible release of data or models that have a high risk for misuse (e.g., pre-trained language models, image generators, or scraped datasets)?
    \item[] Answer: \answerYes{} % Replace by \answerYes{}, \answerNo{}, or \answerNA{}.
    \item[] Justification: We used well-established datasets that are curated, and we validate that all our trained models satisfy the safeguard requirements.
    \item[] Guidelines:
    \begin{itemize}
        \item The answer \answerNA{} means that the paper poses no such risks.
        \item Released models that have a high risk for misuse or dual-use should be released with necessary safeguards to allow for controlled use of the model, for example by requiring that users adhere to usage guidelines or restrictions to access the model or implementing safety filters. 
        \item Datasets that have been scraped from the Internet could pose safety risks. The authors should describe how they avoided releasing unsafe images.
        \item We recognize that providing effective safeguards is challenging, and many papers do not require this, but we encourage authors to take this into account and make a best faith effort.
    \end{itemize}

\item {\bf Licenses for existing assets}
    \item[] Question: Are the creators or original owners of assets (e.g., code, data, models), used in the paper, properly credited and are the license and terms of use explicitly mentioned and properly respected?
    \item[] Answer: \answerYes{} % Replace by \answerYes{}, \answerNo{}, or \answerNA{}.
    \item[] Justification:  We properly credit all external code assets used in our work by citing the original papers. The only assets that are not ours and were used in this work are the backbone architecture (credited when citing the original paper by~\citet{karras2022elucidating, song2020score}), and the public datasets (credited in Appendix~\ref{app:training_details}).
    \item[] Guidelines:
    \begin{itemize}
        \item The answer \answerNA{} means that the paper does not use existing assets.
        \item The authors should cite the original paper that produced the code package or dataset.
        \item The authors should state which version of the asset is used and, if possible, include a URL.
        \item The name of the license (e.g., CC-BY 4.0) should be included for each asset.
        \item For scraped data from a particular source (e.g., website), the copyright and terms of service of that source should be provided.
        \item If assets are released, the license, copyright information, and terms of use in the package should be provided. For popular datasets, \url{paperswithcode.com/datasets} has curated licenses for some datasets. Their licensing guide can help determine the license of a dataset.
        \item For existing datasets that are re-packaged, both the original license and the license of the derived asset (if it has changed) should be provided.
        \item If this information is not available online, the authors are encouraged to reach out to the asset's creators.
    \end{itemize}

\item {\bf New assets}
    \item[] Question: Are new assets introduced in the paper well documented and is the documentation provided alongside the assets?
    \item[] Answer: \answerYes{} % Replace by \answerYes{}, \answerNo{}, or \answerNA{}.
    \item[] Justification: All the information regarding model and datasets is provided in Appendix~\ref{app:hyperparamters}.
    \item[] Guidelines:
    \begin{itemize}
        \item The answer \answerNA{} means that the paper does not release new assets.
        \item Researchers should communicate the details of the dataset\slash code\slash model as part of their submissions via structured templates. This includes details about training, license, limitations, etc. 
        \item The paper should discuss whether and how consent was obtained from people whose asset is used.
        \item At submission time, remember to anonymize your assets (if applicable). You can either create an anonymized URL or include an anonymized zip file.
    \end{itemize}

\item {\bf Crowdsourcing and research with human subjects}
    \item[] Question: For crowdsourcing experiments and research with human subjects, does the paper include the full text of instructions given to participants and screenshots, if applicable, as well as details about compensation (if any)? 
    \item[] Answer: \answerNA{} % Replace by \answerYes{}, \answerNo{}, or \answerNA{}.
    \item[] Justification: We did not do experiments with human subjects.
    \item[] Guidelines:
    \begin{itemize}
        \item The answer \answerNA{} means that the paper does not involve crowdsourcing nor research with human subjects.
        \item Including this information in the supplemental material is fine, but if the main contribution of the paper involves human subjects, then as much detail as possible should be included in the main paper. 
        \item According to the NeurIPS Code of Ethics, workers involved in data collection, curation, or other labor should be paid at least the minimum wage in the country of the data collector. 
    \end{itemize}

\item {\bf Institutional review board (IRB) approvals or equivalent for research with human subjects}
    \item[] Question: Does the paper describe potential risks incurred by study participants, whether such risks were disclosed to the subjects, and whether Institutional Review Board (IRB) approvals (or an equivalent approval/review based on the requirements of your country or institution) were obtained?
    \item[] Answer: \answerNA{} % Replace by \answerYes{}, \answerNo{}, or \answerNA{}.
    \item[] Justification: We did not do experiments with human subjects.
    \item[] Guidelines:
    \begin{itemize}
        \item The answer \answerNA{} means that the paper does not involve crowdsourcing nor research with human subjects.
        \item Depending on the country in which research is conducted, IRB approval (or equivalent) may be required for any human subjects research. If you obtained IRB approval, you should clearly state this in the paper. 
        \item We recognize that the procedures for this may vary significantly between institutions and locations, and we expect authors to adhere to the NeurIPS Code of Ethics and the guidelines for their institution. 
        \item For initial submissions, do not include any information that would break anonymity (if applicable), such as the institution conducting the review.
    \end{itemize}

\item {\bf Declaration of LLM usage}
    \item[] Question: Does the paper describe the usage of LLMs if it is an important, original, or non-standard component of the core methods in this research? Note that if the LLM is used only for writing, editing, or formatting purposes and does \emph{not} impact the core methodology, scientific rigor, or originality of the research, declaration is not required.
    %this research? 
    \item[] Answer: \answerNA{} % Replace by \answerYes{}, \answerNo{}, or \answerNA{}.
    \item[] Justification: Our method does not involve LLMs in any of the core modules.
    \item[] Guidelines:
    \begin{itemize}
        \item The answer \answerNA{} means that the core method development in this research does not involve LLMs as any important, original, or non-standard components.
        \item Please refer to our LLM policy in the NeurIPS handbook for what should or should not be described.
    \end{itemize}

\end{enumerate}